\newif\ifarxiv % -- choose preprint (arxiv) version versus conference format
\arxivtrue
\newif\ifreview % -- turn on and off author info, page+line numbers and appendix proofs
\reviewfalse

\newif\iffinal % -- turn on and off colored comments
\finaltrue
\ifarxiv
\documentclass{article}
\else
\ifreview
\documentclass[3p,review,authoryear]{elsarticle}
\else
\documentclass[3p,final,authoryear]{elsarticle}
\fi
\journal{Artificial Intelligence}
\fi

\usepackage[utf8]{inputenc}

\ifarxiv
\usepackage[a4paper,left=2cm,right=2cm,top=2cm,bottom=2cm]{geometry}
\usepackage[onehalfspacing]{setspace}
\usepackage[round,compress]{natbib}
\fi

\usepackage{amssymb,amsfonts,amsmath,amsthm}
\usepackage{dsfont}

\usepackage{mathtools}
\usepackage{booktabs}
\usepackage{xspace}
\usepackage[dvipsnames]{xcolor}
\usepackage{graphicx}
\usepackage{enumitem}
\usepackage{subcaption}

\ifarxiv
\usepackage[pdfusetitle,hidelinks,breaklinks=true]{hyperref} % for pdfinfo and indexing
\else
\usepackage{hyperref}
\fi

\ifarxiv
\usepackage{authblk}
\fi

\usepackage{algorithm}
\usepackage[noend]{algorithmic}
\usepackage{tikz}
\usetikzlibrary{shapes,arrows}
\usetikzlibrary{decorations}
\usetikzlibrary{plotmarks}
\usetikzlibrary{calc}
\usetikzlibrary{mindmap}
\usetikzlibrary{shadows}
\usetikzlibrary{backgrounds}
\usetikzlibrary{patterns}
\usetikzlibrary{shapes.symbols}

\allowdisplaybreaks

\newtheorem{theorem}             {Theorem}
\newtheorem{lemma}      [theorem]{Lemma}

\newtheorem{definition} [theorem]{Definition}

\newcommand{\ab}{\hspace{0.125em}}                        % 1/8 em space
\newcommand{\ie}{\hbox{i.\ab e.}\xspace}                  % i.e.
\newcommand{\eg}{\hbox{e.\ab g.}\xspace}                  % e.g.
\newcommand{\Real}{\mathbb{R}}

\newcommand{\Natural}{\mathbb{N}}

\DeclareMathOperator{\Unif}{Unif}                         % uniform distribution
\DeclareMathOperator{\Bin}{Bin}                           % binomial distribution
\newcommand*{\bigO}{\mathcal{O}}

\DeclareMathOperator{\poly}{poly}                         % polynomial
\DeclareMathOperator{\Hamming}{H}                         % Hamming distance
\newcommand{\prob}[1]{\Pr\left(#1\right)}                 % probability
\newcommand{\Prob}[1]{\Pr\left(#1\right)}                 % probability
\newcommand{\expect}[1]{\mathrm{E}\left[#1\right]}        % expectation
\newcommand{\filtuc}[1]{\mathcal{F}_{#1}}                 % filtration
\newcommand{\filt}[1]{\filtuc{#1}}

\newcommand{\RRRMO}{\ensuremath{\textsc{RR}_{\mathrm{MO}}}\xspace}  % the Real Royal Road function

\newcommand{\posFit}{\ensuremath{G}\xspace}               % set of strictly positive fitness (f1>0,f2>0)
\newcommand{\frontP}{\ensuremath{F}\xspace}               % Pareto optimal set of RRRMO
\newcommand{\lfrontP}{\ensuremath{F'}\xspace}             % Pareto optimal set of lRRRMO

\newcommand{\LZ}{\textsc{LZ}\xspace}                      % # of leading zeroes
\newcommand{\TZ}{\textsc{TZ}\xspace}                      % # of trailing zeroes
\newcommand{\ones}[1]{|#1|_1}
\newcommand{\zeros}[1]{|#1|_0}
\newcommand{\OM}{\textsc{OneMax}\xspace}                           % OneMax function
\newcommand{\JUMP}{\textsc{Jump}\xspace}                           % Jump function
\newcommand{\RRRfull}{\textsc{RealRoyalRoad}\xspace}               % RRR full name

\newcommand{\COCZ}{\textsc{COCZ}\xspace}                           % COCZ function
\newcommand{\COCZfull}{\textsc{CountingOnesCountingZeroes}\xspace} % COCZ full name

\newcommand{\OMM}{\textsc{OMM}\xspace}                             % OMM function
\newcommand{\OMMfull}{\textsc{OneMinMax}\xspace}                   % OMM full name

\newcommand{\LOTZ}{\textsc{LOTZ}\xspace}                           % LOTZ function
\newcommand{\LOTZfull}{\textsc{LeadingOnesTrailingZeroes}\xspace}  % LOTZ full name

\newcommand{\OJZJfull}{\textsc{OneJumpZeroJump}\xspace}            % OJZJ full name

\newcommand{\cdist}{\ensuremath{\textsc{cDist}}\xspace}   % crowding distance

\newcommand{\muplusga}{($\mu$+$\lambda$)~GA\xspace}       % (mu+lambda) GA algorithm

\newcommand{\uRRRMO}{\ensuremath{\mathrm{u}\textsc{RR}_{\mathrm{MO}}^*}\xspace} % the Real Royal Road function for uniform crossover
\newcommand{\uRRRMOx}[2]{\ensuremath{\mathrm{u}\textsc{RR}_{\mathrm{MO}}^{#1,#2}}\xspace} % a function in the class with specific permutation and target
\newcommand{\uRRRMOc}{\ensuremath{\mathrm{u}\textsc{RR}_{\mathrm{MO}}}\xspace} % the class of Real Royal Road functions for uniform crossover

\newcommand{\permut}[2]{\ensuremath{#1(#2)}\xspace}         % application of permutation #1 on string #2

\newcommand{\strl}[1]{\ensuremath{#1_\ell}\xspace}          % left part of a bitstring
\newcommand{\strr}[1]{\ensuremath{#1_r}\xspace}             % right part of a bitstring
\newcommand{\subsetU}{\ensuremath{U}\xspace}                % subset U and P of strl{x} (P clash with the population or not?)
\newcommand{\subsetP}{\ensuremath{P}\xspace}                %
\newcommand{\subsetC}{\ensuremath{C}\xspace}                % subsets C and T of strr{x}
\newcommand{\subsetT}{\ensuremath{T}\xspace}                %
\newcommand{\subsetK}{\ensuremath{K}\xspace}                % subset K of the search space
\newcommand{\subsetW}{\ensuremath{W}\xspace}                % subset K of the search space

\newcommand{\subsetL}{\ensuremath{L}\xspace}                % subset L of locations in a bitstring
\newcommand{\onesL}[1]{\ensuremath{X_{#1}}\xspace}          % number of ones in subset L at specific time
\newcommand{\maxOZL}[1]{\ensuremath{Y_{#1}}\xspace}         % maximum between the numbers of ones and zeroes in subset L at specific time

\newcommand{\LO}{\textsc{LO}\xspace}                        % # of leading ones
\newcommand{\TOs}{\textsc{TO}\xspace}                       % # of trailing ones (note: \TO is already used by algorithmic)

\iffinal
\newcommand{\todo}[1]{}
\newcommand{\newedit}[1]{#1}
\newcommand{\brandnewedit}[1]{#1}
\newcommand{\andre}[1]{}
\newcommand{\dirk}[1]{}
\newcommand{\cuong}[1]{}
\newcommand{\bahare}[1]{}
\else
\newcommand{\todo}[1]{\textcolor{red}{[TODO: #1]}}
\newcommand{\newedit}[1]{\textcolor{blue}{#1}}
\newcommand{\brandnewedit}[1]{\textcolor{cyan}{#1}}
\newcommand{\andre}[1]{\textcolor{green!50!black}{[(Andre) #1]}}
\newcommand{\dirk}[1]{\textcolor{magenta}{[(Dirk) #1]}}
\newcommand{\cuong}[1]{\textcolor{orange}{[(Cuong) #1]}}
\newcommand{\bahare}[1]{\textcolor{purple}{[(Bahare) #1]}}
\fi

\ifreview
\usepackage[switch,mathlines]{lineno}

\usepackage{etoolbox} % patch 'normal' math environments to work with lineno
\newcommand*\linenomathpatch[1]{
  \cspreto{#1}{\linenomath}%
  \cspreto{#1*}{\linenomath}%
  \csappto{end#1}{\endlinenomath}%
  \csappto{end#1*}{\endlinenomath}%
}
\linenomathpatch{equation}
\linenomathpatch{gather}
\linenomathpatch{multline}
\linenomathpatch{align}
\linenomathpatch{alignat}
\linenomathpatch{flalign}
\fi

\ifarxiv
\title{Crossover Can Guarantee Exponential Speed-Ups in Evolutionary Multi-Objective Optimisation}
\author[1]{Duc-Cuong Dang}
\author[1]{Andre Opris}
\author[1]{Dirk Sudholt}
\affil[1]{University of Passau, Passau, Germany}
\date{}
\fi

\begin{document}

\ifarxiv
\maketitle
\else
\begin{frontmatter}

\title{Crossover Can Guarantee Exponential Speed-Ups in Evolutionary Multi-Objective Optimisation}
\author[1]{Duc-Cuong~Dang}\ead{duccuong.dang@uni-passau.de}
\author[1]{Andre~Opris}\ead{andre.opris@uni-passau.de}
%\author[1,2,3]{Bahare~Salehi}\ead{salehi.bahaar@gmail.com}
\author[1]{Dirk~Sudholt\corref{cor1}}\ead{dirk.sudholt@uni-passau.de}
\cortext[cor1]{Corresponding author}
\affiliation[1]{organization={Chair of Algorithms for Intelligent Systems, University of Passau},
%    addressline={Instra\ss{}e 43},
    city={Passau},
%    postcode={94032},
    country={Germany}}
%\affiliation[2]{organization={School of Electrical and Computer Engineering, Shiraz University},
%%    addressline={},
%    city={Shiraz},
%%    postcode={},
%    country={Iran}}
%\affiliation[3]{organization={Engineering School of Brussels, Universit\unexpanded{\'{e}} Libre de Bruxelles},
%%    addressline={Avenue Franklin Roosevelt 50},
%    city={Brussels},
%%    postcode={1050},
%    country={Belgium}}
\fi

\begin{abstract}

Evolutionary algorithms are popular algorithms for multiobjective optimisation (also called Pareto optimisation) as they use a population to store trade-offs
between different objectives. Despite their popularity, the theoretical
foundation of multiobjective evolutionary optimisation (EMO) is still in its
early development.
Fundamental questions such as the benefits of the crossover
operator are
still not fully
understood.
We provide a theoretical analysis of the well-known EMO algorithms GSEMO and
NSGA-II to showcase the possible advantages of crossover: we propose classes of ``royal road'' functions on which these algorithms cover the whole Pareto front in expected polynomial
time if crossover is being used. But when disabling crossover, they
%In sharp contrast, these and many other EMO algorithms without crossover
require exponential time in expectation to cover the Pareto front.
The latter even holds for a large class of black-box algorithms using any elitist selection and any unbiased mutation operator. Moreover, even the expected time to create a single Pareto-optimal search point is exponential. We provide two different function classes, one tailored for one-point crossover and another one tailored for uniform crossover, and we show that immune-inspired hypermutations cannot avoid exponential optimisation times. Our work shows the first example of an exponential performance gap through the use of
crossover for the widely used NSGA-II algorithm and contributes to a deeper understanding of its limitations and capabilities.
\end{abstract}

\ifarxiv
\textbf{Keywords: evolutionary computation, runtime analysis, recombination, multi-objective optimisation, unbiased black box algorithms, \newedit{hypermutation}}
\else
\begin{keyword}
evolutionary computation \sep runtime analysis \sep recombination\sep multi-objective optimisation\sep unbiased black box algorithms\newedit{\sep \newedit{hypermutation}}
\end{keyword}

\end{frontmatter}

\ifreview
\linenumbers
\fi

\fi

%\ifreview
%\linenumbers
%\fi

\section{Introduction}\label{sec:intro}

Many optimisation problems have multiple conflicting objectives and the aim is
to find a set of Pareto-optimal solutions. Evolutionary algorithms (EAs) are
general-purpose optimisers that use principles from natural evolution such as
mutation, crossover (recombination) and selection to evolve a population
(multi-set) of candidate solutions. EAs such as the popular algorithm
NSGA-II~\citep{Deb2002} are well suited for this task as they are able to use
their population to store multiple trade-offs between objectives.
%The field of evolutionary multiobjective optimisation (EMO) has found applications in many areas, including \dirk{To do}.
%
However, the theoretical understanding of evolutionary multiobjective
optimisation (EMO) is lagging far behind its success in
practice~\citep{ZhengLuiDoerrAAAI22}. There is little understanding on how the
choice of search operators and parameters affects performance %, especially
%. There is little understanding of the working principles behind EAs
in multiobjective settings.

In single-objective evolutionary optimisation, a rigorous theory has emerged
over the past 25 years. It led to a better understanding of the working
principles of EAs via performance guarantees and it inspired the
design of novel EAs with better performance guarantees, e.\,g.\
choosing mutation rates from a heavy-tailed distribution to enable large
changes~\citep{Doerr2017-fastGA}, changing the order of crossover and mutation
and amplifying the probability of improving
mutations~\citep{Doerr2015}, parent selection preferring worse search
points~\citep{CorusLOW21} or adapting mutation rates during the run~\citep{Doerr2019opl}.

In particular, the importance of the crossover operator is not well understood, despite being a
topic of intensive, ongoing research in evolutionary computation and in
population genetics~\citep{paixao_unified_2015}. In single-objective
optimisation there is a body of works on the usefulness of
crossover~\citep[Section~8.4]{Sudholt2018} on illustrative pseudo-Boolean example
problems~\citep{Jansen2005c,Storch2004,Koetzing2011,Dang2017,Sudholt2016,Corus2018a,Doerr2015}
and problems from combinatorial optimisation such as
    shortest paths~\citep{Doerr2012},
    graph colouring problems~\citep{Fischer2005,Sudholt2005} and
    the closest string problem~\citep{Sutton21}.

However, in EMO results are scarce. Understanding and rigorously analysing the
dynamic behaviour of EAs is hard enough in single-objective optimisation. EMO
brings about additional challenges as there is no total order between search
points. Search points may be incomparable due to trade-offs between different
objectives. The most widely used EMO algorithm NSGA-II~\citep{Deb2002} imposes a
total order by using non-dominated sorting (sorting the population according to
ranks based on dominance) and a diversity score called crowding distance to
break ties between equal ranks. Understanding this ranking is non-trivial, and
the first rigorous runtime analyses of NSGA-II were only published at
AAAI~2022~\citep{ZhengLuiDoerrAAAI22}.

%\paragraph{Our contribution:}
%\textbf{Our contribution:}
\newedit{\subsection{Our Contribution}}

\brandnewedit{In this paper, we demonstrate the possible advantages of
crossover for EMO by presenting examples of $n$-bit pseudo-Boolean functions
on which the use of crossover has a drastic effect on performance. We prove using rigorous runtime analysis that well-known EMO algorithms GSEMO and NSGA-II using crossover can find the whole Pareto front in polynomial expected time, while these and large classes of black-box optimisation algorithms without crossover require exponential expected time.
These functions therefore can be regarded as a ``royal road'' for the success of crossover in multiobjective optimisation, similar to previous results for single-objective optimisation~\citep{Jansen2005c}.
To our knowledge, this is the first proof of
an exponential performance gap for the use of crossover for NSGA-II.
In parallel independent work, \citet{Doerr2023} showed a polynomial gap for the use of
crossover in NSGA-II.
%Along with parallel independent work by~\citet{Doerr2023}, this is the first proven speedup through crossover
%exponential\footnote{In parallel independent work, \citet{Doerr2023} showed a polynomial speedup owing to crossover, see Section~\ref{sec:related-work}.} performance gap due to the use of crossover proven for NSGA-II.
}

\brandnewedit{More specifically, we propose a test function \RRRMO as a ``royal road'' for one-point crossover. Our function is inspired by Jansen and Wegener's royal road functions for single-objective optimisation~\citep{Jansen2005c}. The function contains a fitness valley of exponential size and very poor fitness. In order to locate the Pareto front, black-box algorithms typically have to cross this large fitness valley. This is hard for all unbiased mutation operators (operators treating bit values and all bit positions symmetrically) as a linear number of bits have to be flipped and the probability of choosing the right bits to flip in order to hit the Pareto front is exponentially small. However, the function is designed such that one-point crossover can combine the prefix with the suffix of two non-dominated solutions that represent different trade-offs stored in the population of GSEMO and NSGA-II, respectively. We prove that these algorithms can find the whole Pareto front of \RRRMO in expected time $O(n^4)$.
}

\brandnewedit{Since one-point crossover lacks the symmetry with respect to bit values (neighbouring bits in the bit string have a higher chance to come from the same parent), one may ask whether this positional bias is the reason for the success of algorithms with one-point crossover. To investigate this, we also consider the somatic contiguous hypermutation operator (or hypermutation for short) from artificial immune systems that has an inherent positional bias and is able to mutate a contiguous interval of bits in the bit string. We prove that also this operator is unable to find the whole Pareto front, supporting our conclusion that the ability of one-point crossover to combine information from different parents is crucial for optimising \RRRMO.}

\brandnewedit{
We also design a function \uRRRMO as a ``royal road'' for uniform crossover. As first observed by~\citet{Jansen2005c} for single-objective royal road functions, designing a royal road function for uniform crossover is generally harder than for one-point crossover as uniform crossover can create an exponential number of offspring in the Hamming distance of the two parents. We extend the construction from~\citet{Jansen2005c} towards multiobjective optimisation in such a way that the function makes GSEMO and NSGA-II find the Pareto front in expected time $O(n^3)$, while large classes of mutation-only algorithms require exponential time. This class includes all unbiased mutation operators as well as the hypermutation operator. The function is constructed such that a large fitness valley has to be crossed. GSEMO and NSGA-II with crossover can do so by crossing two non-dominated solutions as parents whose bit values agree in the left half of the string and whose bit values are complementary in the right half of the string. The uniform crossover is likely to create an even balance of ones and zeros in the right half, while keeping all bits in the left half, on which both parents agree, intact. It thus has a good probability of finding a target set of exponential size from which it is easy to cover the Pareto front. On this function unbiased mutation fails to find the target as it must treat all bit positions symmetrically and so is not able to keep half the bits unchanged while also making large changes to the other half of bits. The same holds for the hypermutation operator if the order of bits in the function definition is permuted in such a way that the operator has no way of choosing precisely the bits that should be mutated. Since unbiased mutations and uniform crossover operator independent from the order of bit positions, all other results hold for arbitrary permutations of bits.
}
\newedit{\brandnewedit{These are captured by the function class \uRRRMOc
        and all} our results are summarised in Table~\ref{tab:summary-results}.}

\begin{table}[ht]
\begin{center}
\begin{tabular}{ccccc}
	\toprule
	\textbf{problem}
	& \textbf{algorithm(s)}
	& \textbf{crossover}
	& \textbf{mutation}
	& \textbf{bounds on $\boldsymbol{\expect{T}}$}\\
	\midrule
	\RRRMO & GSEMO, NSGA-II & none & standard bit mutation & $n^{\Omega(n)}$ (Theorem~\ref{thm:gsemo-pc-zero})\\
%	\RRRMO & NSGA-II & none & standard bit mutation & $n^{\Omega(n)}$ (Theorem~\ref{thm:nsga-ii-pc-zero})\\
	\RRRMO & ($\mu$+$\lambda$)~Black box & none & any unbiased & $2^{\Omega(n)}$ (Theorem~\ref{thm:elitist-blackbox-unary-unbiased})\\
	\RRRMO & GSEMO & none & hypermutation & $2^{\Omega(n)}$ (Theorem~\ref{thm:gsemo-pc-zero-hypermutation})\\
	\RRRMO & GSEMO & one-point & standard bit mutation & $\bigO\left(\frac{n^4}{1-p_c} + \frac{n}{p_c}\right)$ (Theorem~\ref{thm:gsemo})\\
	\RRRMO & NSGA-II & one-point & standard bit mutation & $\bigO\left(\frac{\mu n^3}{1-p_c} + \frac{\mu^2}{n p_c}\right)$ (Theorem~\ref{thm:nsga-ii})\\
	\midrule
	\uRRRMOc & GSEMO, NSGA-II & none & standard bit mutation & $n^{\Omega(n)}$ (Theorem~\ref{thm:uRRMO-gsemo-pczero-stdbit})\\
	\uRRRMOc & GSEMO, NSGA-II & none & any unbiased & $2^{\Omega(n)}$ (Theorem~\ref{thm:uRRMO-gsemo-pczero-unbiasedvar})\\
	\uRRRMOc & GSEMO, NSGA-II & none & hypermutation & $2^{\Omega(n)}$ (Theorem~\ref{thm:uRRMO-gsemo-pczero-hypermut})\\
%	\uRRRMOc & NSGA-II & none & standard bit mutation & $n^{\Omega(n)}$ (Theorem~\ref{thm:uRRMO-nsgaii-pczero-stdbit})\\
%	\uRRRMOc & NSGA-II & none & unbiased & $2^{\Omega(n)}$ (Theorem~\ref{thm:uRRMO-nsgaii-pczero-unbiasedvar})\\
%	\uRRRMOc & NSGA-II & none & hypermutation & $2^{\Omega(n)}$ (Theorem~\ref{thm:uRRMO-nsgaii-pczero-hypermut})\\
	\uRRRMOc & ($\mu$+$\lambda$)~Black box & none & any unbiased & $2^{\Omega(n)}$ (Theorem~\ref{thm:uRRMO-blackbox-unbiasedvar})\\
	\uRRRMOc & ($\mu$+$\lambda$)~Black box & none & hypermutation & $2^{\Omega(n)}$ (Theorem~\ref{thm:uRRMO-blackbox-hypermut})\\
	\uRRRMOc & GSEMO & uniform & standard bit mutation &  $\bigO\left(\frac{n^3}{p_c(1-p_c)}\right)$ (Theorem~\ref{thm:uRRMO-gsemo})\\
	\uRRRMOc & NSGA-II & uniform & standard bit mutation & $\bigO\left(\frac{\mu n^2}{1-p_c} + \frac{\mu^2 n}{p_c}\right)$ (Theorem~\ref{thm:uRRMO-nsgaii})\\
	\bottomrule
\end{tabular}
\end{center}
\caption{Summary of the results for the $\RRRMO$ function and the $\uRRRMOc$
function class. We assume standard bit mutation with mutation rate $1/n$,
hypermutation with any parameter $r\in(0,1]$ and the crossovers (if used) with
probability $p_c\in (0,1]$ to be applied.
\newedit{The bounds on $\expect{T}$ are in terms of number of fitness evaluations.}
The results in \brandnewedit{lines 1, 2, 4, 5}
%the first half of the table
%concerning with $\RRRMO$ but excluding hypermutation (the second line)
first
appeared in~\cite{Dang2023}.}\label{tab:summary-results}
\end{table}

%, with high probability.
%To our knowledge, this is the first proof of
%an exponential performance gap for the use of crossover for NSGA-II.
% was shown in .
% for
%%the same algorithm.
%NSGA-II.

%%This
%Our results showcase the potential benefits of crossover on function designed
%to serve as a ``royal road'' for the success of crossover, that is, an
%illustrative example of a problem where the use of crossover is essential.
%The design of \newedit{our first function} \RRRMO is deliberately simple to support
%a rigorous theoretical analysis and to be suited for teaching purposes.
%We study GSEMO as a simple algorithm %well-suited for a theoretical analysis
%and present a more involved analysis for NSGA\nobreakdash-II as the best known EMO algorithm.
%Our hardness results apply to a broad class of EMO algorithms to show that all
%\newedit{unbiased} mutation operators are ineffective. %\newedit{and so is hypermutation alone}.

As another technical contribution, we refine and generalise previous arguments from the
analysis of NSGA\nobreakdash-II (Lemma~\ref{lem:nsga-ii-protect-layer}) about
the survival of useful search points from function-specific arguments to general
classes of fitness functions.
%, including sufficient conditions for useful search points to survive the
%replacement selection.
%We are optimistic that
Our work may serve as a stepping stone towards analyses
of the benefits of crossover on wider problem classes, in the same way that this
was achieved for single-objective optimisation.

\newedit{A preliminary version of this work appeared at AAAI~2023 (see
\citet{Dang2023}) where only \RRRMO was defined and compared against unbiased mutations. In this extended and improved manuscript we added results on hypermutations and the function \uRRRMO that requires a completely different construction from \RRRMO, along with corresponding runtime analyses for GSEMO, NSGA-II and classes of algorithms without crossover that require exponential time.}
%and the performance gap for the use
%of one-point crossover with standard bit mutation is discussed and in parallel independent work, \citet{Doerr2023} showed a polynomial gap for the use of
%crossover in NSGA-II.
%%In this extension we will also analyse GSEMO
%%with hypermutation on \RRRMO which leads to an exponential runtime.
%This paper extends the previous work to analyse GSEMO with hypermutation on
%\RRRMO and proves an exponential runtime.
%It remains
%an open question whether NSGA-II with hypermutation is able to optimise \RRRMO
%efficiently. We also extend the previous work by introducing the function
%class \uRRRMOc as a second example where GSEMO and NSGA-II exhibit an exponential
%large performance gap for the use of uniform crossover. We study the EMO
%algorithms on \uRRRMOc with hypermutation replacing standard bit mutation and show
%that they require expected exponential runtime. Note that uniform crossover
%requires a substantially different Royal Road function and a
%%completely %%-- too strong?
%different analysis (as in~\cite{Jansen2005c}).}

%\textbf{Related work:}
%\paragraph{Use of crossover in single-objective optimisation:}
\newedit{\subsection{Related Work}}
In single-objective optimisation, the first proof that using crossover can speed
up EAs was provided by~\citet{Jansen2002} for the function class
$\JUMP_k$, where a fitness valley of size~$k$ has to be crossed.
%For
%appropriate~$k$
For $k=\log n$, the performance gap was between polynomial and superpolynomial
%(but sub-exponential)
times.
These results were refined
in~\citep{Koetzing2011, Dang2017}.

\newedit{Most relevant to our work is the seminal paper by~\citet{Jansen2005c} that showed the first exponential performance gaps for the use of crossover.
The functions were called ``real royal road'' functions as previous attempts at defining ``royal road'' functions were unsuccessful~\citep{Mitchell1992,Forrest1993}.
\Citet{Jansen2005c} defined a royal road function for one-point crossover called \RRRfull, which
%, and our
%function design is inspired by this work.
encourages EAs to evolve
strings with all 1-bits gathered in a single block, and then one-point crossover can easily assemble the optimal string when choosing parents whose blocks of 1-bits are located at opposite ends of the bit string. A ($\mu$+1)~Genetic Algorithm with one-point crossover optimises \RRRfull in expected time $O(n^4)$, while all mutation-only EAs need exponential time with overwhelming probability.
\Citet{Jansen2005c} also defined a class of royal road functions for uniform crossover that can be solved by a ($\mu$+1)~Genetic Algorithm using uniform crossover in expected time $O(n^3)$ whereas evolutionary algorithms without crossover need exponential time with overwhelming probability.
Solving these royal road functions required crossover and a population of at least linear size. In follow-on work%
%\footnote{The paper by~\citet{Storch2004} was written after work by~\citet{Jansen2005c} was completed, even though the former appeared in print before the latter.}
, \citet{Storch2004} presented different royal road functions designed such that similar results could be shown for genetic algorithms using crossover
%and the smallest possible population size: 2.
\brandnewedit{and a small population size of only $2$.}
}

Advantages through crossover were also proven
for combinatorial problems:
shortest paths~\citep{Doerr2012}, graph colouring
problems~\citep{Fischer2005,Sudholt2005} and the closest string
problem~\citep{Sutton21}.
Crossover speeds up hill climbing on $\OM(x)$ that simply counts the number of ones
in $x$ by a constant factor~\citep{Sudholt2016,Corus2018a}. A cleverly designed EA
called (1+($\lambda$,$\lambda$))~GA outperforms the
best mutation-only EAs on \OM by a factor of $\bigO(\log n)$~\citep{Doerr2015,Doerr2018}.
Finally, crossover increases robustness on difficult monotone
pseudo-Boolean functions~\citep{Lengler2020a}.

%\paragraph{Early analyses of EMO algorithms (without crossover):}
Early rigorous analysis of EMO focused on simple algorithms, like SEMO (flipping a single bit for mutation) and its variant GSEMO (using standard bit mutations as a global search operator),
without crossover.
\citet{Laumanns2004} introduced two biobjective benchmark functions
    \LOTZfull (\LOTZ) and \COCZfull (\COCZ)
    %and their respective many objective variants m\LOTZ and m\COCZ,
to prove linear and sub-linear speed-ups in the expected optimisation
time of two variants of SEMO over a single-individual algorithm called ($1$+$1$)~EMO.
%, \eg a speed up in the order of $n/\log{n}$ can be achieved for LOTZ.
%
%\citet{Giel2003} provided a general expected runtime $\bigO(n^n)$
%of GSEMO on any fitness function and
%proved the runtime bound $\bigO(n^3)$ for \LOTZ with overwhelmingly high probability.
%
%Following up on \cite{Laumanns2004},
\citet{Giel2010} gave an example of a
biobjective function on which an exponential performance gap between
SEMO and ($1$+$1$)~EMO can be proven.
\citet{Covantes2020} proposed the use of diversity measures such as crowding distance
in the parent selection for SEMO.
%and GSEMO.
They proved that the use of a power-law
ranking selection to select parents ranked by the crowding distances yields
a linear speed-up in the expected optimisation time for \LOTZ, and for also the
\OMMfull (\OMM) function.
\citet{Doerr2021} introduced the \OJZJfull function, which generalises
$\JUMP_k$ to the multiobjective setting, to show that GSEMO, in contrast to SEMO,
can fully cover the Pareto set, and to further prove that %.
%They also proved that
the performance of GSEMO can be
%further
improved with the use of heavy-tailed mutation.
% and stagnation detection.

%\paragraph{Recent results on NSGA-II:}
NSGA-II \citep{Deb2002} is a practical and hugely popular reference algorithm for EMO,
however its theoretical analysis only succeeded recently.
%
%The standard benchmark functions for the analysis are %% [cuong] make sure we say these abbreviations once
%    CountingOnesCountingZeroes (COCZ),
%    OneMinMax (OMM) and
%    LeadingOnesTrailingZeroes (LOTZ).
%
\citet{ZhengLuiDoerrAAAI22} conducted the first runtime analysis of NSGA-II without
crossover and proved expected time bounds $\bigO(\mu n^{2})$ and
$\bigO(\mu n \log n)$ to find the whole Pareto set of \LOTZ and \OMM,
resp., if the population size $\mu$ is at least four times the size of the Pareto set.
%
%\citep{Bian2022PPSN} improved these bounds by a linear factor by considering crossover
%in a variant of the algorithm where parents are selected by tournaments and
%the tournament sizes are sampled uniformly at random between $1$ and $\mu$.
%The required factor on the population size $\mu$ is also reduced to two by assuming
%the use of a stable sort during the computation of the crowding distances.
%
The drawback of using a too small population size was further studied in
\citet{Zheng2022}. The authors proposed
% and shown to be due to the consideration of the same fitness
%multiple times. Thus
an incremental procedure to compute the crowding distances
as an improvement to the standard algorithm.
%
%Picked up from \cite{Doerr2021},
\citet{Qu2022PPSN} showed that
heavy-tailed mutations presented for the single objective settings \citep{Doerr2017-fastGA}
are also highly beneficial for EMO.
%Our work is similar to these papers in the spirit, however we generalise
%\RRRfull \citep{Jansen2005c}
%%in order
%to show the advantage of crossover.
%\todo{Could also list more papers here!}
\newedit{\Citet{Cerf2023} showed that NSGA-II is able to solve the NP-hard bi-objective minimum spanning tree problem. The upper bound $O(m^2 n w_{\max}\log(nw_{\max}))$ on the expected number of iterations depends on the number of edges, $m$, the number of vertices, $n$, and the largest edge weight, $w_{\max}$.}

%\paragraph{Advantages of crossover:}
%There are only a few works with provable advantages for crossover in EMO,
Only a few works rigorously prove the advantages of crossover in EMO,
even though experimentally they are noticeable (\eg \citet{Qu2022PPSN}).
\citet{Qian2011} proposed the REMO algorithm that initialises the population
with local optima and uses crossover to quickly fill the Pareto set %front
of example functions \LOTZ and \COCZ. This constitutes a speedup of
order~$n$ compared to the SEMO algorithm.
They later extended this work to more general function classes and multiobjective
minimum spanning trees~\citep{Qian2013}.
\citet{Qian_Bian_Feng_2020} compared two variants of
%the POSS algorithm
GSEMO
%\dirk{POSS vs.\ SEMO?} \bahare{POSS (Pareto Optimization for
%Subset Selection) vs.\ PORSS (recombination-based POSS). However I quote this
%sentence from the paper. POSS employs a simple MOEA with mutation only, which
%is slightly modified from the GSEMO algorithm}
with and without crossover, called POSS and PORSS respectively, % in the paper,
for the sub-set selection problem and showed that the recombination-based GSEMO
is almost always superior. In particular, they provide an exponential performance
gap for constructed instances of sub-set selection.
\citep{Huang2021b} compared the effectiveness of immune-inspired hypermutation operators against classical mutation operators in runtime analysis of a simple multi-objective evolutionary algorithm. The results on four bi-objective optimization problems imply that the hypermutation operators can always achieve the Pareto fronts in polynomial expected runtime and also have advantage in maintaining balance between exploration and exploitation compared to the classical mutation operators.
\citet{Doerr2022} introduced the (1+($\lambda$,$\lambda$))~GSEMO
algorithm and proved that it optimizes \OMM in expected time
$\bigO(n^2)$, a speedup of order $\log n$ compared to GSEMO.

\citet{Bian2022PPSN}
%provided a running time analysis for NSGA-II on
%bi-objective problems \textsc{LOTZ}, \textsc{OneMinMax} and \textsc{COCZ}. %They
%proposed to replace the original binary tournament selection of NSGA-II with
introduced a new parent
selection strategy named stochastic tournament selection using crowding distance to favour diverse parents as in~\citet{Covantes2020} to improve the expected
running time upper bounds of \LOTZ, \OMM and \COCZ to $\bigO(n^{2})$. Their analysis
relies on crossover to fill the Pareto set quickly. %front quickly.
However, the work
by~\citet{Covantes2020} already showed that for SEMO on \LOTZ the same
performance guarantee can be obtained,
%through tailoring the parent selection mechanism and
\brandnewedit{with a similar parent selection mechanism
called power-law ranking},
and that crossover is not required for an $\bigO(n^2)$ bound.
%\brandnewedit{In fact, we claim that stochastic tournament selection proposed in~\citet{Bian2022PPSN} is essentially equal to the power-law ranking selection proposed in~\citet{Covantes2020}. A proof of this is given in~\ref{apx:stoch-tournament}.}

\brandnewedit{Finally, \citet{Doerr2023}, which appeared at the same time as our preliminary work~\citep{Dang2023}, gives the joint first proof of speedups through crossover in NSGA-II. In contrast to this work, they consider a benchmark function \textsc{OneJumpZeroJump} based on the well-known \textsc{Jump} function and show a polynomial speedup through crossover. While their function is conceptually simpler than ours, the resulting speedup is polynomial, whereas ours is exponential. Their work uses uniform crossover and we provide results for both one-point crossover and uniform crossover.}

\section{Preliminaries}\label{sec:prelim} % notation, dominance definitions and algorithms
%\andre{
%For a search point $x \in \{0,1\}^n$ let $\LO(x)$ be the number of leading ones, $\TOs(x)$ be the number of trailing ones, $\LZ(x)$ be the number of leading zeros and $\TZ(x)$ be the number of trailing zeros respectively. For $x=111001011011$ we have for example $\LO(x)=3$, $\TOs(x)=2$, $\LZ(x)=0$ and  $\TZ(x)=0$. Denote by $\ones{x}$ the number of ones in a bit-string $x$. \\
%This is something I would write in the section about \uRRRMO below, since it is only relevant there:\\
%For a natural number $n$ which is divisible by $6$ and $x \in \{0,1\}^n$ we set the following:
%Let $x:=(\strl{x},\strr{x})$ where $\strl{x}$ and $\strr{x}$ have both length $m:=n/2$. Furthermore let $\strr{x}:=(\strr{x}^1,\strr{x}^2,\strr{x}^3)$ where $\strr{x}^i$ has length $n/6$. We say that $\strr{x}$ is on the circle $C$ if $\LO(\strr{x})+\TZ(\strr{x}) = m$ or $\LZ(\strr{x})+\TOs(\strr{x}) = m$, i.e. if $\strr{x}=1^i0^{m-i}$ or $\strr{x}=0^i1^{m-i}$ for $i \in \{1, \ldots ,m\}$. Note that $C$ is a closed path with length $2m=n$ and Hamming distance $1$ between neighbored points. We say that $\strr{x}$ is in the target $T$ if each of the blocks $\strr{x}^1,\strr{x}^2$ and $\strr{x}^3$ consists of $\lfloor{n/12}\rfloor$ ones and $\lceil{n/12}\rceil$ zeros.\\
%} \cuong{I moved these below and remove the second part specific to the function}
Consider maximising a function %$f$ with $d$ objectives:
$f(x)=(f_1(x),\dots,f_d(x))$ where $f_i\colon \{0,1\}^n \rightarrow \Natural_0$
for $1\leq i \leq d$.
Given two search points $x, y \in \{0, 1\}^n$,
    $x$ \emph{weakly dominates} $y$, denoted by $x \succeq y$,
    if $f_{i}(x) \geq f_{i}(y)$ for all $1\leq i \leq d$;
    and $x$ \emph{dominates} $y$, denoted $x \succ y$,
    if one inequality is strict.
Each solution that is not dominated by any other in $\{0, 1\}^n$ is called \emph{Pareto-optimal}.
A set of
%Pareto-optimal points
these solutions
that cover all possible non-dominated fitness values of $f$
is called a \emph{Pareto-optimal set} (or \emph{Pareto set} for short) of $f$.
%The goal is to obtain as many members of Pareto front which is the set of all
%Pareto-optimal solutions.
%
\newedit{Let $\mathcal{A}$ be an algorithm that seeks to optimise $f$ through querying the values of $f(x)$. Such queries are called \emph{fitness evaluations}. We assume $\mathcal{A}$ operates in \emph{iterations} (also referred to as \emph{generations}) and that it maintains a multi-set $P_t$ of solutions called \emph{the population}. By $\mu(t)$ we denote the number of fitness evaluations made in iteration~$t$. The running time of $\mathcal{A}$ on $f$ in
terms of generations is defined as
\[
    T^{\mathrm{gen}}_{\mathcal{A}(f)}:=\min\{t\mid F \subseteq P_t \wedge (F \text{ is a Pareto-optimal set})\},
\]
\ie the first generation $t$ where a Pareto-optimal set $F$ is found in $P_t$,
and the running time in terms of fitness evaluations is
$T_{\mathcal{A}(f)}:=\sum_{i=0}^{T^{\mathrm{gen}}_{\mathcal{A}(f)}}\mu(i)$.
For algorithms that make randomised decisions like EAs, we are often interested
in the expected running time $\expect{T_{\mathcal{A}(f)}}$. For a class $\mathbb{F}$
of multi-objective functions, the expected running time on $\mathbb{F}$ is considered as
in the worst case performance:
\[
    \expect{T_{\mathcal{A}(\mathbb{F})}} = \max_{f\in\mathbb{F}} \expect{T_{\mathcal{A}(f)}}.
\]}%
\newedit{For a search point $x \in \{0,1\}^n$, $x_i$ denotes the $i$-th bit of
    $x$ for $i \in \{1,\dots,n\}$. Occasionally, we also use the bit indexing in a
    circular sense, \eg $x_k$ can be also the $(k \bmod n)$-th bit of $x$ and we identify $x_0$ with $x_n$.
    % which may result in $x_0$, in
    %that case we use the convention $x_0\equiv x_n$.
Let $\ones{x}$ be the number of ones in $x$, similarly $\zeros{x}$ is the number
of zeroes. Furthermore, we use $\LO(x)$, $\TOs(x)$, $\LZ(x)$ and  $\TZ(x)$
to denote, respectively, the numbers of leading ones, of trailing ones, of leading
zeroes and  of trailing zeros of $x$. For example, $x=111001011011$ has $x_4=0$,
$\ones{x}=8$, $\zeros{x}=4$, $\LO(x)=3$, $\TOs(x)=2$, $\LZ(x)=0$ and
$\TZ(x)=0$.
For two strings $x,y$ of the same length, the Hamming distance between them
is denoted $\Hamming(x,y)$.
For strings (or bit values) $x,y,z,\dots$, their concatenation into a single string
is denoted $(x,y,z,\dots)$.
Let $\sigma\colon[n]\rightarrow[n]$ be any permutation of $[n]$ (\ie $\sigma$ is a bijection),
then for any string $x \in \{0,1\}^n$ the output of permutating the bits of $x$
according to $\sigma$ is denoted $\permut{\sigma}{x}:=(x_{\sigma(1)},\dots,x_{\sigma(n)})$.}
\newedit{Let $\vec{1}$ be the vector of all ones, \ie $\vec{1}:=(1,1,\dots)$.
For any natural number $n$, $[n]$ denotes the set of natural numbers $\{1,\dots,n\}$.
The natural logarithm is denoted $\ln{n}$ and that of base $2$ is denoted $\log{n}$.
The symbol $\oplus$ denotes the XOR (exclusive OR) operator.}

%\andre{In the following we introduce the algorithms where crossover is either uniform crossover or 1-point crossover and mutation is either bitwise mutation with mutation rate $1/n$ or hypermutation (which hypermutation here? From Corus et al? please specify, I have no idea.}

\newedit{\subsection{Analytical Tools}}

\newedit{The well-known method of typical runs (see \cite{Wegener2002} Section~11)
will be used for our analysis of the algorithms, and we will detail this method
when it comes to the proofs. We also use drift analysis \citep{Jun2004}, which has become
a standard tool in runtime analysis~\citep{Lengler2020}.}%, in some of our arguments, and the
%following drift theorems will become handy.}

\newedit{
\begin{theorem}[Additive drift theorem \citep{Jun2004}]
\label{thm:additive-drift}
Let $(X_t)_{t\geq0}$ be a sequence of non-negative random variables with a
finite state space $S\subseteq \Real^+$ such that $0\in S$. Define
$T := \inf\{t \geq 0 \mid X_{t} = 0\}$.
\begin{enumerate}
\item[(i)] If $\exists \delta>0$ such that
$\forall s\in S\setminus\{0\}, \forall t\geq 0\colon
    \expect{X_{t}-X_{t+1}\mid X_{t}=s}\geq\delta$
then $\expect{T \mid X_0}\leq \frac{X_0}{\delta}$.
\item[(ii)] If $\exists\delta>0$ such that
$\forall s\in S\setminus\{0\}, \forall t\geq 0\colon
    \expect{X_{t}-X_{t+1}\mid X_{t}=s}\leq\delta$
then $\expect{T \mid X_0}\geq \frac{X_0}{\delta}$.
\end{enumerate}
\end{theorem}
%Note that the proof in \citet{Lengler2020} only claims, for example for (i),
%$\expect{T}\leq\expect{X_0}/\delta$ however if we know $X_0$ then this implies
%$\expect{T\mid X_0}\leq\expect{X_0\mid X_0}/\delta = X_0/\delta$, which is
%our statement for (i). Inversely by the tower law our statement also implies
%$\expect{\expect{T\mid X_0}} \leq \expect{X_0/\delta}=\expect{X_0}/\delta$,
%so the results are equivalent.}
Note that, by the tower rule, the above statements also imply $\expect{\expect{T\mid X_0}} \leq \expect{X_0/\delta}=\expect{X_0}/\delta$ for statement $(i)$ and likewise for $(ii)$ with a reversed inequality.}

\newedit{
\begin{theorem}[Negative drift theorem, \eg see \citep{Oliveto2011,Oliveto2012Erratum,Lengler2020}]
\label{thm:negative-drift}
For all $a,b,\delta,\eta,r>0$ with $a< b$,
there exist a $c>0$, a $n\in \Natural$ such that the following holds for all
$n\geq n_0$. Suppose $(X_t)_{t\geq 0}$ is a sequence of random variables with
a finite state space $S\in \Real^+$ with associated filtration $\filt{t}$.
Assume that $X_0\geq bn$, and let $T_a:=\min\{t\geq 0\mid X_t\leq an\}$ be
the hitting time of $S \cap [0,an]$.
Assume further that for all $s\in S$ with $s>an$, for all $j \in \Natural$
and for all $t\geq 0$ the following conditions hold:
\begin{enumerate}
\item[(a)]$\expect{X_t - X_{t+1} \mid \filt{t}, X_t=s} \leq -\delta$
\item[(b)]$\displaystyle
           \prob{|X_t - X_{t+1}| \geq j \mid \filt{t}, X_t = s} \leq \frac{r}{(1+\eta)^j}$
\end{enumerate}
Then $\prob{T_a \leq e^{cn}} \leq e^{-cn}$.
\end{theorem}}

%\newedit{The following bounds of the binomial coefficients are useful.
%\begin{lemma}[\todo{Add a reference}]\label{lem:binom-coeff}
%For natural numbers $n$ and $k\leq n$, it holds that
%\begin{enumerate}
%\item[(i)] $\displaystyle {n \choose k} \leq {n \choose n/2} \leq 4^{n/2}$,
%\item[(ii)] $\displaystyle {n \choose k}
%               \geq \frac{2^{n\entropy(k/n)}}{n+1}
%               \geq \left(\frac{n}{k}\right)^k$,
%\end{enumerate}
%where $\entropy(p)$ with $p\in[0,1]$ denotes the Shannon entropy of the Bernoulli
%distribution with parameter $p$, \ie $\entropy(p):=-p\log{p}-(1-p)\log(1-p)$.
%\end{lemma}}

\newedit{The following bounds on the amplified success rate in
$\lambda$ independent trials are useful.
\begin{lemma}[Lemma~10 in \cite{Badkobeh2015}]\label{lem:lambda-trials}
For every $p\in[0,1]$ and every $\lambda\in\Natural$,
\begin{align*}
1-(1-p)^\lambda \in \left[\frac{p\lambda}{1+p\lambda},\frac{2p\lambda}{1+p\lambda}\right].
\end{align*}
\end{lemma}}

\subsection{Algorithms}

The GSEMO algorithm
%with 1-point
%\andre{crossover}
is shown in Algorithm~\ref{alg:gsemo}.
Starting from a randomly generated solution, in each generation a new search point
$s'$ \newedit{is created as follows. With probability $p_c \in [0,1]$,
%then by
%bitwise mutation, with mutation rate $1/n$,
a crossover is performed on parents selected uniformly at random, and then mutation is applied to the outcome. Otherwise, only mutation is applied to a parent selected uniformly at random.}
%Mutation then flips each bit independently with probability $1/n$.
If the new offspring $s'$ is not dominated by any solutions of the current population $P_t$
then it is inserted, and all search points weakly dominated by $s'$ are removed from the
population. Thus the population size $|P_t|$ may vary over time.

\begin{algorithm}[t]
    \begin{algorithmic}[1]
    \STATE Initialize $P_0:=\{s\}$ where $s \sim \Unif(\{0,1\}^n)$;
    \FOR{$t:= 0$ \TO $\infty$}
    		\STATE{Sample $p_1 \sim \Unif(P_t)$;}
        	\STATE{Sample $u \sim \Unif([0,1])$;}
        	\IF{($u<p_c$)}
            	\STATE{Sample $p_2 \sim \Unif(P_t)$;}
            	\STATE{Create $s$ by \newedit{crossover} between $p_1$ and $p_2$;}
        	\ELSE
            	\STATE{Create $s$ as a copy of $p_1$;}
        	\ENDIF
        \STATE{\newedit{Create $s'$ by mutation on $s$};} %with rate $1/n$;}
        \IF{($s'$ is \NOT dominated by any individual in $P_t$)}
        	\STATE{Create the next population $P_{t+1} := P_t \cup \{s'\}$;}
            \STATE{Remove all $x \in P_{t+1}$ %individuals
            weakly dominated by $s'$;} % from~$P_{t+1}$;}
        \ENDIF
    \ENDFOR
    \end{algorithmic}
    \caption{GSEMO Algorithm}
    \label{alg:gsemo}
\end{algorithm}

The NSGA-II \citep{Deb2002,NSGAIICode2011} %1-point crossover
%\andre{crossover}
is shown in Algorithm~\ref{alg:nsga-ii}. In each generation, a population
$Q_t$ of $\mu$ offspring is created by \newedit{applying binary tournament selection for parent selection and then applying crossover
and mutation or just mutation, according to a crossover probability $p_c \in [0, 1]$.} Note that the \newedit{crossover operator}
%, if applied with probability
%$p_c \in (0,1)$,
produces two offspring, and that the binary
tournaments (with replacement) in line~\ref{alg:nsga-ii:tournament} use the
same criteria as the sorting in line~\ref{alg:nsga-ii:survival-sort}.
The merged population $R_t$ of both parents and offspring is then partitioned
into layers $F^1_{t+1}, F^2_{t+1}, \dots$ of non-dominated solutions \newedit{such that $F_{t+1}^1$ contains all non-dominated solutions in $R_t$ and for $i \ge 2$, $F_{t+1}^i$ contains all non-dominated solutions in $R_t \setminus \{F_{t+1}^1 \cup \dots \cup F_{t+1}^{i-1}\}$.
Then}
crowding distances are computed within each layer. Search points of $R_t$ are
then sorted with respect to the indices of the layer that they belong to
as the primary criterion, and then with the computed crowding distances
as the secondary criterion. Only the $\mu$ best solutions of $R_t$ are kept in
the next generation.

\begin{algorithm}[!ht]
    \begin{algorithmic}[1]
    \STATE Initialize $P_0 \sim \Unif( (\{0,1\}^n)^{\mu})$;
    \STATE Partition $P_0$ into layers $F^1_0,F^2_0,\dots$ of non-dominated fitnesses,
           then for each layer $F^i_0$ compute the crowding distance $\cdist(x,F^i_0)$
           for each $x \in F^i_0$;
    \FOR{$t:= 0$ \TO $\infty$}
        \STATE Initialize $Q_t:=\emptyset$;
        \FOR{$i:= 1$ \TO $\mu/2$}
            \STATE Sample $p_1$ and $p_2$, each by a binary tournament\label{alg:nsga-ii:tournament};
                   %using lexicographical criteria $(1/i, \cdist(x,F^i_{t+1}))$;
            \STATE{Sample $u \sim \Unif([0,1])$;}
            \IF{($u<p_c$)}
                \STATE{Create $s_1, s_2$ by \newedit{crossover} on $p_1, p_2$;}
            \ELSE
                \STATE{Create $s_1, s_2$ as exact copies of $p_1, p_2$;}
            \ENDIF
            \STATE{Create $s'_1$ by \newedit{mutation} on $s_1$ with rate $1/n$;}
            \STATE{Create $s'_2$ by \newedit{mutation} on $s_2$ with rate $1/n$;}
            \STATE Update $Q_t:=Q_t \cup \{s'_1,s'_2\}$;
        \ENDFOR
        \STATE Set $R_t := P_t \cup Q_t$;
        \STATE Partition $R_t$ into layers $F^1_{t+1},F^2_{t+1},\dots$ of non-dominated fitnesses,
               then for each layer $F^i_{t+1}$ compute $\cdist(x,F^i_{t+1})$ for each $x \in F^i_{t+1}$;
        \STATE Sort $R_t$ lexicographically %in the descending order of
               by $(1/i, \cdist(x,F^i_{t+1}))$\label{alg:nsga-ii:survival-sort};
        \STATE Create the next population $P_{t+1} := (R[1],\dots,R[\mu])$;
    \ENDFOR
    \end{algorithmic}
    \caption{NSGA-II Algorithm \citep{Deb2002}}
    \label{alg:nsga-ii}
\end{algorithm}

%In essence, the crowding distance of a search point in a set is the Manhattan
%distance between its nearest neighbours in the set and normalized for each dimension.
%
For a set $S=(x_1,x_2,\dots,x_{|S|})$ of search points
%on a problem with
%$d$ objectives, \ie for each $i \leq |S|$, $f(x_i) = (f_1(x_i),f_2(x_i),\dots,
%f_d(x_i)) \in \Real^d$,
the crowding distances are calculated \newedit{by computing a distance function for each objective and then summing up these values for all objectives}. %\bahare{need polish}
Thus $S$ is sorted separately for each objective $f_k$ ($k\leq d$) and
the first and last ranked individuals are assigned an infinite crowding distance.
The remaining individuals are then assigned the differences between the values of $f_k$
of those ranked immediate above and below the search point and normalized
by the difference in $f_k$-values of the first and last ranked search points.
Let
$S_k=(x_{k_1},x_{k_2},\dots,x_{k_{|S|}})$
denote the elements of $S$ sorted in descending order w.\,r.\,t.\ $f_k$, then
$\cdist(x_i, S)
    := \sum_{k=1}^{d} \cdist_{k}(x_i, S)$ where
\begin{align*}
%\cdist(x_i, S)
%    &:= \sum_{k=1}^{d} \cdist_{k}(x_i, S), \label{eq:crowd-dist}
%\text{ and here }\\
%\!\!\!\!\!\!
\cdist_{k}(x_{k_i}, S)
    &\!:= \! \begin{cases}
        \infty \!\!\!\! & \text{if } i \in \{1, |S|\},\\
        \frac{f_k\left(x_{k_{i-1}}\right) - f_k\left(x_{k_{i+1}}\right)}{f_k\left(x_{k_1}\right) - f_k\left(x_{k_{|S|}}\right)} \!\!\!\! & \text{otherwise.}
    \end{cases}%\!\!\!\!\!\!\!\!\label{eq:crowd-dist-eachdim}
\end{align*}

\newedit{The following lemma shows that with a sufficiently large population size, search points of the first-ranked layer
%, with a specific lattice structure in the objective space,
are protected between generations. No assumption about the search space is required, thus the result also holds for search spaces other than bit strings. The proof generalises the arguments from~\citet{ZhengLuiDoerrAAAI22}
which correspond to the specific application of the lemma with $m=n+1$.
\begin{lemma}\label{lem:nsga-ii-protect-layer}
	Let $t$ be one generation of the NSGA-II optimising a biobjective function $g(x):=(g_1(x),g_2(x))$. Suppose that $F_t^1$ and $F_{t+1}^1$ cover at most $m$ distinct fitness vectors. % comes from the fact that there are at most $n$ non-dominated solutions.
	Then the following holds.
	\begin{itemize}
		\item [(i)]
		At most $4m$ individuals in $F_t^1$ have positive crowding distance.
		\item [(ii)]
		If $\mu \geq 4m$ then for every $x \in F_{t+1}^1$ there exists a $y \in P_{t+1}$ with $g(y)=g(x)$. %Trivial für t=0
	\end{itemize}
\end{lemma}
%I changed the lemma a little bit since the order of computing the particular sets is P_0,F_0,R_0,F_1,P_1,R_1,F_2, \ldots
\begin{proof}
	\item [(i)]
	Note that $\vert{M}\vert \leq m$ for $M:=\{(g_1(x),g_2(x)) \mid x \in F_t^1\}$. To obtain the result it suffices to show that for each $(a,b) \in M$ at most four individuals in $F_t^1$ with fitness vector $(a,b)$ have a positive crowding distance. Let $K:=\vert{F_t^1}\vert$ and assume that $x_1, \ldots ,x_K$ are the individuals in $F_t^1$ and let $S_1=(x_{1_1}, \ldots , x_{1_K})$ and $S_2=(x_{2_1}, \ldots , x_{2_K})$ be the sorting of $F_t^1$ with respect to $f_1$ and $f_2$ respectively.
	Suppose there are $L\geq 1$ individuals in $F_t^1$ with fitness $(a,b)$. %,
	%for $L\leq 4$ our claim of the four individuals is trivial, thus we assume $L\geq 5$. %[cuong] I don't think we need this assumption
	By the definition of the sorting, then there must exist $r,s \in \{1,\dots,K-L+1\}$
	such that
	the subsequences $(x_{1_{r+i}})_{0\leq i\leq L-1}$ of $S_1$,
	and $(x_{2_{s+j}})_{0\leq j\leq L-1}$ of $S_2$ have that
	$f(x_{1_{r+i}})=f(x_{2_{s+j}})=(a,b)$,
	in other words, they are the $L$ individuals of $F^t_1$ with fitness $(a,b)$.
	%This implies that for a search point $x \in F^t_1$ with $f(x)=(a,b)$, if
	Furthermore,
	%if one of these individuals $x$ is not in
	%$x \notin \{x_{1_r},x_{1_{r+L-1}},x_{2_s},x_{2_{s+L-1}}\}$
	for each individual $x$ from these that is not in
	$\{x_{1_r},x_{1_{r+L-1}},x_{2_s},x_{2_{s+L-1}}\}$
	there exist
	$2\leq i \leq L-2$ and $2\leq j \leq L-2$ such that $x=x_i$ and $x=x_j$,
	and %we have
	%\begin{align*}
	$
	\cdist(x, S) %&=
	= \frac{f_1\left(x_{1_{i-1}}\right) - f_1\left(x_{1_{i+1}}\right)}{f_k\left(x_{1_1}\right) - f_k\left(x_{1_K}\right)}%\\
	%&
	+ \frac{f_2\left(x_{2_{j-1}}\right) - f_2\left(x_{2_{j+1}}\right)}{f_2\left(x_{2_1}\right) - f_2\left(x_{2_K}\right)} = 0.
	$
	%\end{align*}
	This means only points in $\{x_{1_r},x_{1_{r+L-1}},x_{2_s},x_{2_{s+L-1}}\}$
	can have positive crowding distances, and the claim follows by noting that the
	cardinality of this set is at most $4$.
	\item [(ii)] Since $F_{t+1}^1$ covers at most $m$ distinct fitness vectors, it follows from (i) and $\mu \geq 4m$ that $P_{t+1}$ contains every search point from $F_{t+1}^1$ with positive crowding distance. The statement follows since for every $x \in F_{t+1}^1$ there is $y \in F_{t+1}^1$ with positive crowding distance %(with respect to $F_{t+1}^1$)
	and $g(x)=g(y)$.
\end{proof}
The statement of Lemma \ref{lem:nsga-ii-protect-layer} holds for any
implementation of the calculation of crowding distances. However, %We notice that
the factor $4$ in front of $m$ there can be reduced to $2$ if a stable sort
algorithm is used to sort each objective, or
the asymmetric version of the crowding distance \citep{Chu2018} is used.}
%
%Our model is more restrictive than the one in \citet{Doerr2017} because we do not
%allow adaptive initialisation.

\subsection{The Operators}

\newedit{We consider two well-known crossover operators in our study. In
one-point crossover, which is described in Algorithm~\ref{alg:1pts-crossover},
a cutting point $c$ is chosen uniformly at random. Then the offspring inherits
the first $c$ bit values from one parent and the last $n-c$ bit values from the
other parent.
In the
%one parent from location $1$ to location $c$, while the rest of the bit values
%come from the other parent at the remaining locations. The other operator is the
uniform crossover operator each bit value is chosen independently from one of
the parents chosen uniformly at random, \ie Algorithm~\ref{alg:unif-crossover}.
% in the  parents has equal
%probability of transferring to the offspring.
The algorithms, which are illustrated by the pseudocode, are the two-offspring
versions of these operators, \ie two offspring solutions $z$ and $\bar{z}$ are
produced, and are used by NSGA-II. In the one-offspring versions, which are used
by GSEMO, one output offspring is chosen uniformly at random from $\{z,\bar{z}\}$.
These operators are known as \emph{binary variation operators} as two inputs are
required to produce
%recombined to create
the offspring solution(s).}
%\todo{We also need the two offspring version for these operators.}

%\dirk{I would have chosen $n-1$ cutting points, excluding those where a whole string is copied.}[done]

\begin{algorithm}[!ht]
    \begin{algorithmic}[1]
    \STATE Choose a cutting point $c \sim \Unif(\{0,\dots,n\})$;
    \STATE Create $z=(z_1,\dots,z_n)$ and $\bar{z}=(\bar{z}_1,\dots,\bar{z}_n)$\\
           where $z_i=x_i$ for $i\leq c$ and otherwise $z_i=y_i$,\\
           and $\bar{z}_i=y_i$ for $i\leq c$ and otherwise $\bar{z}_i=x_i$;
    \RETURN Strings $z$ and $\bar{z}$;
    \end{algorithmic}
    \caption{One point crossover between strings $x,y \in \{0,1\}^n$}
    \label{alg:1pts-crossover}
\end{algorithm}

\begin{algorithm}[!ht]
    \begin{algorithmic}[1]
    \STATE Sample $z=(z_1,\dots,z_n)$ and $\bar{z}=(\bar{z}_1,\dots,\bar{z}_n)$\\
    where independently $z_i \sim \Unif(\{x_i, y_i\})$,\\
    then $\bar{z}_i=x_i$ if $z_i=y_i$ or otherwise $\bar{z}_i=y_i$, for each $i\in[n]$;
    \RETURN Strings $z$ and $\bar{z}$;
    \end{algorithmic}
    \caption{Uniform crossover between strings $x,y \in \{0,1\}^n$}
    \label{alg:unif-crossover}
\end{algorithm}

\newedit{As for the mutation operators, we first consider standard bit
mutation which is shown in Algorithm~\ref{alg:bitwise-mutation}. In this
operator each bit is independently copied from the parent to the offspring but
has a small probability $1/n$ of being flipped (inverted). This operator is a \emph{global
operator}, in the sense that any search point can be reached, \ie has a positive
probability of being the output, from a given parent.
}

\brandnewedit{When analysing evolutionary algorithms on the proposed royal road functions, we will encounter scenarios in which there is no fitness gradient on a part of the bit string. This implies that a sequence of standard bit mutations will tend to create a nearly even balance between zeros and ones in the considered part, regardless of the initial search point. Since we will use this observation in several proofs, we give a lemma bounding the expected time for reaching a state where zeros and ones are nearly balanced.}
%\newedit{We have the following result about repeatedly applying
%standard bit mutation on a bit string.

\begin{lemma}\label{lem:repeat-bitwise-mutation}
\newedit{Consider a sequence of search points $x(0), x(1), x(2), \dots$ from $\{0, 1\}^n$ such that $x(t+1)$ results from applying a standard bit mutation (Algorithm~\ref{alg:bitwise-mutation}) to $x(t)$.
%on a bit string $x$ of length $n$, thus we will index it by time as $x(t)$,
For every $x(0)$ and every constants
%For any constant
$\varepsilon, c \in(0,1)$ the following holds.
Let $\subsetL\subseteq [n]$ be a set of fixed locations in $x(t)$ with $|\subsetL|=cn=:m$
    and let $\onesL{t}$ be the number of ones in $x(t)$ from these locations,
    define $\maxOZL{t}:=\max\{\onesL{t},m-\onesL{t}\}$, \ie the maximum between the numbers of ones and of zeroes in $x(t)$ at $\subsetL$.
%$\strssL{x(t)}$ be a string formed by collecting of the bits of $x(t)$ from these locations.
%and a set $\strssL{x(t)}$ of bits of $x(t)$ at some fixed locations where $|\strssL{x(t)}|=cn=:m$
%\andre{and a subsequence $\strssL{x(t)}$ of $x(t)$ where $|\strssL{x(t)}|=cn=:m$
%for some constant $c\in(0,1)$.}\cuong{``subsequence'' is a weaker notion than subset
%and it is important for our results to hold to talk about a subset}
%\andre{your assumption ''$\strssL{x(0)}=1^m$'' should be ''$\strssL{x(0)} \geq m/2$''?}\cuong{No, $m$ is the largest distance}
%Define
%  $T_1:=\min\{t\mid \ones{\strssL{x(t)}}\leq (1+\varepsilon)m/2\}$ and
%  $T_2:=\min\{t\mid \max\{\ones{\strssL{x(t)}},\zeros{\strssL{x(t)}}\}\leq(1+\varepsilon)m/2\}$,
If
  $T_1:=\min\{t\mid \onesL{t}\leq (1+\varepsilon)m/2\}$ and
  $T_2:=\min\{t\mid \maxOZL{t}\leq(1+\varepsilon)m/2\}$,
then $\expect{T_1} \leq \expect{T_2} = \bigO(n)$, and}
%
%\newedit{This also holds if we swap the roles of $0$ and $1$
%in the definition of $T_1$ and $T_2$.}
\newedit{this also holds if we swap the roles of $1$ with $0$
in the definition of $\onesL{t}$.}
%, \ie
%starting with $\strssL{x(0)}=0^m$ %\andre{$\strssL{x(0)} \geq m/2$}\cuong{well, $\strssL{x(0)}$ is a string}
%and defining
%$T_1:=\min\{t\mid \zeros{\strssL{x(t)}}\leq (1+\varepsilon)m/2\}$.
%\andre{But in our algorithms below we don't start at $0^m$ or $1^m$ in general.}
%\dirk{Can't we just drop all assumptions on the initialisation? The statement should hold for all initial search points.}
\end{lemma}
\begin{proof}
%\newedit{The result for $\expect{T_1}$ follows from the additive drift theorem. Define $X_t:= \ones{\strssL{x(t)}}$ with support $[0,m]$,
%\begin{align*}
%X_t:=
%  \begin{cases}
%    \ones{\strssL{x(t)}} &\text{ if } \ones{\strssL{x(t)}}>(1+\varepsilon)m/2\\
%    0             &\text{ otherwise }.
%  \end{cases}
%\end{align*}
\newedit{The first inequality follows by noting that $\maxOZL{t}\leq (1+\varepsilon)m$
implies $\onesL{t}\leq (1+\varepsilon)m$, therefore $T_1 \le T_2$ and so $\expect{T_1}\leq \expect{T_2}$.}

\newedit{To prove the bound for $\expect{T_2}$, we compare the expected change (called \emph{drift}) of the two
processes
    $(\onesL{t})_{t\geq 0}$ with support $[0,m]$ and
    $(\maxOZL{t})_{t\geq 0}$ with support $[m/2,m]$:
    $\delta_t(s):= \expect{\onesL{t} - \onesL{t+1}\mid \onesL{t}=s}$ with
    $\sigma_t(s):= \expect{\maxOZL{t} - \maxOZL{t+1}\mid \maxOZL{t}=s}$.
%}
%
%then the number of ones and zeroes in $\strssL{x(t)}$ being flipped follows
%binomial laws $\Bin(X_t,1/n)$ and $\Bin(m-X_t,1/n)$, respectively. Therefore, for all
%states $s>(1+\varepsilon)m/2$, we have a drift
%\begin{align*}
%\delta_t(s):=
%\expect{X_t - X_{t+1}\mid X_t=s}
%  = (1/n)\left(s - (m-s)\right)
%  = 2s/n - c
%  > (1+\varepsilon)c - c
%  = c\varepsilon.
%\end{align*}
%Therefore on the process $X'_t$ where $X'_t=X_t$ if $X_t>(1+\varepsilon)m/2$
%and $X'_t=0$ otherwise, we also have for all
%$s>0\colon \expect{X'_t - X'_{t+1}\mid X'_t=s}\geq \delta_t(s)\geq c\varepsilon$.
%Theorem~\ref{thm:additive-drift} then implies that
%$\expect{T_1 \mid x(0)} \leq X'(0)/(c\varepsilon) = \bigO(n)$.}
%
%\newedit{
%To prove the bound for $\expect{T_2}$, we compare the drift of the process
%$(X_t)_{t\geq0}$ above with that of
%\begin{align*}
%Y_t:=
%  \begin{cases}
%    \max\{\ones{\strssL{x(t)}},\zeros{\strssL{x(t)}}\} &\text{ if } \max\{\ones{\strssL{x(t)}},\zeros{\strssL{x(t)}}\}>(1+\varepsilon)m/2\\
%    0             &\text{ otherwise},
%  \end{cases}
%\end{align*}
%$Y_t:=\max\{\ones{\strssL{x(t)}},\zeros{\strssL{x(t)}}\}$ with support $[m/2,m]$,
%The drift of this process is
%thus define
%$
%\sigma_t(s):=
%    \expect{Y_t - Y_{t+1}\mid Y_t=s}
%$.
Note that by changing the indices, \ie by setting $k:=s-i$, we can write
$\delta_t(s)$ and $\sigma_t(s)$ as:
\begin{align*}
\delta_t(s)
    &= \sum_{i=-(m-s)}^{s} i \Prob{\onesL{t} - \onesL{t+1}=i \mid \onesL{t}=s}
     = \sum_{k=0}^{m} (s-k) \Prob{\onesL{t+1}=k \mid \onesL{t}=s},\\
\sigma_t(s)
    &= \sum_{i=-(m-s)}^{s-m/2} i \Prob{\maxOZL{t} - \maxOZL{t+1}=i \mid \maxOZL{t}=s}
     = \sum_{k=m/2}^{m} (s-k) \Prob{\maxOZL{t+1}=k \mid \maxOZL{t}=s}.
\end{align*}}
%Note that at any time $t$ with $X_t>0$, the two processes are essentially
%the same and they only start to differ when at a time $t+1$ we have
%$\ones{\strssL{x(t+1)}}<m/2$, for which then $X_{t+1}=\ones{\strssL{x(t+1)}}$ while
%$Y_{t+1}=\zeros{\strssL{x(t+1)}} = m-X_{t+1}$. %\ones{\strssL{x(t+1)}}$.
%\andre{This is a computation we should discuss, because it seems to me that there is some issue: The last inequality does not hold since you could also start at $s=(1+\varepsilon)m/2+1$ and reduce $Y_t$ by this value with probability $\Omega(1)$ by flipping a one-bit/zero-bit. But we have that (I think you mean $\sigma_t(s) = s - Y_{t+1}$ and $\delta_t(s) = s-X_{t+1}$) $\Pr(\sigma_t(s) = i) \geq \Pr(\delta_t(s)=i) - O(n^{-\Omega(n)})$ for every possible $i$ and thus $\expect{Y_t-Y_{t+1} \mid Y_t = s} = \sum_{i=-(m-s)}^{s} i \Prob{\sigma_t(s)=i} \geq \sum_{i=-(m-s)}^{s} i \Prob{\delta_t(s)=i} - o(1) \geq c\varepsilon - o(1)$.[done]}
%\cuong{Yes, I have fixed this, so we have to consider the original process first. By the way, $\sigma_t(s) = s - \expect{Y_{t+1}\mid Y_{t}=s}$ so not exactly the way you wrote, also you have to explain where the $O(n^{-\Omega(n)})$ comes from in your statement.}

\newedit{We first claim that
    for all $s> (1+\varepsilon)m/2$
    and all $k \in [m/2,m]$, it holds that
    $\Prob{\onesL{t+1}=k \mid \onesL{t}=s}\leq\Prob{\maxOZL{t+1}=k \mid \maxOZL{t}=s}$.
Assume $k>m/2$, then if
%Specifically, for $s> (1+\varepsilon)m/2$ and $k \in [m/2+1,m]$, and if
$\onesL{t}\geq m/2$ we get $\maxOZL{t}=\onesL{t}$ and %\andre{I would distinct $k>m/2$ and $k = m/2$. Otherwise the second equality is not correct[done].}
\begin{align*}
\Prob{\maxOZL{t+1}=k \mid \maxOZL{t}=s}
  &= \Prob{\maxOZL{t+1}=k \mid \onesL{t}=s} \\
  &= \Prob{\onesL{t+1}=k \mid \onesL{t}=s} + \Prob{\onesL{t+1}=m-k \mid \onesL{t}=s}\\
  &\geq \Prob{\onesL{t+1}=k \mid \onesL{t}=s}.
\end{align*}
Otherwise if $\onesL{t}<m/2$ we have $\maxOZL{t}=m - \onesL{t}$ and
\begin{align*}
\Prob{\maxOZL{t+1}=k \mid \maxOZL{t}=s}
    &= \Prob{\maxOZL{t+1}=k \mid \onesL{t}=m-s} \\
    &= \Prob{\onesL{t+1}=k \mid \onesL{t}=m-s} + \Prob{\onesL{t+1}=m-k \mid \onesL{t}=m-s}\\
    &\geq \Prob{\onesL{t+1}=m-k \mid \onesL{t}=m-s}
    %\\
    %&
    %= \Prob{m-\onesL{t+1}=k \mid m-\onesL{t}=s}
    = \Prob{\onesL{t+1}=k \mid \onesL{t}=s},
\end{align*}
where the last equality follows from the symmetry of zeros and ones.
%the binomial distribution.
%that the distribution of
%$m-X_{t+1}$ conditioned on $m-X_{t}=s$ is $s - \Bin(s,1/n) + \Bin(m-s,1/n)$ and
%this is the same as the distribution of $X_{t+1}$ conditioned on $X_t=s$.
When $k=m/2$ we also have that $\Prob{\maxOZL{t+1}=k \mid \maxOZL{t}=s} = \Prob{\onesL{t+1}=k \mid \onesL{t}=s}$.
These prove the claim.}

%In all cases, $\Prob{X_{t+1}=k \mid X_t=s}\leq\Prob{Y_{t+1}=k \mid Y_t=s}$ and
\newedit{Hence for any $s> (1+\varepsilon)m/2$, we get:
\begin{align*}
\delta_t(s)
  &= \sum_{k=0}^{m} (s-k) \Prob{\onesL{t+1}=k \mid \onesL{t}=s}\\
  &= \sum_{k=0}^{m/2-1} (s-k) \Prob{\onesL{t+1}=k \mid \onesL{t}=s}
     + \sum_{k=m/2}^{m} (s-k) \Prob{\onesL{t+1}=k \mid \onesL{t}=s}\\
  &\leq \sum_{k=0}^{m/2-1} m \Prob{\onesL{t+1}=k \mid \onesL{t}=s}
     + \sum_{k=m/2}^{m} (s-k) \Prob{\maxOZL{t+1}=k \mid \maxOZL{t}=s}\\
  &\leq (m/2)m {s \choose \varepsilon m/2} (1/n)^{\varepsilon m/2}
      + \sigma_t(s),%\\
  %&\leq \sigma_t(s) + O(n^2)2^{cn}n^{-\Omega(n)} = \sigma_t(s) + o(1),
\end{align*}
where the last inequality follows since at least $\varepsilon m/2$ ones
(among $s$) at locations $\subsetL$ need to be flipped to have $\onesL{t+1}< m/2$ given that
$\onesL{t}=s> (1+\varepsilon)m/2$. Using $\binom{s}{\varepsilon m/2} \le 2^s \le 2^m = 2^{cn}$ we have
$(m/2)m {s \choose \varepsilon m/2} (1/n)^{\varepsilon m/2}
\leq O(n^2) 2^{cn} n^{-\Omega(n)} = o(1)$.}
\newedit{To compute $\delta_t(s)$, note that
the number of ones and zeroes being flipped at $\subsetL$ follows
binomial laws $\Bin(\onesL{t},1/n)$ and $\Bin(m-\onesL{t},1/n)$, respectively.
Therefore, for all states $s>(1+\varepsilon)m/2$, we have %a drift
\begin{align*}
\delta_t(s)
%:= %\expect{X_t - X_{t+1}\mid X_t=s}
  = (1/n)\left(s - (m-s)\right)
  = 2s/n - c
  > (1+\varepsilon)c - c
  = c\varepsilon.
\end{align*}
%
%and this implies
Thus overall, we get $\sigma_t(s) \geq \delta_t(s)
- o(1) \geq c\varepsilon - o(1)$. Now consider the process
$Z_t$ which has $Z_t=\maxOZL{t} - m/2$ if $\maxOZL{t}> (1+\varepsilon)m/2$ and
$Z_t=0$ otherwise (\ie $\maxOZL{t}\in \{0\} \cup [m/2,(1+\varepsilon)m/2]$), then
$\expect{Z_t - Z_{t+1}\mid Z_t=s} \geq \sigma_t(s) \geq c\varepsilon-o(1)$.
By Theorem~\ref{thm:additive-drift} we have that
$\expect{T_2 \mid x(0)}\leq (m/2)/(c\varepsilon - o(1)) = \bigO(n)$ for any $x(0)$.}

\newedit{The final remark follows from the symmetry of zeros and ones.}
%
%The results for $\strssL{x(0)}=0^m$ and $T_1:=\min\{t\mid \zeros{\strssL{x(t)}}\leq (1+\varepsilon)m/2\}$
%follow by exactly the same arguments after defining
%$X_t:= \zeros{\strssL{x(t)}}$.}
\end{proof}

\newedit{Standard bit mutation is a so-called
\emph{unbiased unary variation operator}~\citep{Lehre2012}. The term \emph{unary} means that only one input search point is used to produce the offspring, and this holds
true for any mutation operator.
%as contrast to crossover, \eg the uniform crossover
%is a binary unbiased variation operator.
The term \emph{unbiased} means that
the operator treats all bit positions $\{1,\dots,n\}$ equally,
as well as the bit values $\{0,1\}$. The formal definition of unbiasedness
can be found in \citep{Lehre2012}, but algorithms that only make use of
unbiased variation operators are resistant to the following modifications of
any function that can shuffle the search space. Let $f\colon \{0,1\}^n\rightarrow
\Real$ be any pseudo-Boolean function on bit strings of length
$n$, $\sigma\colon[n]\rightarrow[n]$ be any permutation of $[n]$
%(\ie $\sigma$ is a bijection),
and a fixed $z\in \{0,1\}^n$, then an algorithm that
only make use of unbiased variation operators behaves the same on $f(x)$ as on
$g(x):=f(\permut{\sigma}{x} \oplus z)$~\citep{Lehre2012}. Furthermore, it is easy
to see that this statement extends to the multi-objective setting, \ie for any
$f\colon \{0,1\}^n\rightarrow \Real^d$.
Owing to Lemma~1 in
\citet{Doerr2020}, every unary unbiased variation operator
%$\mathrm{op}(x)$
on $\{0,1\}^n$ can be modelled as a two-step process:
first choose a Hamming radius $r$ from some distribution, then return a search
point on the Hamming sphere $S_r(x):=\{y \in \{0,1\}^n\mid H(x,y)=r\}$ chosen
uniformly at random. This is summarised in Algorithm~\ref{alg:unary-unbiased-op}.}

\begin{algorithm}[!ht]
    \begin{algorithmic}[1]
    \STATE Sample $z=(z_1,\dots,z_n)$ where independently
    for each $i\in[n]\colon$
        \[
        z_i = \begin{cases}
               x_i & \text { with prob. } 1-1/n,\\
               1-x_i & \text{ with prob. } 1/n
               \end{cases}
        \]
    \RETURN String $z$;%String ($z_1$,\dots,$z_n$) where each $z_i \sim \{x_i, 1-x_i\}$ with probabilities $\{1-1/n,1/n\}$ respectively;
    \end{algorithmic}
    \caption{Standard bit mutation of a string $x \in \{0,1\}^n$}
    \label{alg:bitwise-mutation}
\end{algorithm}

\begin{algorithm}[!ht]
    \begin{algorithmic}[1]
    \STATE Sample $r$ from some distribution on $\{0, \dots, n\}$;
    \STATE Sample string $z\sim\Unif(S_r(x))$ where $S_r(x):=\{y \in \{0,1\}^n\mid H(x,y)=r\}$;
    \RETURN String $z$;
    \end{algorithmic}
    \caption{Unary unbiased variation operator applied on string $x \in \{0,1\}^n$}
    \label{alg:unary-unbiased-op}
\end{algorithm}

\newedit{It has been shown that some
mutation operators can imitate the one-point crossover operator, such as the
somatic contiguous hypermutation operator. \brandnewedit{This operator has been introduced by~\citet{KelseyTimmis2003} and we consider the refined operator presented as CHM$_3$ in~\citet{Zarges2011}.} We therefore
include this operator (referred simply as \emph{hypermutation}) in our study, and
allow the algorithms to use it in place of standard bit mutation.
The hypermutation operator is described in Algorithm~\ref{alg:hypermutation}. It is parametrised by a parameter $r\in(0,1]$ and the input string is copied to the
ouput, except for a contiguous region of the string (in a circular sense),
starting from a random location $p$ and with a random length $\ell$. In this region bit
values are flipped independently with probability $r$.
\brandnewedit{In the special case of two complementary parents and $r=1$, a hypermutation produces the same distribution of offspring as a two-point crossover (where two cutting points are chosen and substrings are chosen from alternating parents). If, additionally, one end of the interval coincides with one end of the bit string, this is equivalent to a one-point crossover.}
}\dirk{Added this text after Andre's talk.}

\newedit{%Note that,
    Even though every bit has the same probability $r/2$ of being flipped
by this operator (see Lemma~4 in \citep{Zarges2011}), hypermutation is
not an unbiased variation operator. This is due to the sense of consecutiveness
induced by the operator: conditioned on flipping a bit position $i$,
there is a higher chance for the bits next to $i$ both to the left and to the right
to be flipped, compared to the other positions.}

\begin{algorithm}[!ht]
    \begin{algorithmic}[1]
    \STATE Choose a location $c \sim \Unif(\{1,\dots,n\})$;
    \STATE Choose a length $\ell \sim \Unif(\{0,1,\dots,n\})$;
    \IF{$\ell=0$}
        \STATE Create $z$ as an exact copy of $x$;
    \ELSE
        \STATE Sample $z=(z_1,\dots,z_n)$ where $z_i=x_i$ for $i \notin \{c,\dots,c+\ell-1\bmod n\}$
               and otherwise
               \[
                z_i = \begin{cases}
                                x_i & \text { with prob. } 1-r,\\
                              1-x_i & \text{ with prob. } r
                              \end{cases}
                \]
    \ENDIF
    \RETURN String $z$;
    \end{algorithmic}
    \caption{Somatic contiguous hypermutation with parameter $r \in (0,1]$ of string $x \in \{0,1\}^n$}
    \label{alg:hypermutation}
\end{algorithm}

\subsection{The General Framework of Elitist Black-box Algorithms}

\newedit{In the following, we aim to show that crossover incorporated in GSEMO and NSGA-II leads to efficient runtimes on our proposed royal road functions, while algorithms without crossover fail badly.
To this end, and to derive negative results that are as general as possible, we use a general model for elitist black-box algorithms, similar to the one introduced by~\citet{Doerr2017}.}

%In a nutshell, NSGA-II is a standard \muplusga with $\lambda=\mu$ generalized
%to as multiple-objective setting. This is done by imposing a total order on
%the objective space, \ie by sorting the population into layers of non-dominated
%fitnesses and further ordering search points within a layer using the crowding
%distance measure. The algorithm therefore fits in the elitist black-box model
%of Algorithm~\ref{alg:muplus-blackbox}.

\begin{algorithm}[tb]
    \begin{algorithmic}[1]
    \STATE Initialize $P_0 \sim \Unif( (\{0,1\}^n)^{\mu})$;
    \STATE Query the ranking $\rho(P_0,f)$ induced by $f$;
    \FOR{$t:= 0$ \TO $\infty$}
        \STATE Choose a probability distribution $D_t(P_t, \rho(P_t,f))$ on
               $\{\{0,1\}^{n}\}^{\lambda}$ which only depends on its two arguments;
        \STATE Sample $Q_t$ from $D_t$,
               and set $R_t := P_t \cup Q_t$;
        \STATE Query the ranking $\rho(R_t,f)$ induced by $f$;
        \STATE Sort $R_t$ according to $\rho(R_t,f)$;
        \STATE Set $P_{t+1} := (R[1],\dots,R[\mu])$;
    \ENDFOR
    \end{algorithmic}
    \caption{Elitist ($\mu$+$\lambda$) black-box algorithm.}
    \label{alg:muplus-blackbox}
\end{algorithm}

In this model, shown in Algorithm~\ref{alg:muplus-blackbox}, adapted from~\citet{Doerr2017}, the algorithm keeps $\mu$ \emph{best} solutions,
%based on the
according to the
ranking function $\rho(P,f)$, it has seen so-far and can only
sample new offspring solutions %by mutating
based on
these elitist solutions. The ranking
function is deterministic and provides a ranking of all search points seen so far.

Note that NSGA-II can be seen as a \muplusga black-box algorithm with $\lambda=\mu$. A ranking function is obtained by imposing a total order on
the objective space, \ie by sorting the population into layers of non-dominated
fitnesses and further ordering search points within a layer using the crowding
distance measure, as explained before. Hence, NSGA-II fits in the elitist black-box model
of Algorithm~\ref{alg:muplus-blackbox}.
%NSGA-II uses the non-dominated sorting and the crowding distance measure
%for the ranking function, but other algorithms can have different choices.

%\section{The Multi-Objective Royal Road Function \andre{for 1-Point Crossover}}\label{sec:rrrmo} % the function and its properties (eg. pareto-front, size of layer, etc)
\section{A Multi-Objective Royal Road Function for One-Point Crossover}
%The Multi-Objective Royal Road Function \newedit{\RRRMO}}
\label{sec:rrrmo} % the function and its properties (eg. pareto-front, size of layer, etc)

\newedit{Now we introduce a multi-objective royal road function designed to show the benefits of one-point crossover. The design of this function is deliberately simple as we believe that this best illustrates
%\newedit{Our first multi-objective royal road function is relatively simple to
%illustrate
the necessity of the one-point crossover operator.
Our construction is inspired by the ``real roal road'' function for one-point crossover in single-objective optimisation by~\citet{Jansen2005c}. However, while their function contains a single global optimum at $1^n$, our function features a set of Pareto-optimal search points at a linear Hamming distance from~$1^n$.}

\newedit{
The idea behind the
%construction
design of
our royal road
function is to encourage EMO %evolutionary
algorithms to evolve a specific number of ones in a search point~$x$, denoted as
$\ones{x}$, and then to evolve a prefix and a suffix of zeros.
%We define the number of leading zeros, $\LZ(x)$ as the length of the longest
%prefix in $x$ that contains only zeros. Similarly, the number of trailing zeros,
%$\TZ(x)$, gives the length of the longest suffix of zeros in $x$. For example,
%for $x = 0010110$
%00$
%we have %$\ones{x}=3$,
% $\LZ(x) = 2$,
 %and
%$\TZ(x) = 1$.
%
%\begin{definition}\label{def:lz-tz}
%The numbers of ones, of leading and trailing zeroes of a bit string $x$ are:
%\begin{align*}
%\OM(x) =& |x|_1 :=
%    \sum_{i=1}^{n} x(i) \\
%\LZ(x) :=&
%    \sum_{i=1}^{n} \prod_{j=1}^{i} \left(1 - x(j)\right) \\
%\TZ(x) :=&
%    \sum_{i=1}^{n} \prod_{j=1}^{i} \left(1 - x(n-j+1)\right)
%\end{align*}
%\end{definition}
%
This is achieved by introducing two conflicting objectives that involve maximising the number of leading zeros and trailing zeros, respectively, in particular for all search points with $3n/5$ ones.
%Our \RRRMO %Real Royal Road
%function %\dirk{To do: explain the term Real Royal Road}
%in multi-objective setting is described by a
%trade-off between leading zeros and trailing zeros for all search points with
%$3n/5$ ones.
There is an additional set of high-fitness search points with
$4n/5$ ones that can be created easily using one-point crossover whereas common
mutation operators require exponential time for the same task.}

\begin{definition}\label{def:rrrmo}
For all $n \in \mathbb{N}$ divisible by~$5$ the bi-objective function
$\RRRMO \colon \{0, 1\}^n \to \mathbb{N}_0^2$ is defined as follows.
Let
$\frontP := \{x \mid \ones{x} = 4n/5 \wedge \LZ(x) + \TZ(x) = n/5\}$
and
$\posFit := \{x \mid \ones{x} \le 3n/5\} \cup \frontP$, then
%\begin{align*}
%\RRRMO(x) &:=
%    (f_1(x), f_2(x))%\\
%\end{align*}
%where
%\begin{align*}
%%\text{ where }
%f_1(x) &:=
%    \begin{cases}
%         n|x|_1 + \textsc{TZ}(x) & \text{if } |x|_1 \le 3n/5 \vee |x|_1 = 4n/5\\
%        0 & \text{otherwise,}
%    \end{cases}\\
%%\text{ and }
%f_2(x) &:=
%    \begin{cases}
%         n|x|_1 + \textsc{LZ}(x) & \text{if } |x|_1 \le 3n/5 \vee |x|_1 = 4n/5\\
%        0 & \text{otherwise.}
%    \end{cases}
%\end{align*}
%\begin{align*}
%%\text{ where }
%& \RRRMO(x) := \\
%&    \!\!\!\begin{cases}
%         (n|x|_1 \!+\! \textsc{TZ}(x),\  n|x|_1 \!+\! \textsc{LZ}(x)) & \!\!\!\text{if } |x|_1 \in \left[0, \frac{3n}{5}\right] \cup \left\{\frac{4n}{5}\right\}\\
%        (0, 0) & \!\!\!\text{otherwise.}
%    \end{cases}
%\end{align*}
\begin{align*}
%\text{ where }
& \RRRMO(x) :=
   \begin{cases}
         (n\ones{x} + \TZ(x),\  n\ones{x} + \LZ(x)) \!\!\! & \text{if } x \in \posFit\\
        (0, 0) & \text{if } x \notin \posFit.
    \end{cases}
\end{align*}
\end{definition}
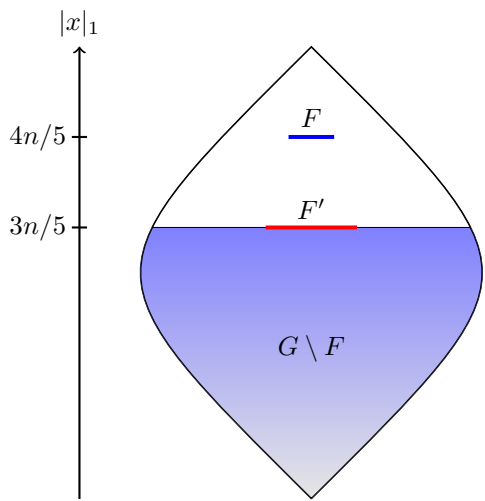
\begin{figure} % the search space of RRMO
\begin{center}
\begin{tikzpicture}%[scale=0.9,transform shape]
	\draw[->,thick] (-0.05,0) -- (-0.05,6) node[above] {$|x|_1$};
	\draw[thick] (-0.15,4.8) -- (0.05,4.8);
	\node at (-0.6,4.8) {$4n/5$};
	\draw[thick] (-0.15,3.6) -- (0.05,3.6);
	\node at (-0.6,3.6) {$3n/5$};
	
	\clip[draw]  (3,0) .. controls (0,3) .. (3,6) .. controls (6,3) .. (3,0);
	%\draw[fill=blue!20] (0.0,0.0) rectangle (6.0,3.6);
    \draw[top color=blue!50,bottom color=black!10] (0.0,0.0) rectangle (6.0,3.6);
	\draw[thick] (3,0) .. controls (0,3) .. (3,6) .. controls (6,3) .. (3,0);
	
	\draw[color=blue,line width=1.5] (2.7,4.8) -- (3.3,4.8);
	\node at (3,5.05) {$\frontP$};
	
	\draw[color=red,line width=1.5] (2.4,3.6) -- (3.6,3.6);
	\node at (3,3.85) {$\lfrontP$};
	
	\node at (3,2) {$\posFit \setminus \frontP$};
\end{tikzpicture}
\end{center}
\caption{A sketch of the function \RRRMO. Here the Boolean
hypercube $\{0,1\}^n$ is illustrated where the $y$ axis shows search points with the same number of
ones. The sketch illustrates the Pareto front~$F$, the set $F'$ and $G \setminus F$.}
%, thus it expands when the number of ones is closed to $n/2$.}
\label{fig:rrmo-function}
\end{figure}

Note that all $x \in \posFit$ strictly dominate all $y \notin \posFit$ as $f(0^n) = (n, n)$
and for $x \in \posFit \setminus \{0^n\}$ both objective values are at least
$n\ones{x} \ge n$.
Algorithms initialising their population uniformly at random will typically
start with search points having at most $3n/5$ ones, that is, only search points
in $\posFit$ that fall into the first case of Definition~\ref{def:rrrmo}. Then the
function gives a strong fitness signal to increase the number of ones. In fact,
every search point $x \in \posFit$ dominates all search points $y$ with $\ones{y} <
\ones{x}$. Every search point $x \in \frontP$ dominates all search points $y \notin
\frontP$.
Comparing two solutions $x, y \in \posFit$ with $\ones{x} = \ones{y}$, $x$ weakly
dominates $y$ if $\TZ(x) \ge \TZ(y)$ and $\LZ(x) \ge
\LZ(y)$; it strongly dominates $y$ if one of these inequalities is
strict. Thus, the set
\begin{align*}
\frontP
    &:= \{0^i 1^{4n/5} 0^{n/5-i} \mid 0 \le i \le n/5\}
\end{align*}
where all zeros contribute to either $\TZ(x)$ or $\LZ(x)$ is the
Pareto-optimal set for \RRRMO and all search points in
\begin{align*}
\lfrontP
    &:= \{0^i 1^{3n/5} 0^{2n/5-i} \mid 0 \le i \le 2n/5\}
\end{align*}
dominate all search points $y \notin \lfrontP \cup \frontP$.
\newedit{Figure~\ref{fig:rrmo-function} illustrates the structure of the search space
	that we have described.}

\newedit{By the design of the function, a solution in $\posFit$ with a larger number of ones always dominates those in $\posFit$ with less number of ones. Therefore, a set $S$ of non-dominated solutions in $G$ only contains solutions with the same number of ones.}
The following lemma bounds the number of non-dominated solutions
%contained in any population.
\newedit{in $S$ more precisely.}

\begin{lemma}\label{lem:size-non-dom-set}
If $S$ is a set of non-dominated solutions \newedit{in $\posFit$} of $\RRRMO$ \newedit{(i.\;e.\ for $x,y \in S$ \brandnewedit{with $x \neq y$} we neither have $x \succeq y$ nor $y \succeq x$ with respect to $\RRRMO$)} \newedit{with
%$k \in \{0, \ldots , n\}$
$k$ ones} %with positive fitness %then $|S| \leq n$.
then $|S| \leq n-k+1$.
%\dirk{Should we stress that $S$ must be a set without duplicates, i.e. no multiset? Otherwise the proof doesn't work.}
%\cuong{I think Andre wrote ``set'' here, but I am not sure if removing the duplicates resolves the issue because we are on the solution space. I have similar remarks with Andre before, it seems sometimes his statements only work on the co-domain rather than the domain.}
\end{lemma}

\begin{proof}
%By our previous observations, $S$ may only contain search points with the same
%number of ones, denoted by $k$.
For $k=0$ there is only one search point, thus
we assume $k \ge 1$. Consider $x \in S$ with $f(x) = (kn + i, kn + j)$. If there
is a search point $y \in S$, $y \neq x$, with $f(y) = (kn + i, kn + j')$ then
$y$ weakly dominates $x$ if $j' \ge j$ and otherwise $x$ weakly dominates $y$.
Thus, for every value~$i$ there can only be one search point in $S$ with an
$f_1$-value of $kn + i$. Since the range of \TZ is
%at most $n-1$ and the fact that there are at most $n-1$ zero
\newedit{$0, \ldots , n-k$}, the claim follows.
\end{proof}

%\begin{proof}
%	We need to find the maximal number of (pairwise) incomparable elements from $\{0,1\}^n$ with respect to \RRRMO. So let $x \in \{0,1\}^n$ be a search point with a maximum number of ones from $S$. Let $\vert{x}\vert_1=k$ for $k \in \mathbb{N}$. (Of course we can assume that $k>0$ since the all zero string is dominated by every other string.) Note that $\RRRMO|_S$ is an injection and $\RRRMO(S) \subset \{i+nk,j+nk \mid i,j \in \{0,...,n-1\}\}$ (where $k$ is fixed here). Additionally we have that $\vert{\RRRMO(S)}\vert \leq \vert{f_1(S)}\vert = n$ since two distinct elements from $\RRRMO(S)$ have two different first components. (Otherwise they would be comparable.) So $\vert{S}\vert \leq n$. Note also that this bound may be tight because there are exactly $n$ incomparable elements when $\ell=1$.
%
%\dirk{Suggest to refine the presentation as follows: we choose $x$ as a search point with a maximum number of ones from $S$ (i.e, $k$) and then argue that all solutions with fewer than $k$ ones are dominated by~$x$.}
%\end{proof}

%\section*{Necessity of recombination}\label{sec:lower-bound-pc-zeross} %in GSEMO and NSGA-II}
\subsection{Hardness \newedit{of \RRRMO} for EMOs %\andre{to optimize \RRRMO}
         without Crossover}\label{sec:lower-bound-pc-zeross} %in GSEMO and NSGA-II}

%\cuong{I am not sure how to formulate algorithms that no longer accept
%$(0,0)$ fitness, so for now I made two separated results like below.
%Perhaps we can think about the ($\mu+\lambda$) elitist black-box algorithm
%like in Carola+Johannes' paper, the tricky part is to find a way to fit
%GSEMO into that framework.}

We first show that disabling crossover by setting $p_c=0$ makes GSEMO and NSGA-II
highly inefficient on \RRRMO. Without crossover, both algorithms require exponential
time even to discover a first Pareto-optimal search point.

\begin{theorem}\label{thm:gsemo-pc-zero}\label{thm:nsga-ii-pc-zero}
\newedit{The following algorithms}
require at least $n^{\Omega(n)}$ evaluations in expectation to find any Pareto-optimal search point for
\RRRMO:
\begin{itemize}
\item GSEMO (Algorithm~\ref{alg:gsemo}) with $p_c=0$ \newedit{and standard bit mutation},
\item NSGA-II (Algorithm~\ref{alg:nsga-ii}) with $p_c=0$, $\mu\in\poly(n)$ \newedit{and standard bit mutation}.
\end{itemize}
%GSEMO (Algorithm~\ref{alg:gsemo}) with $p_c=0$ \newedit{and standard bit mutation}
%requires at least $n^{\Omega(n)}$ evaluations in expectation to find any Pareto-optimal search point for
%\RRRMO.
\end{theorem}
\begin{proof}
\newedit{We first show the result for GSEMO. By}
classical Chernoff bounds the probability of initialising the algorithm with
a search point of at most $3n/5$ ones, \ie a search point in $\posFit\setminus\frontP$,
is $1-2^{-\Omega(n)}$.
We assume in the following that this has happened and note that then the
algorithm will never accept a search point $s'$ with fitness $(0, 0)$,
\ie $s' \notin \posFit$.
%We consider two following phases of the run of the algorithm.
%
%\textbf{Initialisation}: This phase is skipped (and we optimistically assume
%the run is finished) if the algorithm is initialised with
%%the $0^n$ search point or %% note: this search point has fitness (n, n) so it is fine
%a search point with more than $3n/5$ ones.
%Thus by a Chernoff bound, the probability of skipping this phase is $2^{-\Omega(n)}$.
%Otherwise, this phase costs one iteration and the algorithm enters the
%second phase.
%
%\textbf{Creating a search point in \frontP}:
%%Here we optimistically
%%assume if a search point in $F$ is created then the algorithm is finished.
%%
%By the end of the previous phase, the algorithm has created
%search points with fitnesses $(f_1>0, f_2>0)$, thus it will never accept
%points with fitness $(0,0)$, which form the set
%We denote the latter points as the set
%    $G:=\{x \in \{0,1\}^n \mid |x|_1 > 3n/5\} \setminus \frontP$.
Furthermore, because $p_c=0$ the algorithm can only rely
on the standard bit mutation operator to generate a search point on $\frontP$.
Fix a search point $y \in F$, then for each search point $x \in \posFit \setminus \frontP$ the Hamming distance to $y$ is at least $H(x, y) \ge \ones{y} - \ones{x} \ge n/5$.
% the shortest Hamming distance between such a search point
%a search point in $\frontP$
%and any search
%point in $\posFit \setminus \frontP$ is $n/5$,
Therefore, flipping $n/5$ specific
bits is required to create $y$ as a mutant of~$x$, %in $\frontP$,
and this
occurs with probability $n^{-n/5}$. Taking a union bound over all
%search points in
$y \in \frontP$, the probability of creating any search point in $\frontP$ is at most
$|\frontP| \cdot n^{-n/5} = O(n^{-n/5 + 1})$.
\newedit{Thus, the expected time for this to happen is stochastically dominated by the expectation of a geometric random variable with parameter $O(n^{-n/5+1})$, which is $\Omega(n^{n/5-1})$.} By the law of total probability, the expected number of
evaluations required for this phase is at least
$
(1-2^{-\Omega(n)}) \cdot \Omega(n^{n/5-1})
    = n^{\Omega(n)}$.

\newedit{The proof for NSGA-II follows closely the above arguments, except for
the initialisation.
%    Theorem~\ref{thm:gsemo-pc-zero},
%with only minor differences due to the fixed population size.
%
%The initialisation phase can be skipped with probability $\mu \cdot 2^{-\Omega(n)}$
%by a union bound, but this is still $o(1)$ since $\mu \in \poly(n)$.
The probability of initialising the whole population of NSGA-II with
$\mu$ search points of at most $3n/5$ ones is at least
$1-\mu\cdot 2^{-\Omega(n)}=1-o(1)$
by a Chernoff and a union bound, and given $\mu \in \poly(n)$.
%
%The second phase
%uses exactly the same arguments, starting with $\mu$
%then starts with $\mu$
%search points of strictly positive objective values and
If this occurs, then afterwards the algorithm will never accept a search point
with fitness $(0,0)$ during the
%$(\mu+\mu)$
survival
selection. Therefore, flipping $n/5$ $0$s to $1$s by mutation is still
required to create the first Pareto-optimal solution, and overall the expected
number of %generations
fitness evaluations
is still at least
    $(1 - o(1))(1 + n^{\Omega(n)})=n^{\Omega(n)}$.
}
\end{proof}

Furthermore, we prove that \emph{every} algorithm from the general framework of Algorithm~\ref{alg:muplus-blackbox}
also requires exponential optimisation time in expectation to create
a first Pareto-optimal point of \RRRMO if only unary unbiased variation
operators are used. This larger generality comes at the expense of a slightly weaker lower bound where the base of the exponential function is~2 instead of~$n$.

\begin{theorem}\label{thm:elitist-blackbox-unary-unbiased}
Every black-box algorithm that fits the model of Algorithm~\ref{alg:muplus-blackbox}
with $\mu=\poly(n)$ and only uses unary unbiased variation operators \newedit{(Algorithm~\ref{alg:unary-unbiased-op})}
for choosing the distribution $D_t$ requires at least
    $2^{\Omega(n)}/\lambda$ generations, or $2^{\Omega(n)}$ fitness evaluations, in expectation
to find any Pareto-optimal
%search point on
point
of \RRRMO.
\end{theorem}
\begin{proof}
Let $G' := \{x \mid 2n/5 \le \ones{x} \le 3n/5\}$ be a subset of~$G$.
Using Chernoff bounds and a union bound over $\mu=\poly(n)$ initial search points, the probability of initialising the first $\mu$ search points in $G'$ is $1 - \mu \cdot 2^{-\Omega(n)}$. Then, owing to elitism the algorithm will only accept
%search
points in $G' \cup F$. %in the future.
%

%According to Lemma~1 in \citep{Doerr2020} every unary unbiased variation operator
%$\mathrm{op}(x)$ on $\{0,1\}^n$ can be modelled as a two-step process:
%    first choose a Hamming radius $r$ from some distribution,
%    then return a search point on the Hamming sphere
%        $S_r(x):=\{y \in \{0,1\}^n\mid H(x,y)=r\}$
%    chosen uniformly at random.
For all search points $x \in G' \setminus F$ and all $y \in F$ we have $H(x, y) \ge \ones{y} - \ones{x} \ge n/5$. Moreover, since $\ones{x} \ge 2n/5$ and $\ones{y} = 4n/5$ there are at least $n/5$ bit positions $i$ in which $x_i = y_i = 1$.
Together,
$n/5 \le H(x, y) \le 4n/5$. \newedit{Let $\mathrm{op}$ be a unary unbiased variation
operator according to Algorithm~\ref{alg:unary-unbiased-op},
then} even when the radius $r$ is chosen as $r := H(x, y)$, the probability that
$\mathrm{op}(x)$
%\newedit{a unary variation operator applied on $x$}
creates $y$ is
$1/\binom{n}{H(x, y)} \le 1/\binom{n}{n/5}
    \leq \frac{(n/5)^{n/5}}{n^{n/5}}
    = 5^{-n/5}$.
%\begin{align*}
%\frac{1}{\binom{n}{H(x, y)}} \le \frac{1}{\binom{n}{n/5}}
%    \leq \frac{1}{\frac{n^{n/5}}{(n/5)^{n/5}}}
%    = 5^{-n/5}.
%\end{align*}
Taking a union bound over all search points $y \in F$, the probability of creating any Pareto-optimal search point is at most $(n/5 + 1) \cdot 5^{-n/5} := p$.
%\frac{|F|}{|S_r(x)|}
%    = \frac{n/5+1}{{n \choose n/5}}
%    \leq \frac{n/5+1}{\frac{n^{n/5}}{(n/5)^{n/5}}}
%    = \frac{n/5+1}{5^{n/5}} =: p.
%\end{align*}
The expected number of evaluations until a Pareto-optimal search point is found is thus at least $1/p = 2^{\Omega(n)}$. The expected number of generations follows since every generation makes $\lambda$ evaluations.
%and the probability of one success among $\lambda$ trial by Bernoulli's
%inequality is
%$
%1 - (1-p)^\lambda
%    \leq \lambda p
%$. Thus the expected waiting time in this phase is at least
%$
%\frac{1}{\lambda p} = \frac{5^{n/5}}{\lambda(n/5+1)} = \frac{5^{\Omega(n)}}{\lambda}
%$
%generations. The total number of generations is
%$(1-o(1))(1+5^{\Omega(n)}/\lambda) = 5^{\Omega(n)}/\lambda$.
%In terms of number of evaluations, this is
%$(1 - o(1))(\mu + 5^{\Omega(n)})=5^{\Omega(n)}$.
\end{proof}

%On the other hand, binary operators like recombination are not absolutely
%required to optimise \RRRMO, because unary operators like a hypermutation
%can simulate the effect of one-point crossover.
%Remark that the unbiasedness of the operator is required for the above
%lower bound, because some unary variation operators like hypermutation can
%simulate the effect of the $1$-point crossover, hence algorithms that use
%them can optimise \RRRMO in polynomial time.

\subsection{Can Hypermutation Help?}

\newedit{As mentioned earlier, hypermutation can
simulate the effect of $1$-point crossover, hence it can speed up the optimisation
process in single-objective optimisation~\citep{Zarges2011}. Both operators are able to flip a contiguous interval of bits, if one-point crossover is applied to a suitable selection of parents. Recall that hypermutation is a unary, but not unbiased operator, and thus evolutionary algorithms using hypermutation are not covered by the general lower bound from Theorem~\ref{thm:elitist-blackbox-unary-unbiased}.
Thus it is an
interesting question whether hypermutation can optimise \RRRMO efficiently or not.
%\cuong{I think the reviewer only asked whether hypermutation can help or not, and did not conjecture anything. Anyway, perhaps it is important to highlight that this is the distinctive difference of the multi-objective setting from the single-objective one, like one can make the two tasks of finding and covering the front conflicting for an operator.}
We show that, when replacing standard bit mutations with hypermutations in GSEMO, the answer is negative.
%Our \RRRMO function gives an example where the answer is negative for GSEMO.
The intuition for this failure is that the parameter $r$ of the operator has
to be set substantially differently between the two tasks: finding the first
Pareto-optimal solution, and covering the whole Pareto front afterwards.}

\newedit{\begin{theorem}\label{thm:gsemo-pc-zero-hypermutation}
GSEMO (Algorithm~\ref{alg:gsemo}) with $p_c=0$, and \newedit{using hypermutation}
with any parameter $r>0$ \newedit{as the only mutation operator} requires at
least $2^{\Omega(n)}$ evaluations in expectation to
%find a Pareto-optimal set of
optimise \RRRMO.
\end{theorem}}

\newedit{In order to prove Theorem~\ref{thm:gsemo-pc-zero-hypermutation}, we first show the following lemma on the parameter $\ell$ of
hypermutation when converting between Pareto-optimal solutions of \RRRMO.}

\newedit{\begin{lemma}\label{lem:gsemo-pc-zero-hypermutation-prep}
For any two distinct Pareto-optimal solutions $x$ and $y$ of $\RRRMO$,
in order to create $y$ from $x$ by hypermutation (Algorithm~\ref{alg:hypermutation})
it is necessary that the parameter $\ell$ is at least $n/5+\Hamming(x,y)/2$.
%Let $x,y$ be two distinct Pareto-optimal search points of $\RRRMO$ with $x=0^{k_1}1^{4n/5}0^{n/5-k_1}$ and $y=0^{k_2}1^{4n/5}0^{n/5-k_2}$.
%Consider the choice of the contiguous region of bits to be flipped during a hypermutation.
%A necessary requirement for creating $y$ via a hypermutation of $x$ is that the contiguous region contains at least $n/5 + \vert{k_2 - k_1}\vert =n/5 + H(x,y)/2$ bit positions.
%That it is possible to create $y$ from $x$ with hypermutation, there must be least $n/5 + \vert{k_2 - k_1}\vert =n/5 + H(x,y)/2$ positions between $c$ and $c+\ell-1$ (i.e. the contiguous region where bits are flipped independently with probability $r$ covers at least $n/5 + H(x,y)/2$ many bits, see Algorithm~\ref{alg:hypermutation}).
\end{lemma}}
\begin{proof}
\newedit{Because $x$ and $y$ are Pareto optimal and distinct, we can write them as
$x=0^{a}1^{4n/5}0^{n/5-a}$ and $y=0^{b}1^{4n/5}0^{n/5-b}$ for some integers
$a,b\in[0,n/5]$, $a\neq b$.
%Furthermore, remark that the roles of $x$ and $y$
%are indifferent in the statement of the result, \ie it also holds true for creating
%$x$ from $y$ since $\Hamming(y,x)=\Hamming(x,y)$, therefore we can assume \wlo
%that $a<b\leq n/5$.
\brandnewedit{We first consider the case $a<b$, which is} illustrated in Figure~\ref{fig:create-another-pareto-2}.}

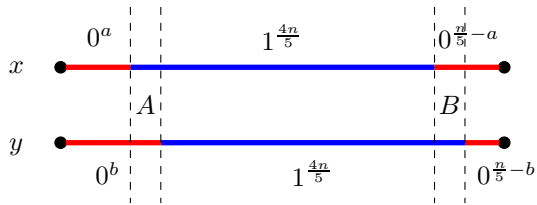
\begin{figure}[h] % the search space of RRMO
\begin{center}
	\begin{tikzpicture}%[scale=0.9,transform shape]
		\draw[*-*] (0,2) -- (6,2);
		\node at (-0.5,2) {$x$};
		\draw[line width=2,color=blue] (1,2) -- (5,2);
		\node at (3,2.4)   {$1^{\frac{4n}{5}}$};
		\draw[line width=2,color=red] (0.15,2) -- (1,2);
		\node at (0.6,2.4) {$0^{a}$};
		\draw[line width=2,color=red] (5,2) -- (5.85,2);
		\node at (5.45,2.4) {$0^{\frac{n}{5}-a}$};
		
		\draw[*-*] (0,1) -- (6,1);
		\node at (-0.5,0.95) {$y$};
		\draw[line width=2,color=blue] (1.4,1) -- (5.4,1);
		\node at (3.4,0.6)   {$1^{\frac{4n}{5}}$};
		\draw[line width=2,color=red] (0.15,1) -- (1.4,1);
		\node at (0.7,0.6) {$0^{b}$};
		\draw[line width=2,color=red] (5.4,1) -- (5.85,1);
		\node at (5.95,0.6) {$0^{\frac{n}{5}-b}$};
		
		\draw[dashed] (1,2.8) -- (1,0.2);
		\draw[dashed] (1.4,2.8) -- (1.4,0.2);
		\node at (1.2,1.5) {$A$};
		
		\draw[dashed] (5,2.8) -- (5,0.2);
		\draw[dashed] (5.4,2.8) -- (5.4,0.2);
		\node at (5.2,1.5) {$B$};
	\end{tikzpicture}
\end{center}
\caption{Creating another Pareto-optimal solution $y$ from $x$. Blocks of $1$s are
depicted in blue while those of $0$s are red.}\label{fig:create-another-pareto-2}
\end{figure}

\newedit{
Let $A$ and $B$ be the sets of positions where $x$ and $y$ differ on the left and
on the right, respectively, as shown in the figure. It is necessary that these $A$,
$B$ positions are covered entirely between $c$ and $c+\ell-1\bmod n$ in order
to convert $x$ to $y$ by the operator. Note that $|A|=|B|=b-a=\Hamming(x,y)/2\leq n/5$.
Consider the following choices for the parameter $c$ of the algorithm.
%
%If $c$ is before the end of $A$ or after the end of $B$\dirk{after the start of $B$?}\cuong{$4n/5+b$ is the end of $B$},
%\ie $c\leq b$ or $c>4n/5+b$ (if $b<n/5$),
%then in order to reach the end of $B$ the $\ell$ positions have to cover the whole
%block $1^{4n/5}$ of $y$ thus $\ell>4n/5>n/5+n/5\geq n/5+(b-a)$.
%
If $c$ is between the end of $A$ and the beginning of $B$, \ie $b<c\leq 4n/5+a$,
then a wrap-around is required to fully \brandnewedit{cover $A$ and particularly
the end of $A$ has to be reached, therefore}
$\ell\geq b+(n-c)\geq b+(n-(4n/5+a))=n/5+(b-a)$.
%
%If $c$ is within $B$, \ie $4n/5+a<c\leq 4n/5+b$, then clearly the $\ell$ positions
%must cover everything, \ie  $\ell=n>n/5+n/5\geq n/5+(b-a)$.
%In every choice, it is always necessary that $\ell\geq n/5+ (b-a) = n/5+\Hamming(x,y)/2$.
For the other choices of $c$, \ie $c\leq b$ ($c$ is before the end of $A$) or
$c>4n/5+a$ ($c$ is after the beginning of $B$), note that at least the whole block
$1^{4n/5}$ of $y$ has to be covered by the $\ell$ positions, and so
$\ell \geq 4n/5>n/5+n/5\geq n/5+(b-a)$.
}
%\dirk{Aren't two cases sufficient (if $a < b$)? If $c \le b$ we need to cover a whole lotta ones. If $c > b$ we need a wrap-around to cover $A$.}
%\cuong{Kind of, but for $c>b$ and let say $c>4n/5+b>b$ then the wrap-around until $A$ is not sufficient to say anything.}

\newedit{The result for $a>b$ follows by noting that creating $x$ from $y$ or
creating $y$ from $x$ require exactly the same set of bits (defined by
$\Hamming(x,y)$) to be flipped, so the necessary condition on $\ell$ \brandnewedit{is the same}.
%\andre{I wrote ''is''. Please check. Changes are fine to me!} the same. %% already done in an old commit..
%This also concludes that the roles of $x$ and $y$
%are indeed indifferent in the statement of the result.
}
\end{proof}

%It remains an open question whether such result also holds for NSGA-II. We note
%also that mixing the hypermutation with bitwise mutation, \eg see
%\citet{Corus2016}, can make GSEMO efficient again, however this is out of the
%scope of our study.} %\andre{where could we postphone this sentence?}\cuong{I have already fixed this, but you deleted the fix in your last commit...}

With Lemma~\ref{lem:gsemo-pc-zero-hypermutation-prep} in place, we now prove Theorem~\ref{thm:gsemo-pc-zero-hypermutation}.
\begin{proof}[Proof of Theorem~\ref{thm:gsemo-pc-zero-hypermutation}]
\newedit{By classical Chernoff bounds, the probability of initialising the algorithm with
a search point belonging to $\posFit\setminus\frontP$, is $1-2^{-\Omega(n)}$.
Assuming this occurs, then we consider two following cases.}

\newedit{\textbf{Case of $r\leq 1/2$}: In this case, we only need to focus on the expected
time to create the first individual in $F$ from those in $G\setminus F$.
Fix a search point $y \in F$ and assume hypermutation is applied to a search point $x \in G \setminus F$. As $H(x, y) \ge n/5$, there is a set of $n/5$ bits that must all be flipped in order to create~$y$. Even if all these bits are part of the interval chosen by the hypermutation, the probability of flipping all these bits is at most $r^{n/5}$.
%
%We will also
%discount the probabilities of selecting the right parent, and also pick the right values for
%$c$ and $\ell$ in the hypermutation operator, \ie assuming they are at most $1$. Nevertheless,
%at least $n/5$ bits must be flipped to create a specific solution in $F$,
%
Thus, by a union
bound on the $n/5+1$ solutions of $F$, the probability of creating any search point in $F$ is at most
$(n/5+1)r^{n/5} \leq (n/5+1)2^{-n/5} = 2^{-\Omega(n)}$. Thus in expectation, $2^{\Omega(n)}$
generations are required to create the first individual in $F$.}

\newedit{\textbf{Case of $r> 1/2$}: Under this setting, it suffices to focus on the expected
time to find a Pareto-optimal set starting from the first Pareto-optimal solution. We
consider the solution $x=0^{n/10}1^{4n/5}0^{n/10} \in F$
and
%, for which due to its symmetry it is
%difficult to create other Pareto-optimal solutions from it, or to create this
%solution from the other Pareto-optimal solutions.
%Particularly, the points $0^{n/5}1^{4n/5}$ and $1^{4n/5}0^{n/5}$ with the largest Hamming
%distance (of $n/5$) to $x$ is the easiest to created but the probability to achieve this
%is still exponentially small.
%%Note that $H(x,y) \leq n/5$ for any other Pareto-optimal solution $y$.
%We
distinguish two subcases:
}

%\begin{figure}[h] % the search space of RRMO
%	\begin{center}
%		\begin{tikzpicture}%[scale=0.9,transform shape]
%			\draw[*-*] (0,2) -- (6,2);
%			\node at (-0.5,2) {$x$};
%			\draw[line width=2,color=blue] (1,2) -- (5,2);
%			\node at (3,2.4)   {$1^{\frac{4n}{5}}$};
%			\draw[line width=2,color=red] (0.15,2) -- (1,2);
%			\node at (0.6,2.4) {$0^{\frac{n}{10}}$};
%			\draw[line width=2,color=red] (5,2) -- (5.85,2);
%			\node at (5.5,2.4) {$0^{\frac{n}{10}}$};
%			
%			\draw[*-*] (0,1) -- (6,1);
%			\node at (-0.5,0.95) {$y$};
%			\draw[line width=2,color=blue] (1.4,1) -- (5.4,1);
%			\node at (3.4,0.6)   {$1^{\frac{4n}{5}}$};
%			\draw[line width=2,color=red] (0.15,1) -- (1.4,1);
%			\node at (0.7,0.6) {$0^{k}$};
%			\draw[line width=2,color=red] (5.4,1) -- (5.85,1);
%			\node at (5.95,0.6) {$0^{\frac{n}{5}-k}$};
%			
%			\draw[dashed] (1,2.8) -- (1,0.2);
%			\draw[dashed] (1.4,2.8) -- (1.4,0.2);
%			\node at (1.2,1.5) {$A$};
%			
%			\draw[dashed] (5,2.8) -- (5,0.2);
%			\draw[dashed] (5.4,2.8) -- (5.4,0.2);
%			\node at (5.2,1.5) {$B$};
%		\end{tikzpicture}
%	\end{center}
%	\caption{Creating another Pareto-optimal solution $y$ from $x$. Blocks of $1$s are depicted
%    in blue while those of $0$s are red.}\label{fig:create-another-pareto}
%\end{figure}

\newedit{If $x$ is the first Pareto-optimal solution found by the algorithm,
the population only contains $x$ and we consider the expected time to create a second Pareto-optimal solution $y\neq x$.
It follows from Lemma~\ref{lem:gsemo-pc-zero-hypermutation-prep} that
at least $n/5+\Hamming(x,y)/2$ bits are considered in $x$ for modification to
create $y$, however creating $y$ only requires modifying $\Hamming(x,y)$ bits.
Note that $\Hamming(x, y) \le n/5$ as in $y$ the block of ones is shifted by at most $n/10$ compared to~$x$, and every shift by one position increases the Hamming distance by~2.
Therefore there are at least $(n/5+\Hamming(x,y)/2)-\Hamming(x,y)
=n/5-\Hamming(x,y)/2\geq n/5-(n/5)/2=n/10$ bits that are considered for modification
but should not be changed to create $y$. By a union bound over the $n/5$ solutions
in $F \setminus\{x\}$ that can be created, the probability of creating a second
Pareto-optimal solution is at most $(n/5)(1-r)^{n/10}<(n/5)2^{-n/10}=2^{-\Omega(n)}$,
and the expected waiting time for this is at least $2^{\Omega(n)}$.}

\newedit{Otherwise, if $x$ is not the first Pareto-optimal solution found by the algorithm:
In this subcase, $x$ still needs to be created from the others in order to find a Pareto-optimal set.
%In fact, the best way
%to create $x$ is to have the population containing only $y=0^{n/5}1^{4n/5}$ or $1^{4n/5}0^{n/5}$,
%but still starting from such
%solutions at least $n/10$ bits need not be flipped in order to create $x$ by the above argument.
By exactly the same argument as in the previous subcase, we note that in order to create $x$ from
any Pareto-optimal solution $y\neq x$, there are at least $n/10$ bits that are considered for
modification but should not be flipped. The probability for hypermutation to achieve this is at
most $2^{-\Omega(n)}$ and the expected waiting time
%to create $x$ (even when all other Pareto-optimal points are already created)
is at least $2^{\Omega(n)}$.}
%
%create a second Pareto-optimal solution, denoted $y$, starting from the first one, denoted $x$,
%in $F$. Again, depending on how $x$ looks like, only certain choices for $c$ and $\ell$ are
%valid but we assume the probability of making those choices is at most $1$. However, we note
%that to create the $y$ search point from $x$, at least a certain number of ones should not be
%flipped, and even in the best case of $x=1^{4n/5}0^{n/5}$ and $y=0^{n/5}1^{4n/5}$, this number
%is $3n/5$. Thus the probability of such event is at most $(1-r)^{3n/5}<2^{-3n/5}$, thus
%expected number of generations is at least $2^{\Omega(n)}$.}

\newedit{Combining the two cases gives a lower bound on the expected
number of generations or of fitness evaluations of $(1-2^{-\Omega(n)})2^{\Omega(n)} = 2^{\Omega(n)}$
to find a Pareto-optimal set of \RRRMO.}
\end{proof}

\newedit{The second case in the proof uses one specific Pareto-optimal solution
in which the block of $4n/5$ ones is exactly centered on the bit string and argues
that generating that solution is difficult. Nevertheless, the argument also holds
for any Pareto-optimal solution with the block of ones sufficiently centered,
\ie at any sublinear distance, and this means that any sublinear number of
points on the Pareto set can be missed and not only one.
The result makes explicit use of the fact that GSEMO can only produce one offspring
in every generation and all dominated solutions are removed. These properties do not hold
for NSGA-II hence it remains an open question whether a similar result can be shown
for NSGA-II. Finally, we also note that mixing the hypermutation with standard bit mutation,
\eg see \citep{Corus2016}, can make GSEMO efficient again, however this is out of the
scope of our study.}

\subsection{Use of \newedit{One-Point Crossover} Implies Expected Polynomial \newedit{Optimization} Time
        %\newedit{for EMOs to Optimize \RRRMO}
        \newedit{on \RRRMO}}\label{sec:upper-bound}

While $\RRRMO$ is hard for many EMO algorithms without crossover, now we show for GSEMO and for NSGA-II that they both succeed in finding the whole Pareto set of $\RRRMO$ in expected polynomial time.

\subsubsection{Analysis of GSEMO}\label{sec:gsemo}

We start with GSEMO as the algorithm is conceptually simpler.
\begin{theorem}\label{thm:gsemo}
GSEMO (Algorithm~\ref{alg:gsemo})
with one-point crossover, standard bit mutation and $p_c \in (0,1)$ requires at most %in expectation
    $\bigO\left(\frac{n^4}{1-p_c} + \frac{n}{p_c}\right)$
    %generations and
    %$\bigO\left(\frac{n^5}{1-p_c} + \frac{n^2}{p_c}\right)$
    fitness evaluations in expectation % this is CLAIM 4
    to find the whole Pareto-optimal set %the Pareto front
    of \RRRMO.
\end{theorem}
\begin{proof}
%At first
%since two individuals with a distinct number of ones are comparable.
%To prove the running time, we consider the typical run approach with
We use the well-known method of typical runs~\cite[Section~11]{Wegener2002} and divide a run into several phases that reflect ``typical'' search dynamics. Each phase has a defined goal and we provide upper bounds on the expected time to achieve these goals. When a phase ends, the next phase starts; however, phases may be skipped if the goal of a later phase is achieved before the phase starts.
%
%divide the run into
%the following phases.
%%
%Note that fortunate events can occur during a phase and allows the skipping
%of the phase (and possibly some that follow it), thus for an upper bound
%argument, we will ignore those events.

\textbf{Phase 1: Create a search point in $G$.}
        %Create a search point which is $0^n$ or dominates an individual with fitness $(0,0)$.}

By a Chernoff bound
%\textbf{Reaching $3n/5$ number of $1$-bits}: % This is CLAIM 2.1
%we see that the probability of
%obtaining more than $3n/5$ ones after initialisation of the individual $s$ is
%bounded by $2^{-\Omega(n)}$ from above (since the initialisation is uniformly at
%random).
the probability that the initial search point is not in $G$ is
at most $2^{-\Omega(n)}$ as it is necessary to create a search point with more than $3n/5$ ones.
%(since an individual with fitness $(0,0)$ is either $0^n$ or has more than $3n/5$ ones). %%[cuong] no, 0^n has fitness (n,n) and not (0,0)
Since all search points not in $G$ have the same fitness vector $(0, 0)$, while no search point in $G$ is found, the population always consists of the latest search point and crossover, if executed, has no effect.
\newedit{The process is then a repeated application of standard bit mutation, and by Lemma~\ref{lem:repeat-bitwise-mutation}
with $c=1$, $\varepsilon=(3/5-1/2)(2/c)=1/5$,
the expected time to reach a search point with at most $3/5n$ ones is $\bigO(n)$.}
Consequently,
the expected number of generations for finding a search point in $G$ is
$1+2^{-\Omega(n)} \cdot \newedit{\bigO(n)} = 1 + o(1)$.

\textbf{Phase 2: Create a search point with $3n/5$ ones.}

Once an individual in $G$ is found, every individual in %the current generation
$P_t$ always has the same
number $i$ of ones
because otherwise those with the highest number of ones will dominate
and remove the others.
We now compute the expected time for $P_t$ to contain individuals with
exactly $3n/5$ ones using a fitness-level argument. %. This can be done with a fitness-level argument and
%by relying on mutation-only steps.
Note that creating a search point with
a higher number of ones always removes the previous population and advances
the process.
Suppose that
%there is an individual $x$
$P_t$ contains individuals
with $i\in \{0,\dots,3n/5-1\}$ ones, then
%$i \in \{0,...,3n/5-1\}$ ones,  %If an offspring $s'$ of $s$ has a lower number of ones than $s$ it is rejected.
%So $s'$ has at least the same number of ones as $s$ if it is accepted.
%Note that every other individual in the population has also $i$ ones.
%So we can generate an individual with more than $i$ ones, but at most $3n/5$
%ones in the following way.
the number of ones can be increased by
selecting an arbitrary individual as a parent, choosing not to apply crossover, and during the mutation flipping exactly one zero bit while
%no other bit is flipped.
%during the standard bit mutation process. So the probability $s_i$ for generating
%an individual with more than $i$ ones is bounded by
keeping the other bits unchanged. The probability of this event is at least
%\[
$
(1-p_c) \cdot \frac{n-i}{n} \cdot \left(1-\frac{1}{n} \right)^{n-1}
    \geq \frac{(1-p_c)(n-i)}{en} %=: s_i.
$.
%\]
%from below.
%By the fitness level method we see that the expected time to create
%an offspring with $3n/5$ ones is bounded by
Thus, summing up expected waiting times of all levels $i$ %$i \in \{0,...,3n/5-1\}$
gives a bound of
\begin{align*}
%\sum_{i=0}^{3n/5-1} \frac{1}{s_i}
%\sum_{i=0}^{3n/5-1} \frac{en}{(1-p_c)(n-i)}
%     &=
       \frac{en}{1-p_c} \sum_{i=0}^{3n/5-1} \frac{1}{n-i} %\\
     = \bigO\left(\frac{n}{1-p_c}\right).
    %&= \frac{en  \cdot O(1)}{1-p_c}, %\\
    %&= O((1-p_c)n)
\end{align*}
%from above.

\textbf{Phase 3: Create the first search point in $\lfrontP$.}%,
                 %\ie with $3n/5$ ones in a single block.}

%Note that we have only individuals with $3n/5$ ones in this phase.
%For this, it suffices to first select an individual $x$ with a minimal
%number $i$ of zeros between the first and the last one from the left with
%$i \in \{1,\dots,2n/5\}$ as parent by the uniform selection, \ie probability
%$1/|P_t|\geq 1/n$ by Lemma~\ref{lem:size-non-dom-set}, then
%%If $i=0$ we are done. So suppose $i>0$. In order to get an improvement we have to choose
%%$x$ as a parent,
%omit crossover and flip the first or the last one from the
%left to zero and another zero between the first and the last one from the left
%to one.
To make progress towards $\lfrontP$, it suffices to first select an individual $x$ with a maximum value of $\LZ(x) + \TZ(x)$, denoted by~$2n/5 - i$, and to increase this sum while maintaining $3n/5$ ones.
By Lemma~\ref{lem:size-non-dom-set}, the probability of selecting $x$ as parent is at least $1/|P_t|\geq 1/n$. If the algorithm then
%If $i=0$ we are done. So suppose $i>0$. In order to get an improvement we have to choose
%$x$ as a parent,
omits crossover, either flips the first 1-bit or the last 1-bit and flips one of the $i$ 0-bits that do not contribute to $\LZ(x) + \TZ(x)$, the fitness is increased.
The probability for this event is at least
%$s_i$ for improvement is then bounded by
%of this event is at least
%\[
$
\frac{1}{n}
\cdot (1-p_c)
\cdot \frac{1}{n}
\cdot \frac{i}{n}
\cdot \left(1-\frac{1}{n} \right)^{n-2}
    \geq \frac{i(1-p_c)}{en^3},
$
%\]
%from below.
and %a bound for
the expected number of
%fitness evaluations to achieve this goal,
generations to complete this phase,
by summing up the expected waiting times
%overall all
over all
$i$, is at most %no more than
\begin{align*}
%\sum_{i=1}^{2n/5} \frac{1}{s_i}
    %&\leq
    \sum_{i=1}^{2n/5} \frac{en^3}{i(1-p_c)}
     = \frac{en^3}{1-p_c} \sum_{i=1}^{2n/5} \frac{1}{i} %\\
    %&
     = \bigO \left(\frac{n^3\log n}{1-p_c} \right).
\end{align*}
%for this first goal.

\textbf{Phase 4: Cover $\lfrontP$ entirely.}

%Suppose this has not been done yet, then
Suppose $\lfrontP$ is not completely covered, then
%Pareto front. Suppose that we couldn't find the whole Pareto front yet. Then
%there is an individual $y$ on $\lfrontP$ \cuong{this is in $P_t$, not $\lfrontP$}
%with $f(y)=(3n^2/5+2n/5-i,3n^2/5+i)$
%for $i \in \{0,...,2n/5\}$, but no individual $z$ with $f(z)=(f_1(y)-1,f_2(y)+1)$
%or $f(z)= (f_1(y)+1,f_2(y)-1)$ if $0<i<2n/5$ respectively
%$f(z)=(f_1(y)-1,f_2(y)+1)$ if $i=0$ respectively $f(z)=(f_1(y)+1,f_2(y)-1)$ if
%$i=2n/5$.
there must exist a missing individual $z$ on $\lfrontP \setminus P_t$
next to a $y \in P_t \cap \lfrontP$,
    \ie $|\TZ(z)-\TZ(y)|=1$
        and $\TZ(z)-\TZ(y)=\LZ(y)-\LZ(z)$.
%Note that we can generate such an individual
Individual $z$ can be generated from $y$ by omitting
crossover and flipping
%the first (resp. last) one from the left to zero and the
%last (resp. first) zero from the left to one (if $f_1(z)=f_1(y)-1$ resp.
%$f_1(z)=f_1(y)+1$) while the other bits are not flipped.
a one at one extreme of the consecutive block of ones to a zero
and a zero at the other extreme to a one while keeping the other bits
unchanged.
%Let $t_i$ be the probability that we generate such an individual $z$ under the condition that there are already $i$ individuals on the front with pairwise different fitness values (where $i \in \{1,...,2n/5\}$).
Since the parent is chosen uniformly at random,
%for
the probability
%$q$
%to create such an individual $z$ is at least
of that event is
%\[
$
%q \geq
\frac{1-p_c}{n^3}
\cdot \left(1-\frac{1}{n}\right)^{n-2}
    \geq \frac{1-p_c}{en^3}
$.
%\]
As $2n/5$ such steps suffice to cover $\lfrontP$, the expected number of generations is
at most
%So the expected time to achieve the second goal is bounded by
\[
\frac{en^3}{1-p_c}
\cdot \frac{2n}{5}
    %\in
    = \bigO \left(\frac{n^4}{1-p_c} \right).%,
\]
%from above which gives
%which is also the asymptotic expected time of the phase.

\textbf{Phase 5: Create the first search point in $\frontP$.} %Create a first point with $4n/5$ ones in a single block.}

Starting from a population $P_t = F'$, thus $|P_t|= 2n/5+1$,
the first search point on $\frontP$ can
be created by crossover as follows.
If the algorithm picks parents $p_1=0^i 1^{3n/5} 0^{2n/5-i}$ with $1 \le i \le n/5$ and $p_2 = 0^{i+n/5} 1^{3n/5} 0^{n/5-i}$ and any cutting point $\ell \in [i + n/5, i + 3n/5]$, the result of the one-point crossover contains a single block of $4n/5$ ones and thus belongs to~$F$. The same applies to the final offspring if the mutation following crossover does not flip any bit.
The probability for selecting $p_1$ and $p_2$ is
$
    \frac{n/5}{(|P_t|)^2} = \frac{n/5}{(2n/5 + 1)^2} = \Omega(1/n)$.
%Assuming \wlo the selected parents are
%    $p_1=0^i 1^{3n/5} 0^{2n/5-i}$
%and $p_2=0^j 1^{3n/5} 0^{2n/5-j}$ with $0 \leq i \leq j \leq 2n/5$,
%then to create an offspring on $\frontP$ further restrictions
%are required on $i$ and $j$:
%(i) we need $i+4n/5 \leq n$ or equivalently $i \leq n/5$; and
%(ii) $j+3n/5 = i+4n/5$ or equivalently $j=i+n/5$;
%and the probability of these events with respect to the population
%size $1+2n/5$ is
%    $\frac{n/5}{1+2n/5}\cdot\frac{1}{1+2n/5} = \Omega(1/n)$.
%%
%Once a pair of right parents are selected, the crossover point only
%needs to be in the interval $[j=i+n/5,i+3n/5]$, which has length
%$2n/5$, for the offspring to have a probability $(1/2)(1-1/n)^n
%\geq 1/8$, if $n \geq 2$, of being on $\frontP$.
%
So the probability of creating an offspring in $F$ is at least $p_c \cdot \frac{2n/5}{n+1}
\cdot \Omega(1/n) \cdot (1-1/n)^n =\Omega(p_c/n)$ and the expected number of generations for this to happen is $\bigO(\frac{n}{p_c})$.
%fitness evaluations
%generations
%are required in expectation to create the first search point
%on $\frontP$.

\textbf{Phase 6: Cover $\frontP$ entirely.}

The creation of the first search point on $\frontP$ removes
all the individuals on $\lfrontP$. We rely on 2-bit-flip
mutation steps to cover $\frontP$ similarly to the arguments in Phase 4.
%in the front, \ie $P \subsetneq F$,
%then there must exist some search point $x \in P$ such that shifting
%its $1$-block will fill up $F$ with one more search point, and this
%only requires a $2$-bit flip by mutation. Thus the overall probability
%increasing the number of search points on $\frontP$ by one is at least
%$(1-p_c) \cdot \frac{1}{|P|} \cdot \left(\frac{1}{n}\right)^2 \left(1 - \frac{1}{n}\right)^{n-2}
%    \geq \frac{1-p_c}{n/5} \cdot \frac{1}{en^2} = \Omega\left(\frac{1-p_c}{n^3}\right)$.
%Thus the expected number of generations to cover the front starting
%from one on it is no more than $\bigO(n^4/(1-p_c))$.
A minor difference is that the number of missing points like $z$
is now $n/5$ since $|\frontP|=n/5+1$. Nevertheless, the asymptotic number
of generations to fully cover $\frontP$ is still $\bigO(\frac{n^4}{1-p_c})$.
%Assuming there are still missing search points
%
%The expected number of generations of the phase is then
%$\bigO\left(\frac{n^4}{1-p_c}+\frac{n}{p_c}\right)$.\\

Summing up the time bounds
% on the expected numbers of generations
of all phases gives a
%the overall
bound of $\bigO\left(\frac{n^4}{1-p_c}+\frac{n}{p_c}\right)$
on the expected number of generations (or evaluations)
for GSEMO to
%optimise
find the
Pareto set.%. of
%\RRRMO.%, %, and
%which is also the expected number of evaluations.
%applying Lemma~\ref{lem:size-non-dom-set} gives the claimed bound on
%the expected number of evaluations.
\end{proof}

\subsubsection{Analysis of NSGA-II}\label{sec:nsga-ii}

We now turn to the analysis of NSGA-II. The search dynamics of NSGA-II are more complex than those of GSEMO due to the non-dominated sorting and the use of crowding distance. Furthermore, the uniform parent selection of GSEMO is replaced with a binary tournament selection, complicating the analysis. Unlike for GSEMO, it is not always guaranteed that all non-dominated solutions survive to the next generation, especially if the population size is chosen too small.

\begin{theorem}\label{thm:nsga-ii}
NSGA-II~(Algorithm~\ref{alg:nsga-ii}) with $1$-point crossover, standard bit mutation, $p_c \in (0,1)$ and $\mu \geq 2n+5$
finds the whole Pareto set of \RRRMO in expected
    $\bigO\left(\frac{\mu}{n p_c} + \frac{n^3}{1-p_c} \right)$ generations and
    $\bigO\left(\frac{\mu^2}{n p_c} + \frac{\mu n^3}{1-p_c} \right)$ fitness evaluations.
\end{theorem}
\begin{proof}
%To prove the running time, we consider the typical run approach with the following phases.
Consider the following phases of a run. %as in the proof of Theorem~\ref{thm:gsemo}.

%\textbf{Reaching $3n/5$ number of $1$-bits}:
\textbf{Phase 1: Create a search point in $G$.}
        %dominates fitness $(0,0)$.}

%We first estimate the expected time to create one individuals with at most
%$3n/5$ ones.
By a Chernoff bound the probability that every initial individual
has more than $3n/5$ ones is at most $2^{-\mu \Omega(n)}$. If this happens, %regardless,
the probability of creating a specific individual in $G$
by mutation is at least $n^{-n}$, regardless of the input
%search point
solution and the
%operations preceding the mutation.
preceding operations.
By the law of total probability, the expected number of evaluations to
obtain a search point in $G$ is at most
$%\[
\mu + 2^{-\mu\Omega(n)}n^n
    = \mu + 2^{- \Omega(n^2)}\cdot n^n
    = \mu + o(1)
%\mu + 2^{-\mu \Omega(n)}n^n \leq \mu + 2^{- \Omega(n^2)}\cdot n^n = \mu + o(1).
$, %\]
and these are $1+o(1)$ generations.

\textbf{Phase 2: Create a search point with $3n/5$ ones.}

%We now compute the expected time for $P_t$ to contain an individual with exactly
%$3n/5$ ones. %using a fitness-level argument as in the proof of GSEMO.           % [cuong] this is not exactly like GSEMO
%
Suppose that the maximal number of ones in $P_t$ is $i \in \{0,\dots,3n/5-1\}$.
Let  $x'$ be an individual with that number of ones, then $x'$ is picked as
a competitor in the two binary tournaments to choose $\{p_1, p_2\}$ with
probability
$
1 - (1 - 1/\mu)^4
  \geq %(4/\mu)/(4/\mu+1)
  %=
  4/(\mu+4)$ \newedit{by Lemma~\ref{lem:lambda-trials}} %Lemma~10 in~\citep{Badkobeh2015},
and this guarantees that at least one of the parents has $i$ ones. Therefore,
with probability at least $\frac{4}{\mu+4}\cdot(1-p_c)\cdot\frac{n-i}{en}=:s_i$,
one of the offspring $\{s'_1,s'_2\}$ has more than $i$ ones, as it suffices to
skip the crossover step, then flip a zero to a one while keeping the remaining bits unchanged
in the mutation step. This reproduction process is repeated $\mu/2$ times, so
the chance of at least one success
%in creating a search point with more ones
is at least
$1-(1-s_i)^{\mu/2}\geq
    \frac{s_i\mu/2}{s_i\mu/2+1}$ by the same lemma. The expected waiting time
by summing up all possible values of $i$ is no more than
$
    \sum_{i=0}^{3n/5-1} \left(1 + \frac{2}{\mu s_i}\right)
$ which is
\begin{align*}
%\sum_{i=0}^{3n/5-1} \left(1 + \frac{2}{\mu s_i}\right)
  %&=
  %3n/5
  \bigO(n)
  + \frac{2}{\mu}\sum_{i=0}^{3n/5-1} \frac{(\mu+4)en}{4 (1-p_c) (n-i)}%\\
  &= \bigO\left(\frac{n}{1-p_c} \right).%,
\end{align*}
\textbf{Phase 3: Create the first search point in $\frontP'$.}%, \ie with $3n/5$
%ones in a single block of consecutive bits.}

%The first goal of this phase is to create an individual with $3n/5$ ones
%cumulated in a single block, i.e. a search point in $F'$, assuming starting with
%none.
%
%Let $x$ be an individual with a positive crowding distance and a minimal
%number $i$ of zeros between the first and the last one from the left with $i \in
%\{1,...,2n/5\}$. Then $x \in F_t^1$.
\newedit{After Phase 2 has completed, as long as no search point in $F$ has been found, all non-dominated search points in $P_t$, that is, all search points in $F_t^1$, will contain $3n/5$ ones.
Note that, unlike for GSEMO, the multiset $F_t^1$ may contain duplicates of the same search point. Let $\widehat{F_t^1}$ denote $F_t^1$ after removing duplicates, then \brandnewedit{$\widehat{F_t^1}$ is a set of non-dominated solutions and} applying Lemma~\ref{lem:size-non-dom-set} to $\widehat{F_t^1}$ yields $|\widehat{F_t^1}| \le n - 3n/5 + 1 = 2n/5+1$. Thus, the multiset $F_t^1$ contains at most $2n/5+1$ different fitness vectors.
Applying Lemma~\ref{lem:nsga-ii-protect-layer} with $m \coloneqq 2n/5+1$, statement $(i)$ yields that at most $4 \cdot (2n/5+1) = 8n/5+4$ individuals in $F_t^1$ have a positive crowding distance. Let us denote the set of all individuals in $F_t^1$ with positive crowding distance as $S_t^*$.
}

\newedit{
Let $x'' \in S_t^*$ be a search point a maximum value of $\LZ(x) + \TZ(x) =: 2n/5 - i$ for $i\in\{1,\dots,2n/5\}$.
Note that $x''$ wins a binary tournament against all individuals in $P_t \setminus S_t^*$ by definition of the non-dominated sorting.
%\newedit{By Lemma~\ref{lem:size-non-dom-set}, $|F_t^1| \le n - 3n/5 + 1 = 2n/5 + 1$.}
% be the minimal number of zeroes between the first and
%the last one of individuals in $P_t$ with $3n/5$ ones. Note that individuals
%fitting these criteria belong to $F^1_t$ and let $x''$ be one of them with
%a positive crowding distance.
%
If $x''$
appears as the first competitor in a binary tournament (which happens
with probability $1/\mu$), and the second competitor is \newedit{in $P_t \setminus S_t^*$ (which happens with probability at least $(\mu - |S_t^*|)/\mu \ge (\mu - 8n/5 - 4)/\mu \ge (2n/5+1)/(2n+5) = 1/5$)}
%below $F_t^1$ or in $F_t^1$ and has zero crowding distance
%\andre{(which happens with probability at least
%$
%%\frac{2n+5 - (8n/5+4)}{2n+5}
%1 - \frac{8n/5+4}{\mu}
%\geq 1 - \frac{8n/5+4}{2n+5}
%%= \frac{2n/5+1}{2n+5}
%$ as there are at most $2n/5+1$ non-dominated solutions in $F_t^1$ by Lemma~\ref{lem:size-non-dom-set} which implies that $F_t^1$ covers at most $2n/5+1$ distinct fitness vectors and thus there are at most $8n/5+4$ individuals in $F_t^1$ with positive crowding
%distance by (i) in Lemma~\ref{lem:nsga-ii-protect-layer}, i.e. there are at least $\mu - (8n/5+4)$ individuals below $F_t^1$ or in $F_t^1$ with zero crowding distance),
then $x''$ wins the tournament.
The same holds for the disjoint event where the roles of the first and second competitor are swapped. Thus,
%As the multiset of individuals with positive crowding distance in $F_t^1$ is disjoint to the multiset of individuals which are either below $F_t^1$ or in $F_t^1$ with zero crowding distance,
the probability of $x''$ winning the tournament is at least $\frac{2}{5 \mu}$.
}
Furthermore there are two tournaments in generating a pair of offspring. Consequently, the probability of $x''$ being the outcome of at least one of them is at least
$
1-(1-\frac{2}{5\mu})^2
\geq \frac{4/(5\mu)}{4/(5\mu)+1}
= \frac{4}{4+5\mu}
$ \newedit{by Lemma~\ref{lem:lambda-trials}.}
So, similarly to the proof of Theorem~\ref{thm:gsemo}, the probability of
increasing the maximum value of $\LZ + \TZ$ in the population beyond $2n/5 - i$ is at least
$
\frac{4}{4+5\mu}\cdot\frac{i(1-p_c)}{en^2}
=:s'_i$ during
the creation of the pair. The success probability for $\mu/2$ pairs
is then at least
$1-(1-s'_i)^{\mu/2}\geq
\frac{s'_i\mu/2}{s'_i\mu/2+1}$, and the expected waiting time
%to achieve the first goal
to complete this phase
is no more than
$\sum_{i=1}^{2n/5} \left(1+\frac{2}{\mu s'_i}\right)$
which is
\begin{align*}
	%2n/5-1 + \sum_{i=1}^{2n/5-1} \frac{2}{\mu s'_i}
	%=
	\bigO(n)\! +\!  \frac{2}{\mu}\sum_{i=1}^{2n/5} \frac{(4+5\mu)en^2}{4(1-p_c)i}
	\!=\! \bigO\left(\frac{n^2 \log{n}}{1-p_c}\right).
\end{align*}

\textbf{Phase 4: Cover $\frontP'$ entirely.}

%The second goal is to cover $F'$ entirely. Suppose this has not been done yet,
Let $y$ and $z$ be as defined in the proof of Theorem~\ref{thm:gsemo} and
additionally assume that $y$ has a positive crowding distance in $F^1_t$.
%
%Similarly to the previous goal and also to the previous proof,
Similarly to the proof of Theorem~\ref{thm:gsemo} and also to the argument of
the previous phase,
the probability
of selecting $y$ and then winning in at least one of two tournaments, and the subsequent mutation step creating $z$ is at least
$
\frac{4}{4+5\mu}\cdot\frac{1-p_c}{en^2}
=:s''$, and the probability of at least one success in $\mu/2$ trials
is at least $\frac{s''\mu/2}{s''\mu/2+1}$ \newedit{by Lemma~\ref{lem:lambda-trials}.}
%of $\mu$ binary tournaments
%is at least $(1-1/e^2)/5$, and then $z$ is created by a two-bit flip with
%at least probability $\frac{1-p_c}{en^2}$.
%The probability that we generate
%such an individual $z$ is bounded by $(1-p_c) \cdot \frac{\alpha}{5n^2} \cdot
%\left(1-\frac{1}{n}\right)^{n-2} \geq \frac{\alpha \cdot (1-p_c)}{5en^2}$
%if $y$ wins one of the $\mu$ binary tournaments. The individual $y$ is chosen
%with probability $\alpha$ for one of the $\mu$ binary tournaments and wins this
%tournament with probability at least $1/5$. %So the probability that we generate
%such an individual $z$ is bounded by
%\[
%(1-p_c) \cdot \frac{\alpha}{5n^2} \cdot \left(1-\frac{1}{n}\right)^{n-2} \geq
%\frac{\alpha \cdot (1-p_c)}{5en^2}
%\]
%from below.
%Since there are at most $2n/5 + 1$ individuals on the Pareto front with pairwise
%different fitness values and $\mu \geq 8n/5+4$ individuals in the population we
%have by lemma \ref{lem:nsga-ii-protect-layer} that for every $x \in P_t$ with $x
%\in \lfrontP$ there is a $w \in P_{t+1}$ with $w = x$.
Furthermore, applying Lemma~\ref{lem:nsga-ii-protect-layer} (ii)
%with $C=3n^2/5 and $D=2n/5$
%while noticing that $\mu\geq 2n+5>4(D+1)$
while noticing that $\mu\geq 2n+5>4(2n/5+1)$ and $F_t^1$ covers at most $2n/5+1$ distinct fitness vectors
implies that a copy of $z$
always survives in future generations.
%Consequently
So $2n/5$ such steps suffice to cover $F'$ and thus
the expected time to finish the phase
%in this phase
is no more than %bounded by
\[
\bigO(n) +
    \frac{2}{\mu}
    \cdot \frac{(4+5\mu)en^2}{4(1-p_c)}
    \cdot \frac{2n}{5}
= \bigO\left(\frac{n^3}{1-p_c}\right).
\]
%from above.

%So with all
% runs
%the goals together we obtain an expected generations of $\bigO(n^3/(1-p_c))$
% and
%therefore $O(\mu n^3/(1-p_c))$ fitness evaluations in expectation
%to construct the $F'$ set.

%\textbf{Construction of the $\frontP$ set}:
\textbf{Phase 5: Create the first search point in $\frontP$.}

%The argument for this phase is similar to the one in the proof of GSEMO
%except that we generalise it to the fixed population size $\mu$ of NSGA-II.
%%
%First, the probability of creating the first search point on $\frontP$
%assuming we have none on the front is
%$p_c \cdot \frac{1}{8}\cdot \frac{2n/5}{n+1}
%    \cdot \frac{n/5}{\mu}\cdot\frac{1}{\mu}
%    = \Omega(\frac{n p_c}{\mu^2})$,
%and the waiting time for this sub-phase is $\bigO(\mu^2/(n p_c))$.
%We first estimate the time to create the first individual on $F$.
%By the end of the previous phase,
Now $P_t$ contains all search
points of $\frontP'$ and they are in $F^1_t$.
%, it then follows from Lemma~\ref{lem:nsga-ii-full-front}
%that these search points form exactly $F^1_t$.
%
%Using the same notation as in the proof of Theorem~\ref{thm:gsemo}
% for $p_1$ and $p_2$,
There are at least $n/5$ solutions
of the form $0^i1^{3n/5}0^{2n/5-i}$
in $P_t$ with $i\leq n/5$ and
positive crowding distance, thus they win the tournament for selecting $p_1$
with probability at least $2 \cdot \frac{1}{5}\cdot\frac{n/5}{\mu}$.
%After this,
Then
it suffices
to select a specific
%search point
solution in $P_t\cap \frontP'$ with positive crowding
distance as $p_2$, \ie with probability at least $2 \cdot \frac{1}{5}\cdot\frac{1}{\mu}$, to
form %a pair of
compatible parents
%Thus the probability of selecting a pair of compatible parents for crossover is
%at least $\frac{1}{25} \cdot \frac{n/5}{\mu}\cdot\frac{1}{\mu}$.
%After this,
so that we have a probability
of at least $p_c \cdot \frac{2n/5}{n+1} \cdot (1-1/n)^n$ to create one offspring
%solution among two is
in $\frontP$. %Thus,
So in each creation of a pair, a
%search
point in $F$ is created with probability at least
$p_c \cdot \frac{4}{25} \cdot \frac{2n/5}{n+1}
 \cdot \frac{n/5}{\mu} \cdot \frac{1}{\mu} \cdot (1-1/n)^n =\Omega(\frac{n p_c}{\mu^2})=:s$,
and among $\mu/2$ pairs produced at least one success occurs with
probability at least $1 - \left(1 - s\right)^{\mu/2} \geq \frac{s\mu/2}{s\mu/2 +1}$ \newedit{(Lemma~\ref{lem:lambda-trials})} %.
%Therefore, only
%So, only
and at most
 $1+\frac{2}{\mu s}=\bigO(\frac{\mu}{n p_c})$
 generations are required
 in expectation
 %to create the first search %point in $\frontP$.
 for this phase.

\textbf{Phase 6: Cover $\frontP$ entirely.}

Once a search point in $\frontP$ is created, the process of covering $\frontP$
is similar to that of covering $\frontP'$ with only minor differences
\newedit{(e.\,g.\ applying Lemma~\ref{lem:size-non-dom-set} with $k=4n/5$ and then Lemma~\ref{lem:nsga-ii-protect-layer}
%with $C=4n^2/5,
%$ and $
%D=n/5$.
with $m=n/5+1$).}
The expected number of
generations in Phase~6 is $\bigO(\frac{n^3}{1-p_c})$.

%So the expected number of generations to complete the phase is
Summing up expected times of all the phases gives an upper bound
$\bigO\left(\frac{\mu}{n p_c} + \frac{n^3}{1-p_c} \right)$ %.
%which is also
on the expected number of generations to optimise \RRRMO. Multiplying this bound
with $\mu$ gives the result in terms of fitness evaluations.
%Having at least one search point on $F$, the probability of increasing
%the number of search points on the front by one is then at least
%$\frac{1-p_c}{\mu}\cdot\frac{1}{en^2}$ and the expected number of
%generations to cover the front is no more than $\bigO(\mu n^3/(1-p_c))$.
%
%So the expected number of generations of the phase is
%$\bigO\left(\frac{\mu^2}{n p_c} + \frac{\mu n^3}{1-p_c} \right)$.
%\andre{What are the number of generations? What are the number of fitness evaluations?}
%
%TODO: We should be able to handle the lost of search point in layer-1
%now, it remains to check how binary tournament and the production
%of multiple offspring affect the analysis.
\end{proof}

\newedit{\section{Multi-Objective Royal Road Functions for Uniform Crossover}}

\newedit{Now we turn to designing a royal road function for uniform crossover.
As already observed by~\citet{Jansen2005c}, royal road functions for uniform crossover are
harder to design than those for one-point crossover.
The reason is that one-point crossover
on two search points $x$ and $y$ can create $H(x, y)+1$ distinct offspring\footnote{For instance, a uniform crossover of $1^n$ and $0^n$ can create $n+1 = H(1^n, 0^n) +1$ offspring $1^i 0^{n-i}$ for $i \in \{0, \dots, n\}$. Since bits on which both parents agree are always copied to the offspring, these bits can be ignored in this consideration and the above argument applies to arbitrary parents $x, y$ and the subset of $H(x, y)$ bits that differ in $x$ and $y$.}, where $H(x, y)$ is the Hamming distance of $x$ and $y$.
%.
%However, for uniform crossover on $x$ and $y$ we have much more possibilities. %and
In contrast, uniform crossover creates one out of $2^{H(x, y)}$ possible distinct offspring since there are two possible choices for each bit on which $x$ and $y$ differ. If $H(x, y)$ is small then mutation is effective at creating any such offspring and there is no large benefit from using uniform crossover. To enforce a large benefit from crossover over mutation, the Hamming distance between relevant search points to be crossed must be large. But if we want uniform crossover to create a target point (e.\,g.\ a Pareto-optimal search point), this target set must be exponentially large in the Hamming distance.}

\newedit{\citet{Jansen2005c} already designed a ``real royal road'' function for uniform crossover with the described property, albeit for single-objective optimisation. Their idea was to split the bit string in two halves and to design a fitness gradient that guides evolutionary algorithms to evolve a specific bit pattern in the left half that must be kept intact to guarantee good fitness. The right half is divided into three equal-sized parts. In the right half the algorithm populates a plateau of search points with equal fitness in such a way that the population is likely to contain many pairs of search points that agree in their left half and are complementary in their second half. A uniform crossover applied to such a pair then maintains the good bit pattern in the left half while creating a uniform random pattern in the right half. This uniform pattern has a good probability to create an equal number of zeros and ones in all sub-parts. Such search points make up a target region of exponential size that is assigned a globally optimal fitness. In contrast, mutation operators are unable to jump into the target area in the same way as they would need to keep the bit pattern in the left half intact while flipping many bits in the right half (this argument exploits the subdivison of the right half in three sub-parts). For common mutation operators that treat all bits symmetrically, this has an exponentially small probability.}

%has a good chance to create an offspring that maintains the good bit pattern in the left half and has the same number of ones in target region is added in which
%in such a way that the algorithm is encouraged to evolve a

%the Hamming distance between
%$x$ and $y$ should be not to small.
%Otherwise mutation can
%create a point on the Pareto front with h. igh probability.
\newedit{Our approach for the design of our multiobjective \uRRRMOc function class extends the construction in~\citet{Jansen2005c} in several ways. Firstly, we lift the constrution from single-objective to multi-objective optimisation and ensure that there is a Pareto front of polynomial size.
Second, we aim to ensure that the function cannot be optimised by any unary unbiased mutation operator. Third, we also aim to ensure that hypermutation cannot optimise the function effectively.
To meet these aims, our design extends the construction in~\citet{Jansen2005c} by sub-dividing both the left part and the right part in four parts as explained in the following. In addition, we will consider a specific permutation of bits as this prevents hypermutation from reaching the target set efficiently while uniform crossover is agnostic to the order of bits and hence performs the same on every permutation of bits.}

\newedit{We first specify one function, denoted \uRRRMO, of the class \uRRRMOc where
the bit positions of these parts are consecutive in the bit string.
%On the other hand, for the functions of \uRRRMOc which are generalised from \uRRRMO,
Later on in the section,
we will generalise this function to the class,
%However note that both uniform crossover and
%bitwise mutation are \emph{unbiased variation operator}, \ie they treats the bit positions
%$\{1,\dots,n\}$ indifferently, as well as the bit values $\{0,1\}$. Thus this allows
%the generalisation of the function to a class of functions, also called the \uRRRMO
%\emph{function class},
for which
the consecutiveness of the bits in each part does not necessarily holds but still the
results for uniform crossover hold.}

\newedit{
Recall the notation for $[n]$, thus $[4] = \{1,2,3,4\}$.
\begin{definition}\label{def:the-subsets}
Fix a natural number $n$ which is divisible by $16$. For any search point
$x \in \{0,1\}^n$, we partition $x$ into substrings $\strl{x}$, $\strr{x}$ such that
    $x=(\strl{x},\strr{x})$
    where
        each $\strl{x}:=(\strl{x}^1,\strl{x}^2,\strl{x}^3,\strl{x}^4)$ and
             $\strr{x}:=(\strr{x}^1,\strr{x}^2,\strr{x}^3,\strr{x}^4)$
             has length $n/2$
             while $\strl{x}^i$, $\strr{x}^j$ have lengths $n/8$
             for $i,j\in[4]$. %\{1,2,3,4\}$. %and $j \in \{1,2,3,4\}$.
On the space $\{0,1\}^{n/2}$ of these substrings, we define the following subsets:
\begin{align*}
\text{For } \strl{x}\colon
    \subsetU &:= \left\{\strl{x} \in \{0,1\}^{n/2} \mid \forall i \in [4]\colon \ones{\strl{x}^i} \in [n/24,n/12]\right\}, \\
    \subsetP &:= \left\{\strl{x} \in \{0,1\}^{n/2} \mid \LO(\strl{x}) + \TZ(\strl{x})=n/2\right\},\\
\text{For } \strr{x}\colon
    \subsetC &:= \left\{\strr{x} \in \{0,1\}^{n/2} \mid \LO(\strr{x}) + \TZ(\strr{x})=n/2 \vee \LZ(\strr{x}) + \TOs(\strr{x})=n/2 \right\}, \\
    \subsetT &:= \left\{\strr{x} \in \{0,1\}^{n/2} \mid \forall j \in [4]\colon \ones{\strr{x}^j} = \zeros{\strr{x}^j}=(n/8)/2=n/16 \right\}.
\end{align*}
\end{definition}}
\newedit{The set $P$ forms a \emph{path} of search points $1^{i}0^{n/2-i}$ of cardinality $n/2+1$.
It is easy to see %that the subsets are disjoint,
%    \ie $\subsetU \cap \subsetP = \emptyset$ and $\subsetC \cap \subsetT = \emptyset$,
that if $\strl{x}\in\subsetP$ then there exists $\strl{y}\in \subsetP\colon \Hamming(\strl{x},\strl{y})=1$.
The set $C$ was called a \emph{circle} in~\citet{Jansen2005c} and it is a closed Hamming path of length~$n$.
If $\strr{x}\in\subsetC$ then there exists $\strr{y}\in \subsetC\colon \Hamming(\strr{x},\strr{y})=1$.
The set $T$ is a \emph{target} that contains all search points for which all sub-parts of $\strr{x}$ have an equal number of ones and zeros.
The cardinality of $\subsetT$ is $|\subsetT| = {n/8 \choose n/16}^4 = 2^{n/2} \cdot (\Theta(1/(\sqrt{n/8})))^4 = 2^{n/2} \cdot \Theta(1/n^2)$ as the size of the largest binomial coefficient is $\binom{n/8}{n/16} = 2^{n/8} \cdot \Theta(1/(\sqrt{n/8}))$. Finally, the set $U$ contains an exponential number of search points for which the ratio of zeros and ones is fairly \emph{uniform} in all sub-parts of $\strl{x}$. A bit string chosen uniformly at random from $\{0, 1\}^{n/2}$ is in $U$ with overwhelming probability, by a straightforward application of Chernoff bounds and a union bound over all four sub-parts.}

\newedit{
We will use $\subsetP$ and $\subsetT$ to define the Pareto-optimal solutions of \uRRRMO. These sets typically have to be reached from search points $x$ with $\strl{x} \in U$ and $\strr{x} \in C$. The following lemma shows that $U$ and $P$ have a large Hamming distance from one another, and the same applies to $C$ and $T$. Hence, large changes are required to create a Pareto-optimal solution from~$x$.
%, and
%the cardinalities of these sets are $|\subsetP| = n/2+1$ and $|\subsetT| = {n/8 \choose n/16}^4$.
%We also have the following %stronger
%properties which say search points crossing
%those subsets are far away from each other.
}

\newedit{\begin{lemma}\label{lem:prop-of-subsets}
The following properties hold for the subsets of Definition~\ref{def:the-subsets}:
\begin{enumerate}
%\item If $\strl{x}\in\subsetP$ then $\exists \strl{y}\in \subsetP\colon \Hamming(\strl{x},\strl{y})=1$.
%\item If $\strr{x}\in\subsetC$ then $\exists \strr{y}\in \subsetC\colon \Hamming(\strr{x},\strr{y})=1$.
\item[(i)] $\forall \strl{x}\in\subsetU, \forall \strl{y}\in\subsetP\colon \Hamming(\strl{x},\strl{y})\in [n/8,3n/8]$.
\item[(ii)] $\forall \strr{x}\in\subsetC, \forall \strr{y}\in\subsetT\colon \Hamming(\strr{x},\strr{y})\in [3n/16,5n/16]$.
\end{enumerate}
\end{lemma}}
\begin{proof}
%\newedit{\begin{enumerate}
%\item Any string $\strl{x}\in\subsetP$ has form $1^{i}0^{m-i}$ with $0\leq i\leq
%m$, therefore if $i<m$ then $\strl{y}$ can be the string $1^{i+1}0^{m-i-1}$,
%and if $i=m$ then $\strl{y}$ is the string $1^{m-1}0$.
%\item Similar argument as for the previous result, except here $\strr{x}$ can
%also be the string $0^{i}1^{m-i}$.
%\item
\newedit{For (i), the definition of $\subsetU$ implies
	$\ones{\strl{x}^j} \in [n/24,n/12]$ and thus also
	$\zeros{\strl{x}^j} \in [n/8-2n/24,n/8-n/24] = [n/24,n/12]$
	for every $j \in [4]$. By the
	definition of the path $\subsetP$, $\strl{y} = 1^a 0^{n/2-a}$ for some integer $a \in [0, n/2]$ and thus
    we have that $\strl{y}^j \in \{0^{n/8},1^{n/8}\}$ for at least three distinct indices $j \in [4]$.
    Hence, for such a $j$ we obtain $n/12 \geq \Hamming(\strl{x}^j,\strl{y}^j) \geq n/24$
    while for the remaining block $i$ (that differs to these three $j$) we get
    $n/8 \geq \Hamming(\strl{x}^i,\strl{y}^i) \geq 0$.
    Adding the distances of each block gives
    $\Hamming(\strl{x},\strl{y})\in [3(n/24),3(n/12)+n/8]=[n/8,3n/8]$.}

\newedit{Similarly for (ii), by the definition of the circle $\subsetC$
    either $\strr{x} = 1^a 0^{n/2-a}$ or $\strr{x} = 0^a 1^{n/2-a}$
    for some integer $a \in [0, n/2]$, we have that there are
    at least three indices $j\in[4]$ such that $\strr{x}^j\in \{1^{n/8},0^{n/8}\}$,
    while by the definition of the target $\subsetT$ we have
    $\ones{\strr{y}^j} = \zeros{\strr{y}^j}=n/16$.
    Therefore,  for these three $j$ indices we have
    $\Hamming(\strr{x}^j,\strr{y}^j)=n/16$.
    Let $i$ be the remaining index, then $n/8\geq \Hamming(\strr{x}^m,\strr{y}^m)\geq 0$,
    and adding the distances implies
    $\Hamming(\strr{x},\strr{y})
    %=\Hamming(\strr{x}^i,\strr{y}^i)
    % + \Hamming(\strr{x}^j,\strr{y}^j)
    % + \Hamming(\strr{x}^k,\strr{y}^k)
    % + \Hamming(\strr{x}^m,\strr{y}^m)
    \in [3(n/16),3(n/16)+n/8] = [3n/16,5n/16]$.
}
%\end{enumerate}}
\end{proof}

\newedit{
\begin{definition}\label{def:urrrmo}
Based on the partition and the subsets from Definition~\ref{def:the-subsets}, we define a function
\begin{align*}
f(x)=(f_1(x),f_2(x)) &:=
	\begin{cases}
		(\LO(\strr{x}),\frac{n}{2} + \TZ(\strr{x})) \!\!\! & \text{ if } \LO(\strr{x}) \neq 0,\\
		(\frac{n}{2}+\LZ(\strr{x}),\TOs(\strr{x})) & \text{ otherwise.}
	\end{cases}
\end{align*}
Then our \uRRRMO function is: %defined by: %Then the royal road function for uniform crossover in the multiobjective setting is defined by:
\begin{align*}
& \uRRRMO(x) :=
	\begin{cases}
		%(1/2,0)
        %   &\text{ if }  \ones{\strl{x}^1} \in [n/12,n/6], \ones{\strl{x}^2} \notin [n/12,n/6] \text{ and } \strr{x} \notin C, \\
		f(x)
            &\text{ if } \strl{x} \in \subsetU \wedge \strr{x} \notin \subsetC, \\
		f(x) + (2n - \vert{\strl{x}}\vert_1)\cdot\vec{1}
            &\text{ if } \strr{x} \in \subsetC, \\
		(\LO(\strl{x}),\TZ(\strl{x})) + 3n \cdot \vec{1}
            &\text{ if } \strl{x} \in \subsetP \wedge \strr{x} \in \subsetT,\\
		(0,0)
            &\text{ otherwise.}
	\end{cases}
\end{align*}
\end{definition}}
%\andre{I added an subdivision for $\strl{x}$ to be able to prove a negative result for elitist unbiased black box algorithms. If you have a better idea to write this function in a more elgant way then feel free to change, but this simplifies also the negative result for NSGA-II.}
%\cuong{The two first cases of \uRRRMO look strange, like they are overlapping.} \andre{Yes, you are right. I changed the function a little bit.}

\newedit{%By a Chernoff bound
    Algorithms that initialise their population uniformly
at random will typically start with search points $x$ with $\strl{x}\in\subsetU$
and  $\strr{x} \notin \subsetC$, i.\;e.\ with fitness $f(x)$. Then the fitness
gives signals to gradually generate a search point $x$ with $\strr{x} \in \subsetC$.
After this the number of ones in $\strl{x}$ will be minimized since the fitness
increases in both components with decreasing $\vert{\strl{x}}\vert_1$. This
gives a search point $x$ with $\strl{x}=0^{n/2}$ and $\strr{x} \in \subsetC$. The fitness
of $x$ is then at least $2n$ in every objective and thus $x$ strictly dominates
every search point $y$ with fitness $f(y)$ or smaller. Moreover, two search
points $x,y$ with $x \neq y$, $\strr{x},\strr{y} \in C$ and $\strl{x},\strl{y} =
0^{n/2}$ are incomparable with respect to $\uRRRMO$. %:
%\begin{itemize}
%	\item If $0<\LO(\strr{x})<\LO(\strr{y})$ then $\TZ(\strr{x})=m-\LO(\strr{x})>m-\LO(\strr{y}) = \TZ(\strr{y})$ and thus $f_1(\strr{x})<f_1(\strr{y})$ and $f_2(\strr{x})>f_2(\strr{y})$.
%	\item If $0 <\LZ(\strr{x})<\LZ(\strr{y})$ then $\TOs(\strr{x}) = m - \LZ(\strr{x}) > m - \LZ(\strr{y}) =\TOs(\strr{y})$ and thus $f_1(\strr{x})<f_1(\strr{y})$ and $f_2(\strr{x})>f_2(\strr{y})$.
%	\item If $\LO(\strr{x})>0$ and $\LO(\strr{y})=0$ (i.e. $\LZ(y)>0$) then $f_1(x) \leq n/2 < f_1(y)$ and $f_2(y) \leq n/2-1<f_2(x)$.
%\end{itemize}
%Thus the algorithm is able to generate all points $x$ with $\strl{x}=0^{n/2}$ and $\strr{x} \in C$.
All search points $z$ with $\strl{z} \in P$ and $\strr{z} \in T$ have fitness at
least $3n$ in each objective and thus dominates every other search point of the search
space, hence they are Pareto-optimal.
Particularly, an optimal search point $z^*$ with $\strl{z^*}=0^{n/2}$ and $\strr{z^*}
\in \subsetT$ can be created by uniform crossover between the two search points
$x,y$ previously mentioned if their $\strr{x}$ and $\strr{y}$ parts are complement
strings.}
%$w$ with $\LO(\strl{w})+\TZ(\strl{w}) \neq m$ or $\strr{w} \notin T$.

\newedit{Since two distinct search points with fitness at least $3n$ in each objective
are incomparable, the set of all possible non-dominated fitness \brandnewedit{vectors} of $\uRRRMO$ is
\[
\{(3n+k,3n+n/2-k) \mid k \in \{0, \dots, n/2\}\},
\]
and the set of all Pareto-optimal solutions for $\uRRRMO$ is $\subsetW := \{x \in \{0,1\}^n \mid \strl{x}\in\subsetP \wedge \strr{x}\in\subsetT\}$.
\andre{Introduced a new notation for all Pareto-optimal solutions with a new macro. Have also integrated this into all the proofs.}
}
%So one sees that points $x$ with $x'' \in C$ are far from a pareto optimal solution. In particular standard bitwise mutation is very inefficient here. However, the key is that uniform crossover on the search points $0^{n/2}0^i1^{n/2-i}$ and $0^{n/2}1^i0^{n/2-i}$ needs only polynomial time.
%Note that two points $x,y$ with $x \neq y$, $x'',y'' \in C$ and $x',y' = 0^{n/2}$ are incomparable with respect to $g:=RRMO^*$: If $x'=1^i0^{m-i}$ and $y'=1^j0^{m-j}$ for $i,j \in \{1,\ldots,m\}$ then $g(x) = (i+n/2,n+m-i)$ and $g(y)=(j+n/2,n+m-j)$. If $x'=1^i0^{m-i}$ and $y'=0^j1^{m-j}$ for $i,j \in \{1,\ldots,m\}$ then $g(x) = (i+n/2,n+m-i)$ and $g(y)=(n+j,n/2+m-j)$. If $x'=0^i1^{m-i}$ and $y'=0^j1^{m-j}$ for $i,j \in \{1,\ldots,m\}$ then $g(x) = (n+i,n/2+m-i)$ and $g(y)=(n+j,n/2+m-j)$. \\
%
%Similar to the section about \RRRMO.
Thus we can bound the number of non-dominated solutions of \uRRRMO contained in any population as follows.
\begin{lemma}\label{lem:size-non-dom-set2}
	If $S$ is a set of non-dominated solutions of $\uRRRMO$ \newedit{(i.\;e.\ for $x,y \in S$ \brandnewedit{with $x \neq y$} we neither have $x \succeq y$ nor $y \succeq x$ with respect to $\uRRRMO$)} with positive fitness then $\vert{S}\vert \leq n$. %If we are on the cycle, we have \vert{S}\vert=n!
\end{lemma}
\begin{proof}
Let $h:=\uRRRMO$. As in the proof of Lemma~\ref{lem:size-non-dom-set} we see
that two search points in $S$ have different $h_1$-values. We may assume
that $h(x) \neq (0,0)$
%or $h(x) \neq (1/2,0)$
for every $x$ since otherwise $S$ consists of only one point. Note that
for each $x \in S$ one of the following statements holds:
    (i) $h(x)=f(x)$,
    or (ii) $h(x)=f(x)+(2n-\ones{\strl{x}})\cdot\vec{1}$,
    or (iii) $h(x)=(\LO(\strl{x}),\TZ(\strl{x}) + 3n\cdot\vec{1}$.
If (i) or (ii) holds, then $\vert{S}\vert \leq n$ since the range of $f_1$ is
$1, \ldots ,n$. If (iii) holds we also obtain $\vert{S}\vert \leq n$ since there
are at most $n/2+1$ distinct values for $\LO(\strl{x})$, \ie for the first
component of $h$.
\end{proof}

%\todo{@Cuong: define the function class \uRRRMO by apply a permutation $\sigma$ of the bit position and an $\mathrm{XOR}$ to a target string $x^*$ to all input string before apply the function \uRRRMO.}
\newedit{As explained at the beginning of this section, we now generalise \uRRRMO to the class \uRRRMOc of functions.
%with the same name.
\begin{definition}\label{def:urrrmo-class}
Let $\Pi_n$ be the set of all permutations of $[n]$,
%, we fix a $\sigma\in\Pi_n$
%(\ie $\sigma\colon [n]\rightarrow[n]$ is a bijection) and a $z \in \{0,1\}^n$,
%and
the class \uRRRMOc of functions is defined as:
%A pseudo-Boolean function $f$ is said to belong to the class
%\uRRRMO of functions if there exist a permutation $\sigma\colon[n]\rightarrow[n]$
%and a string $z\in\{0,1\}^n$ such that for all $x \in \{0,1\}^n$:
\begin{align*}
\uRRRMOc
    &:=\{\uRRRMOx{\sigma}{z}\}_{\sigma\in\Pi_n,z\in\{0,1\}^n}.\\
\text{where each }
\uRRRMOx{\sigma}{z}(x)
    &:= \uRRRMO(
                %(x_{\sigma(1)},\dots,x_{\sigma(n)})
                \permut{\sigma}{x} \oplus z).
\end{align*}
\end{definition}}

\newedit{Note that if $\mathrm{e}$ is the identity permutation, \ie $\mathrm{e}(i)=i$ for
all $i\in[n]$, then $\uRRRMOx{\mathrm{e}}{0^n}\equiv\uRRRMO$.
%\andre{Todo for Dirk to rephrase}\cuong{I have also fixed this, but the fix was deleted.}
%The performance of an algorithm on the function class $\uRRRMOc$,
%such as the expected running time is considered to be the worst case performance,
%\ie the most difficult function to optimise in the class.
As discussed in preliminaries, the behaviour of an algorithm that only employs
unbiased variation operators, like standard bit mutation and uniform crossover,
remains the same across every function of \uRRRMOc. Therefore
in the proofs of the results for such an algorithm, the choice of which function
does not matter to the expected running time and we will automatically assume the
specific function \uRRRMO.}

\newedit{\subsection{Hardness of \uRRRMOc for EMOs without Uniform Crossover}}

\newedit{We first show that without uniform crossover, both GSEMO and NSGA-II
\brandnewedit{using standard bit mutation}
are inefficient in optimising \uRRRMOc functions. For this purpose, we define
the following subset of the search space $\{0,1\}^n$ where the algorithms will
typically start in and then likely remain there for long time if crossover
is disabled:
%start in, \ie the structure $\subsetU$ is typical for a random initial point,
%and then is stuck with, \ie bitwise mutation has no issue in involving the
%structure $\subsetC$, when the crossover is disabled:
\begin{align*}
\subsetK
  &:= \{x \in\{0,1\}^n \mid \strl{x}\in\subsetU \vee \strr{x}\in\subsetC\},
\end{align*}
%and we
Note that any point $x\in\subsetK$ is not Pareto-optimal,
\brandnewedit{i.\;e.\ $\subsetK \cap \subsetW = \emptyset$,}
because
$\subsetU\cap\subsetP=\emptyset$ and $\subsetC\cap\subsetT=\emptyset$, but
%for
any $y \in \subsetW$ requires that $\strl{y} \in P$ and $\strr{y} \in T$.
Since $K$ includes all search points falling into the first two cases in the
definition of \uRRRMO, the fitness of all search points that are neither Pareto
optimal nor in $K$ is $(0, 0)$. Hence, every search point in $K$ dominates every
search point outside of~$K$, except for Pareto optima.}

\newedit{
\begin{theorem}\label{thm:uRRMO-gsemo-pczero-stdbit}\label{thm:uRRMO-nsgaii-pczero-stdbit}
\newedit{The following algorithms }
requires $n^{\Omega(n)}$ fitness evaluations in expectation to find any Pareto-optimal search point of any function of the class $\uRRRMOc$,
\begin{itemize}
\item GSEMO (Algorithm~\ref{alg:gsemo}) with $p_c=0$ and standard bit mutation,
\item NSGA-II (Algorithm~\ref{alg:nsga-ii}) with $\mu\in\poly(n)$, $p_c=0$ and standard bit mutation.
\end{itemize}
%. The same holds for NSGA-II (Algorithm~\ref{alg:nsga-ii}) with $\mu \in \text{poly}(n)$, $p_c=0$ and standard bit mutation.
\end{theorem}}
\begin{proof}
\newedit{
We first show the result for GSEMO.
By a Chernoff bound and a union bound, the probability that the initial search
point $x(0)$ has that $\strl{x(0)}^j \notin [n/24,n/12]$ for one $j \in [4]$ is at
most $2^{-\Omega(n)}$. Thus with probability $1-2^{-\Omega(n)}$, we have that
$x(0)\in \subsetK$.
%Let $K':=\{x \mid \ones{\strl{x}^i} \in [n/12,n/6] \text{ for } i \in \{1,2,3\} \text{ or } x_r \in C\}$. The probability of initializing the first search point in $K'$ is $1-e^{-\Omega(n)}$.
Assume that this holds, then, as discussed above, the algorithm will only accept points from
$\subsetK$ or from $\subsetW$. This means that any point $y \in \subsetW$ (i.\;e.\ with $\strl{y}\in \subsetP \wedge \strr{y}\in \subsetT$),
can only be created from a solution $x \in \subsetK$ by standard bit mutation.
We consider the following cases.}
\newedit{If $\strl{x} \in \subsetU$, then by Lemma~\ref{lem:prop-of-subsets}
we have $\Hamming(\strl{x},\strl{y})\geq n/8$, thus it is necessary that
at least $n/8$ bits in $\strl{x}$ have to be flipped to create any $\strl{y}\in \subsetP$. This happens with probability %creating any $\strl{y}\in \subsetP$ %\ie a necessary condition
%for the Pareto-optimality of $y$,
%by standard bit mutation is then
at most $|\subsetP|(1/n)^{n/8}=(n/2+1)(1/n)^{n/8}=n^{-\Omega(n)}$ by a union bound.}
% which is a necessary condition to create a Pareto-optimal individual.
%and by optimistically assume that the ideal changes will happen to the other
%part $\strr{x}$ of $x$.
%
\newedit{Otherwise if $\strr{x}\in\subsetC$, then again by Lemma~\ref{lem:prop-of-subsets}
we get $\Hamming(\strr{x},\strr{y})\geq 3n/16$. Thus in this case it is necessary
to flip at least $3n/16$ bits and by a union bound %on the cardinality of $\subsetT$
%
%. Therefore, even if we optimistically
%assume the ideal changes happen to $\strl{x}$ of $x$,
the probability of creating
$\strr{y}\in \subsetT$ from $\strr{x}$,
%\ie also a necessary condition for the Pareto-optimality of $y$,
%(and hence having $y\in\subsetK$)
%by a union bound and by Lemma~\ref{lem:binom-coeff}
is at most
\begin{align*}
|\subsetT|(1/n)^{3n/16}
    \leq 2^{n/2} n^{-3n/16}
    = n^{-3n/16+n/(2\log{n})}
    = n^{-\Omega(n)}%.
\end{align*}
since $\subsetT \subset \{0,1\}^{n/2}$ and $|\{0,1\}^{n/2}| = 2^{n/2}$. Thus overall,
the expected number of fitness evaluations to create
$y \in \subsetW$ is at least
%
%we also have to flip at least $k/2 = n/12$ bits during one mutation step in order to create one Pareto-optimal search point $z$ (since one of the strings $\strr{z}^1$, $\strr{z}^2$ or $\strr{z}^3$ is of the form $0^k$ or $1^k$ and therefore $H(x,z) \geq k/2$). Therefore we need at least
$(1-e^{-\Omega(n)})n^{\Omega(n)} = n^{\Omega(n)}$ by the law of total probability.}

\newedit{For NSGA-II, the probability of having any solution $x$ of the initial population
$P_0$ with $|P_0|=\mu$ and $\strl{x}\notin U$ is at most $\mu \cdot 2^{-\Omega(n)}=o(1)$
by a Chernoff bound and a union bound, given $\mu\in\poly(n)$.
Assume that this does not happen,
%\ie with probability $1-o(1)$,
then
in every generation of NSGA-II no solution created with fitness $(0,0)$ can survive
%the truncated survival selection
%(among the $2\mu$) with strictly positive fitness.
to the next generation.
%By classical Chernoff bounds we see that with probability $e^{-\Omega(n)}$ every search point $x$ initializes with $\LO(\strl{x}) + \TZ(\strl{x}) \neq m$ and $\strr{x} \notin C$ by a union bound. Suppose that this happens.
%Let $K'$ be as in the proof of Theorem~\ref{thm:uRRMO-gsemo-pczero}. The probability of initializing all $\mu$ search points in $K'$ is at least $1-\mu e^{-\Omega(n)} = 1-e^{-\Omega(n)}$ since $\mu$ is polynomial.
%
Therefore, the algorithm will only accept search points from $\subsetK$
or from $\subsetW$, and this means that the first $y \in \subsetW$ can only be created from a parent $x \in \subsetK$ by standard bit mutation.
As in the case for GSEMO, the probability of such an event is at most $n^{-\Omega(n)}$.
%Hence, in order to create one Pareto-optimal point it is necessary to flip $n/12$ bits in one mutation step. This happens with probability at most $n^{-\Omega(n)}$. Since there are $\mu$ mutation steps in one generation, the expected number of generations to optimize $\uRRRMO$ with NSGA-II is at least $(1-e^{-\Omega(n)})(n^{\Omega(n)}/\mu) = n^{\Omega(n)}$.
Hence, overall the expected number of fitness evaluations to find
the first $y \in W$ is then again at least $(1-o(1))n^{\Omega(n)}
=n^{\Omega(n)}$.}% by the law of total probability.
%fitness evaluations in expectation to optimize $\uRRRMO$.}
\end{proof}

\subsection{Both Unary Unbiased Variation and Hypermutation Fail on \uRRRMOc}

\newedit{We now generalise the results of the previous section to unary unbiased
variation and hypermutation operators, and also to elitist black-box algorithms.
The following theorem shows that without uniform crossover, these algorithms and
operators require an exponential expected number of fitness evaluations to optimise
the function class \uRRRMOc, particularly in the worst case even to find any
Pareto-optimal solution.}

\newedit{
\begin{theorem}
\label{thm:uRRMO-gsemo-pczero-unbiasedvar}\label{thm:uRRMO-nsgaii-pczero-unbiasedvar}\label{thm:uRRMO-blackbox-unbiasedvar}
\label{thm:uRRMO-gsemo-pczero-hypermut}\label{thm:uRRMO-nsgaii-pczero-hypermut}\label{thm:uRRMO-blackbox-hypermut}
The following algorithms
requires $2^{\Omega(n)}$ fitness evaluations in expectation to
%find any Pareto-optimal search point of any function of the class $\uRRRMOc$:
optimise the function class $\uRRRMOc$:
\begin{itemize}
\item GSEMO (Algorithm~\ref{alg:gsemo}) with $p_c=0$,
\item NSGA-II (Algorithm~\ref{alg:nsga-ii}) with $\mu\in\poly(n)$ and $p_c=0$,
\item Any ($\mu$+$\lambda$) elitist black-box algorithm (Algorithm~\ref{alg:muplus-blackbox}),
\end{itemize}
%that use either:
that are allowed to use any of the following mutation operators in each step:
\begin{itemize}
\item any unary unbiased variation operator (Algorithm~\ref{alg:unary-unbiased-op}),
\item hypermutation (Algorithm~\ref{alg:hypermutation}) using any parameter $r\in(0,1]$.
\end{itemize}
%(for choosing the distribution $D_t$ in the case of the elitist black-box algorithm).
\end{theorem}}

%\newedit{In order to prove a negative result for blackbox algorithms with unbiased variation operators we first show that the probability to create a Pareto-optimal search point $y$ from $x \in K$ via an unbiased variation operator is exponentially small.}
\newedit{To prove the results for unary unbiased variation operators, it suffices
to look at the basic function $\uRRRMO$ of the class, and we will use the same set
$K$ as defined in the previous section for this. The drawback of these operators is
depicted in the following lemma which shows that the probability of creating any
Pareto-optimal search point from any $x \in K$ by such an operator is exponentially
small.}

\newedit{\begin{lemma}\label{lem:uRRMO-unbiasedvar}
%	Let $x \in K$. Then the probability to create any $y \in \subsetW$ from $x \in K$ via an unbiased variation operator is at most $2^{-\Omega(n)}$.
Let $\mathrm{op}$ be any unary variation operator according to Algorithm~\ref{alg:unary-unbiased-op},
then for every $x\in\subsetK$ it holds that
%\begin{align*}
$\prob{\mathrm{op}(x)\in\subsetW} = 2^{-\Omega(n)}$.
%\end{align*}
\end{lemma}}
\begin{proof}
	\newedit{
    %Let $\mathrm{op}$ be a unary unbiased variation operator according to
	%Algorithm~\ref{alg:unary-unbiased-op}.
	%For any way of partitioning $x$ into parts
	%$x_m$ and $x_k$, assume that $\mathrm{op}$ has sampled its parameter
	%$r$ such that $r=s+q$ for any non-negative integers $s \leq \vert{x_m}\vert$ and $q \leq \vert{x_k}\vert$ and even that
	%it has flipped exactly $s$ bits in $x_m$, then the $q$ bits to be flipped
	%in $x_k$ remain uniformly distributed,
	%\ie any particular choice has probability $1/{|\strr{x}| \choose q}$
	For any substring $x_{\mathrm{sub}}$ of $x$, conditional on $\mathrm{op}$ flipping $r$ bits in $x$ and, among these, $s$ bits in $x_{\mathrm{sub}}$, the positions of the $s$ bits to be flipped in $x_{\mathrm{sub}}$ are uniformly distributed in $x_{\mathrm{sub}}$.
	%
	%of partitioning $x$ into parts
	%$x_m$ and $x_k$, assume that $\mathrm{op}$ has sampled its parameter
	%$r$ such that $r=s+q$ for any non-negative integers $s \leq \vert{x_m}\vert$ and $q \leq \vert{x_k}\vert$ and even that
	%it has flipped exactly $s$ bits in $x_m$, then the $q$ bits to be flipped
	%in $x_k$ remain uniformly distributed,
	%\ie any particular choice has probability $1/{|\strr{x}| \choose q}$
	%\todo{Check and see if we can prove this formally}.
	%
	Let $y:=\mathrm{op}(x)$, we consider two cases:}

\newedit{\textbf{Case $\strl{x}\in \subsetU$:}
By Lemma~\ref{lem:prop-of-subsets} we get
that $\Hamming(\strl{x},\strl{y})\in [n/8,3n/8]$, therefore even if we optimistically
assume that
%$\mathrm{op}$ has mutated some $q$ bits of $\strr{x}$ correctly and ideally
%such that
the number of
%remaining number of
bits to be flipped in $\strl{x}$ is %exactly
$s=\Hamming(\strl{x},\strl{y})$ (otherwise the sought probability is zero),
still such a number of specific bits need to be flipped on $\strl{x}$ to create
$\strl{y}\in\subsetP$. By a union bound and the well-known inequality
${n \choose k}\geq \left(\frac{n}{k}\right)^k$,
%and by Lemma~\ref{lem:binom-coeff},
the probability of creating any $\strl{y}\in\subsetP$ by $\mathrm{op}(x)$
(a necessary condition for $y \in \subsetW$) in one offspring
production is at most:
\begin{align*}
	|\subsetP| \left(1/{n/2 \choose \Hamming(\strl{x},\strl{y})}\right)
	\leq |\subsetP| \left(1/{n/2 \choose n/8}\right)
	\leq (n/2+1) \left(\frac{n/2}{n/8}\right)^{-n/8}
	= 4^{-\Omega(n)}
	\leq 2^{-\Omega(n)}.
\end{align*}
%
%Since $\strl{y}^1 \in \{0^{n/4},1^{n/4}\}$ or $\strl{y}^2 \in \{0^{n/4},1^{n/4}\}$, we have $n/12 \leq H(x,y) \leq 5n/12$. By unbiasedness, the probability of generating $y$ is at most $1/\binom{n/2}{H(x,y)} \leq 1/\binom{n/2}{n/12} \leq (n/12)^{n/12}/(n/2)^{n/12} =6^{-n/12}$. By a union bound on all possible $\strl{y}$ we obtain a probability of at most $(n/2+1) \cdot 6^{-n/12}=2^{-\Omega(n)}$ to generate a Pareto-optimal search point.\\
}

%\newedit{Otherwise if $\strr{x}\in \subsetC$:
%By Lemma~\ref{lem:prop-of-subsets} we get that $\Hamming(\strr{x},\strr{y})\in [n/6,n/3]$,
%and even if we optimistically assume that $\strl{x}$ has been mutated correctly and
%$r$ and $\radl$ have been sampled ideally so that $\radr = r - \radl = \Hamming(\strr{x},\strr{y})$,
%some $\radr$ bits need to be flipped on $\strr{x}$ to create $\strr{y}\in\subsetT$.
%The probability of creating $y\in\subsetK$ by $\mathrm{op}(x)$ during one offspring
%production is then at most
%\begin{align*}
%|\subsetT| \left(1/{n/2 \choose \Hamming(\strr{x},\strr{y})}\right)
%  &\leq |\subsetT| \left(1/{n/2 \choose n/6}\right)
%   \leq {n/6 \choose n/12}^3 \left(\frac{n/2+1}{2^{n\entropy(1/3)}}\right)\\
%  &\leq (n/2+1)(4^{n/12})^3 2^{((1/3)\log{3} + (2/3)\log(3/2))n}\\
%  &\leq (n/2+1)2^{n/2} 2^{-0.91n}
%   =     (n/2+1) 2^{-0.41n}
%   = 2^{-\Omega(n)}.
%\end{align*}
%by a union bound and by Lemma~\ref{lem:binom-coeff}.}

\newedit{\textbf{Case $\strr{x}\in \subsetC$:}
Let $B$ be the event that $y \in \subsetW$ is created from $x\in K$,
and $A_{s,q}$ be the event that $s$ and $q$ bits are flipped in $\strl{x}$ and $\strr{x}$,
respectively, by $\mathrm{op}(x)$ for any fixed $s,q \in \{0, \dots, n/2\}$.
By the law of total probability, it holds that
\begin{eqnarray}\label{eq:PR(B)}
	\Pr(B) = \sum_{s=0}^{n/2}\sum_{q=0}^{n/2}\Pr(B \mid A_{s,q}) \cdot \Pr(A_{s,q}).
\end{eqnarray}}
\newedit{We have $\Pr(A_{s,q}) \leq \binom{n/2}{s} \binom{n/2}{q}/\binom{n}{s+q}$ for all $0 \le s+q \le n$ because
if for the chosen Hamming radius $r$ we have $r \neq s+q$ then $\Pr(A_{s,q})=0$ and otherwise
we have an equality due to the unbiasedness.
%As explained earlier (also on the unbiasedness),
Also due to this unbiasedness
for any $s \in \{0, \dots, n/2\}$
the probability of generating one individual $y$ with a specific $\strl{y}$-part if $s$
bits on the left are flipped is $1/\binom{n/2}{s}$. Then by a union bound over
the points of $\subsetP$, we obtain for any fixed $q, s \in \{0, \dots, n/2\}$
that %:
%\begin{align}\label{eq:PR(B,A)}
$
\Pr(B \mid A_{s,q})
%\leq |\subsetP| \cdot \frac{1}{\binom{n/2}{s}} = \frac{n/2+1}{\binom{n/2}{s}}.
\leq |\subsetP|/\binom{n/2}{s} = (n/2+1)/\binom{n/2}{s}
%\end{align}}
$.}
\newedit{Furthermore, by Lemma~\ref{lem:prop-of-subsets} we have that
$\Pr(B \mid A_{s,q})=0$ for $q \notin [3n/16,5n/16]$. Thus the inner sum of \eqref{eq:PR(B)}
for any $s$ is at most}
\newedit{
\begin{eqnarray*}
\sum_{q=0}^{n/2} \Pr(B \mid A_{s,q}) \cdot \Pr(A_{s,q}) \leq \sum_{q=3n/16}^{5n/16} \frac{\binom{n/2}{s} \cdot \binom{n/2}{q}}{\binom{n}{s+q}} \cdot \frac{n/2+1}{\binom{n/2}{s}} = (n/2+1) \sum_{q=3n/16}^{5n/16} \frac{\binom{n/2}{q}}{\binom{n}{s+q}}.
\end{eqnarray*}}

\newedit{
We define
%$M(s)
%    \coloneqq \max\{\binom{n/2}{q}/\binom{n}{s+q} \mid 3n/16 \le q\le 5n/16\}$
%% Cuong: better to use a function with two variables, otherwise we have to put max everywhere
$M(s,q)
\coloneqq \binom{n/2}{q}/\binom{n}{s+q}$
and claim that $M(s,q) = 2^{-\Omega(n)}$ for all integers
$s\in[0,n/2]$, $q\in[3n/16,5n/16]$. This implies the statement because then
$\Pr(B)\leq (n/2+1)
\sum_{q=3n/16}^{5n/16} M(s,q) \le n^2 \cdot 2^{-\Omega(n)} = 2^{-\Omega(n)}$.
We consider two subcases.}

\newedit{If $0\leq s \leq 3n/8$:
We have $q\leq s+q \leq 3n/8 + 5n/16 = 11n/16$
and also $n-q\geq n-5n/16=11n/16$, thus $q\leq s+q\leq n-q$.
%We know from properties of binomial coefficients that $\binom{n}{k} = \binom{n}{n-k}$ and that $\binom{n}{k}$ is non-decreasing when $k$ approaches $n/2$. Together, this %implies $\binom{n}{k} \le \binom{n}{k'}$ for all $k \le k' \le n-k$
% and, applied to our context,
%$\binom{n}{s+q} \geq \binom{n}{q}$.
This means that $\binom{n}{s+q}$ is not farther away from the central binomial
coefficient $\binom{n}{n/2}$ than $\binom{n}{q}=\binom{n}{n-q}$, thus the former coefficient
is no less than the latter, \ie $\binom{n}{s+q}\geq\binom{n}{q}$.
Hence, we get
\begin{align}\label{eq:M(s,q)}
M(s,q)
&\leq \binom{n/2}{q}/\binom{n}{q}
= \frac{(n/2)!}{n!} \cdot \frac{(n-q)!}{(n/2-q)!}
= \prod_{j=1}^{n/2} \frac{n/2-q+j}{n/2+j}
%= \prod_{j=1}^{n/2} \left(1 - \frac{q}{n/2+j}\right)
\overset{q\geq 3n/16}{\leq} \prod_{j=1}^{n/2} \frac{n/2-3n/16+j}{n/2+j}\nonumber\\
&=\prod_{j=1}^{n/2} \frac{5n/16+j}{n/2+j}
=\prod_{j=1}^{n/2} \left(1 - \frac{3n/16}{n/2+j}\right)
\le \prod_{j=1}^{n/2} \left(1 - \frac{3n/16}{n/2+n/2}\right)
%    & = \prod_{j=1}^{n/2} \frac{(n/2-3n/16+j)}{n/2+j}\\
%    & = \prod_{j=1}^{n/2} \left(1 - \frac{3n/16}{n/2+j}\right)\\
%    & \leq \prod_{j=1}^{n/2} \left(1 - \frac{3n/16}{n/2+n/2}\right)\\
= \left(1 - \frac{3}{16}\right)^{n/2} = 2^{-\Omega(n)}.
\end{align}}
% -- no need, i hope the computation is clear now since everything is like a monotone function with a single variable
%\newedit{The second inequality comes from the fact that $\frac{(n-q)!}{(n/2-q)!} = (n-q) \cdot \ldots \cdot (n/2-q+1)$ is monotone decreasing in $q$, i.e. the largest summand in the sum consisting of $(n/8+1)$ summands in obtained for $q=3n/16$.}

\newedit{If $3n/8+1\leq s \leq n/2$:
%Since $s+q \geq n/2$ for every $q \in [3n/16,5n/16]$
This implies $s+q > 3n/8 + 3n/16=9n/16>n/2$
and because
%$\binom{n}{s+q}$ is decreasing in $s$ and we use $s \leq n/2$ to conclude that
$\binom{n}{k}$ is a decreasing function of $k$ when $k\geq n/2$ and given $s\leq n/2$,
it holds that
$\binom{n}{s+q} \geq \binom{n}{n/2+q}$. Hence,
%, similarly to the calculations \eqref{eq:D(s)} done in the first case,
we have
\begin{align*}
M(s,q)
&\leq \frac{\binom{n/2}{q}}{\binom{n}{n/2+q}}
= \frac{(n/2)!}{n!} \cdot \frac{(n/2+q)!}{q!}
= \prod_{j=1}^{n/2} \frac{q+j}{n/2+j}
\overset{q\leq 5n/16}{\leq} \prod_{j=1}^{n/2} \frac{5n/16+j}{n/2+j}%\\
%    &\leq \prod_{j=1}^{n/2} \frac{5n/16+n/2}{n/2+n/2}
%    = \left(\frac{13}{16}\right)^{n/2}
\overset{\eqref{eq:M(s,q)}}{=} 2^{-\Omega(n)}. \qedhere
\end{align*}}
\end{proof}

\newedit{Regarding hypermutations, we remark their following ability to
imitate uniform crossover on an ideal function (\eg \uRRRMO), but also their
drawbacks on other functions of the class.}
\newedit{Uniform crossover keeps the bit values of bit positions where both
parents agree. On all disagreeing bits, uniform crossover generates uniform
random bit values. Hypermutation can simulate the latter effect of uniform
crossover within the substring given by parameters $c$ and $c+\ell-1$ if $r=0.5$
is chosen. Then all bits in this substring are set to uniform random values.
\citet{Zarges2011} showed that this can speed up the optimisation process.
However, hypermutation subjects \emph{all} bits in the substring to mutation,
whereas uniform crossover only ``mutates'' bits in which both parents disagree. In particular, uniform crossover treats bits where both parents agree differently from those where both parents disagree. Our royal road function requires bits in $\strl{x}$ to be kept and only bits in $\strr{x}$ to be mutated with a high mutation rate. Here the ability to treat these subsets of bits differently is crucial. Hypermutation does not have this ability, and we show that there is a function in the class \uRRRMOc for which hypermutation fails badly.}
%However this simulation assumes that for crossover two individuals $x,y$ with $H(x,y) = \ell$ between $c$ and $c+\ell-1$ are chosen.
%\andre{Dirk, could you revise this description? I dropped the last sentence, since for me everything is said if the bits are complementary!}
%In general hypermutation with $r=0.5$ works slighly different than uniform crossover on two individuals, since crossover keeps the bits which coincide in both individuals unchanged. This property is called \textit{respectful} \citep{Radcliffe1994}. For example hypermutation on $1^n$ gives with overwhelming probability an individual with a linear number of ones and zeros whereas uniform crossover on $(1^n, 1^n)$ leads to $1^n$. This leads to the result that GSEMO and even any black box algorithm using only hypermutation (which includes also NSGA-II) is not able to optimize the function class \uRRRMOc efficiently.}

\newedit{We consider the function $\uRRRMOx{\sigma}{0^n} = \uRRRMO(\sigma(x))$
with the permutation $\sigma$ satisfying:
\begin{align}
\forall x\in\{0,1\}^n\colon
\permut{\sigma}{x}
	= \permut{\sigma}{
       \underbrace{(\strl{x}^1, \strl{x}^2, \strl{x}^3, \strl{x}^4}_{\let\scriptstyle\textstyle\substack{\strl{x}}},
       \underbrace{\strr{x}^1, \strr{x}^2, \strr{x}^3, \strr{x}^4)}_{\let\scriptstyle\textstyle\substack{\strr{x}}}
       }
	%= (\strl{x}^1 , \strr{x}^1 , \strl{x}^2 , \strr{x}^2, \strl{x}^3 , \strr{x}^3, \strl{x}^4 , \strr{x}^4), %% this is the wrong one, which is rather \sigma^{-1}
	= \underbrace{(\strl{x}^1, \strl{x}^3, \strr{x}^1, \strr{x}^3}_{\let\scriptstyle\textstyle\substack{\strl{\permut{\sigma}{x}}}},
      \underbrace{\strl{x}^2, \strl{x}^4, \strr{x}^2, \strr{x}^4)}_{\let\scriptstyle\textstyle\substack{\strr{\permut{\sigma}{x}}}}.\label{eq:permut-for-hypermut}
\end{align}
Hereinafter, we use the notation in
Definition~\ref{def:the-subsets} to write the parts of a bit string $x$ and also apply this notation to the permuted string $\permut{\sigma}{x}$
of $x$, \eg $\strl{\permut{\sigma}{x}}=(\strl{x}^1, \strl{x}^3, \strr{x}^1, \strr{x}^3)$.
For the sake of clarity, we also introduce %the sets
    $\subsetU_\sigma, \subsetP_\sigma, \subsetC_\sigma, \subsetT_\sigma$
as the analogous sets to those defined in Definition~\ref{def:the-subsets},
however these sets are defined on the permuted string $\permut{\sigma}{x}$
instead of on $x$, \eg
$\subsetU_\sigma
    :=\left\{\strl{\permut{\sigma}{x}} \in \{0,1\}^{n/2}
             \mid \forall i \in [4]\colon \ones{\strl{\permut{\sigma}{x}}^i} \in [n/24,n/12]\right\}$.
The set of non Pareto-optimal solutions with non-null fitness vectors,
and that of Pareto-optimal solutions for $\uRRRMOx{\sigma}{0^n}$ are then:
\begin{align*}
\subsetK_\sigma
    &:=\{x\in\{0,1\}^n
        \mid \strr{\permut{\sigma}{x}} \in \subsetU_\sigma \vee \strr{\permut{\sigma}{x}}\in\subsetC_\sigma\},\\
\subsetW_\sigma
    &:=\{x \in \{0,1\}^n
        \mid \strl{\permut{\sigma}{x}} \in P_\sigma \wedge \strr{\permut{\sigma}{x}} \in T_\sigma\}.
\end{align*}
Note that the function to optimise is $\uRRRMO(\permut{\sigma}{x})$, however hypermutations
can only manipulate the original string $x$. Given the intertwined locations of the blocks
as shown in \eqref{eq:permut-for-hypermut}, the probability of creating any Pareto-optimal
solution for this function from any $x\in\subsetK_\sigma$ is then exponentially small.}
%The
%function to optimise is $\uRRRMO(\sigma(x))$, however hypermutations can only manipulate
%the original string $x$ and given the intertwined locations of the blocks, this turns out
%5to be a difficult task to optimise such a function.}

%We first prove the following lemma which states that with exponentially small probability hypermutation with any parameter $r>0$ creates any $y \in \subsetW_\sigma$ from an $x \in \subsetK_\sigma$:}

\begin{lemma}\label{lem:uRRMO-hypermut}
\newedit{
%Let $x \in \subsetK_\sigma$. Then hypermutation with any parameter $r > 0$ creates a Pareto-optimal search point $y \in \subsetW_\sigma$ from $x$ with probability at most $2^{-\Omega(n)}$.
Let $\mathrm{op}$ be any hypermutation operator according to Algorithm~\ref{alg:hypermutation}
with any parameter $r>0$, then for any $x\in\subsetK_\sigma$ it holds that
%\begin{align*}
$\prob{\mathrm{op}(x)\in\subsetW_\sigma} = 2^{-\Omega(n)}$.
}
\end{lemma}
\begin{proof}
\newedit{%Let $y:=\mathrm{op}(x)$, then
We distinguish two cases for $x$.}

\newedit{If $\strl{\permut{\sigma}{x}}\in\subsetU_\sigma$ \brandnewedit{then we argue as in the proof of Lemma~\ref{lem:prop-of-subsets}(i): %$y_{\ell, \sigma} \in U_\sigma$:
%The optimality of
creating a mutant} $y\in\subsetW_\sigma$ requires $\strl{\permut{\sigma}{y}} \in \subsetP_\sigma$
which means $\strl{\permut{\sigma}{y}}=1^a 0^{n/2-a}$ for some non-negative
integer $a$, and this implies $\strl{\permut{\sigma}{y}}^j \in \{0^{n/8},1^{n/8}\}$
for at least three distinct $j\in[4]$. \brandnewedit{We call these \emph{bad blocks}. By definition of $\subsetU_\sigma$,
%Furthermore,
$\ones{\strl{\permut{\sigma}{x}}^j}\in[n/24,n/12]$ for all blocks $j$, thus all bad blocks must be altered by the hypermutation as a necessary condition for creating $y$.}
%for three distinct $j$.
%Hence,
%thus in order to create $y$ it is necessary to flip at least $n/24$ bits in
%each of these blocks.
\brandnewedit{Recall that hypermutation operates on the original string~$x$.} Therefore, the values of $c$ and $\ell$ in Algorithm~\ref{alg:hypermutation}
have to be chosen so that the positions between $c$ and $c+\ell-1 \bmod n$ \brandnewedit{contain bits from all bad blocks. Since $\sigma$ permutes blocks as a whole, every interval containing bits from three bad blocks must cover at least one bad block entirely.}
%can touch at least some three blocks of
%$\strl{\permut{\sigma}{x}}=(\strl{x}^1,\strl{x}^3,\strr{x}^1,\strr{x}^3)$.
%
%i.e. the positions $i$ with $i \in \{c, \ldots , c+\ell-1 \mod n\}$ (i.e. the
%positions between $c$ and $c+\ell-1$) cover at least $n/24$ components of all
%those three blocks. Since the positions between $c$ and $c+\ell-1$ form one
%consecutive block or two consecutive blocks where one of them contains position
%$1$ and the other one position $n$, one complete block $y_{\ell,\sigma}^j$ is
%between $c$ and $c+\ell-1$.
%
%Regardless of which three blocks and of the values of $c$ and $\ell$, we remark
%that one block of the three is then fully covered by these $\ell$ consecutive
%positions.
This block has to be mutated into $0^{n/8}$ or $1^{n/8}$ in one
application of the hypermutation. By a union bound this happens with probability
at most $2 (r \cdot (1-r))^{n/24} \leq 2 \cdot (1/4)^{n/24} = 2^{-\Omega(n)}$
since we have to flip at least $n/24$ bits
%in $y_{\ell,\sigma}^j$
while keeping at least
%$n/8-n/12=n/24$
$n/24$ bits unchanged.
}

\newedit{If $\strr{\permut{\sigma}{x}}\in\subsetC_\sigma$ then %Suppose that $y_{r, \sigma} \in C_\sigma$:
%Since $y_{\ell,\sigma}^j \in \{0^n,1^n\}$ for three distinct $j$,
%In this setting, we get
$\strr{\permut{\sigma}{x}} \in \{1^a0^{n/2-a}, 0^a1^{n/2-a} \mid 0 \le a \le n/2\}$,
thus $\strr{\permut{\sigma}{x}}^j\in\{0^{n/8},1^{n/8}\}$ for at least three\footnote{Note that this can be exactly three blocks if for one $i \in [4]$ we have
$\strr{\permut{\sigma}{x}}^i \in \{1^{n/16}0^{n/16}, 0^{n/16}1^{n/16}\}$, because then
$\ones{\strr{\permut{\sigma}{x}}^i} = \zeros{\strr{\permut{\sigma}{x}}^i}\ = n/16$
and nothing needs to be changed for this block $i$.
} distinct
$j\in[4]$\brandnewedit{, and we again refer to these bits as \emph{bad blocks}}.
So, in order to have $\strr{\permut{\sigma}{y}}\in\subsetT_\sigma$,
we have to flip half of the bits in each of these three blocks %.
%$y_{\ell,\sigma}^j$, i.e. between $c$ and $c+\ell-1$ there are two complete
%blocks $y_{r,\sigma}^i$ and $y_{r,\sigma}^{i+1}$ and one complete block
%$y_{\ell,\sigma}^j$ from $\{0^n,1^n\}$.
and this implies that the \brandnewedit{interval of bits in $x$ subjected to hypermutation has to be chosen such that all bad blocks in
%at least some three blocks of
$\strr{\permut{\sigma}{x}}=(\strl{x}^2,\strl{x}^4,\strr{x}^2,\strr{x}^4)$
are touched. %Due to the intertwined locations of the blocks in $x$, which can be seen
%in \eqref{eq:permut-for-hypermut},
Since the interval is chosen in $x$ and $x=(\strl{x}^1,\strl{x}^2,\strl{x}^3,\strl{x}^4,\strr{x}^1,\strr{x}^2,\strr{x}^3,\strr{x}^4)$,
every interval containing bits from three bad blocks also contains a whole block from
%we see that one block of
$\strl{\permut{\sigma}{x}}=(\strl{x}^1,\strl{x}^3,\strr{x}^1,\strr{x}^3)$.
This can be easily verified by considering cases where blocks $\strl{x}^2$ and $\strr{x}^2$ are bad and otherwise considering bad blocks $\strl{x}^4$ and $\strr{x}^4$.}
%is between those three blocks.\andre{I clarified these.}
%Hence, due to the intertwined locations of the blocks in $x$, which can be seen
%in \eqref{eq:permut-for-hypermut}, we note that regardless of
%which three blocks
%and of
\brandnewedit{Thus, every interval contains a whole bad block in
%regardless of the values of $c$ and $\ell$,
%one block of
$\strr{\permut{\sigma}{x}}$
and also one whole block of $\strl{\permut{\sigma}{x}}=(\strl{x}^1, \strl{x}^3, \strr{x}^1, \strr{x}^3)$.
%is fully covered between $c$ and $c+\ell-1\bmod n$.
We bound the probability of creating optimal configurations in these two blocks.}
Flipping half of the bits in one block of
$\strr{\permut{\sigma}{x}}$ happens with probability
$\binom{n/8}{n/16}\cdot r^{n/16} \cdot (1-r)^{n/16}$.
\brandnewedit{For the block in $\strl{\permut{\sigma}{y}}$ we note that
there are
at most
$(n/8)+1$ optimal configurations $1^{b}0^{n/8-b}$ for $b\in\{0, \dots, n/8\}$. We fix one optimal configuration and argue that all bits in $\strl{\permut{\sigma}{y}}$ differing from in must be flipped, and all other bits in that block must not be flipped.
If $m$ bits differ, the probability for this event is }
%decisions whether to mutate or not to mutate bits in this block must be made
%happens with probability
$r^m\cdot (1-r)^{n/8-m}
\leq \max(r,1-r)^m \cdot \max(r,1-r)^{n/8-m}
= \max(r,1-r)^{n/8}
$ \brandnewedit{where the last expression no longer depends on~$m$.}
%for some $m \in \{0, \dots, n/8\}$.
%Note that
%$r^m \cdot (1-r)^{n/8-m}
%\leq \max(r,1-r)^{n/8}$ and
%Since
%there are
%%at most
%$n/8+1$ possible strings for a block in $\strl{\permut{\sigma}{y}}\in\subsetP_\sigma$
%as these strings can be any $1^{b}0^{n/8-b}$ for $b\in\{0, \dots, n/8\}$,
%Therefore,
\brandnewedit{By a union bound over all $(n/8)+1$ optimal configurations,} the probability of creating any $y \in \subsetW_\sigma$ from $x \in \subsetK_\sigma$ is at most, \brandnewedit{using $\binom{n/8}{n/16} \leq 4^{n/16}$,}
\[
\binom{n/8}{n/16} \left(r (1-r)\right)^{\frac{n}{16}}
\cdot
\left(\frac{n}{8}+1\right)\max(r,1-r)^{\frac{n}{8}}
    %\leq \left(\frac{n}{8}+1\right)
    %     (4r(1-r))^{\frac{n}{16}}
    %     \max(r,1-r)^{\frac{n}{8}}
    \leq \left(\frac{n}{8}+1\right)
             \left(4\max(r^2(1-r),r(1-r)^2)\right)^{\frac{n}{16}}
    =: p.
\]
%by a union bound and $\binom{n/8}{n/16} \leq 4^{n/16}$.
%If $r \notin [1/5,4/5]$
%we have that $(4r(1-r))^{n/16} \leq (16/25)^{n/16} =
%2^{-\Omega(n)}$ and thus $p \leq 2^{-\Omega(n)}$.
%If $r \in [1/5,4/5]$ we have that $\max(r,1-r)^{n/8} \leq (4/5)^{n/8} =
%2^{-\Omega(n)}$ and also $p \leq 2^{-\Omega(n)}$.
Note that $4\max(r^2(1-r),r(1-r)^2)$ on $r\in(0,1]$ reaches the maximum
value $16/27$ at either $r=1/3$ or $r=2/3$, therefore $p\leq (n/8+1)(16/27)^{n/16}
=2^{-\Omega(n)}$.
}
\end{proof}

\newedit{We can now prove Theorem~\ref{thm:uRRMO-blackbox-unbiasedvar}.}
\begin{proof}[Proof of Theorem~\ref{thm:uRRMO-blackbox-unbiasedvar}]
\newedit{Due to the unbiasedness of the variation operators, the result of
Lemma~\ref{lem:uRRMO-unbiasedvar} not only holds for $\uRRRMO$ but
also holds
for $\uRRRMOx{\sigma}{0^n}$ with $\sigma$ defined as in Equation~\eqref{eq:permut-for-hypermut}
(but with the sets $\subsetK_\sigma, \subsetW_\sigma$ in the statement of
the lemma) and in fact for any function of the class $\uRRRMO$ (with its appropriate
sets $\subsetK$ and $\subsetW$).
Thus, for simplification we will use the same function $\uRRRMOx{\sigma}{0^n}$
throughout the proof.}
%(the statement of Lemma~\ref{lem:uRRMO-unbiasedvar} with the sets
%$\subsetK_\sigma$ and $\subsetW_\sigma$).}

\newedit{Regarding the algorithms, note that the result of NSGA-II follows
immediately from that of the elitist black-box algorithm because NSGA-II with
$p_c=0$ is a special case of an elitist algorithm with $\lambda=\mu$. Hence,
it remains to show the results for GSEMO and the elitist black-box algorithm.}

\newedit{Starting with GSEMO, the probability that the initial search point belongs
to $\subsetK_\sigma$ is at least $1-2^{-\Omega(n)}$ by a Chernoff bound.
Assume this does happen, then the first $y \in \subsetW_\sigma$ can be only created
from a point $x \in \subsetK_\sigma$. By Lemmas~\ref{lem:uRRMO-unbiasedvar} and
\ref{lem:uRRMO-hypermut} this happens with probability is both at most
$2^{-\Omega(n)}$ by an unary variation operator or a hypermutation. Thus the
expected number of fitness evaluations is at least $(1-2^{-\Omega(n)})2^{\Omega(n)}=2^{\Omega(n)}$.}

\newedit{Now consider an elitist blackbox algorithm. Initialising every search
point in $\subsetK_\sigma$ happens with probability $1-\mu\cdot 2^{-\Omega(n)}=
1-o(1)$ by a Chernoff and a union bound, given $\mu \in \poly(n)$.
If this happens then due to the elitism the algorithm accepts only search points
from $\subsetK_\sigma$ or $\subsetW_\sigma$. Again by Lemmas~\ref{lem:uRRMO-unbiasedvar}
and \ref{lem:uRRMO-hypermut} the probability for creating a point in
$\subsetW_\sigma$ from a point in $K_\sigma$ with an unbiased variation
operator or hypermutation is $2^{-\Omega(n)}$. Thus the expected number of fitness
evaluations is then at least $(1-o(1))2^{\Omega(n)}=2^{\Omega(n)}$.}
\end{proof}

\newedit{Unlike the previous result for GSEMO on \RRRMO (Theorem~\ref{thm:gsemo-pc-zero-hypermutation})
where only one hypermutation operator with a fixed parameter $r$ is allowed,
such a restriction is not required here for the class \uRRRMOc.
\brandnewedit{In particular, Theorem~\ref{thm:gsemo-pc-zero-hypermutation} includes algorithms that may choose from all unary unbiased variation operators and hypermutation in every generation, and the hypermutation parameter $r$ can be chosen anew in every application of the operator.}
}

\newedit{\subsection{Use of Uniform Crossover Implies Expected Polynomial Optimisation Time on \uRRRMO}}

\newedit{We show for GSEMO and for NSGA-II that they can find the whole Pareto front for $\uRRRMO$ in polynomial expected time when using both standard bit mutation and uniform crossover.}

\newedit{\subsubsection{Analysis of GSEMO}\label{sec:gsemo-2}}

\newedit{\begin{theorem}\label{thm:uRRMO-gsemo}
GSEMO (Algorithm~\ref{alg:gsemo}) with $p_c\in(0,1)$, uniform
crossover and standard bit mutation finds a Pareto-optimal set of any
function in the class $\uRRRMOc$ in
   $\bigO\left(\frac{n^3}{p_c(1-p_c)}\right)$
fitness evaluations in expectation.
%Suppose that $0<p_c<1$. Then the expected number of fitness evaluations of GSEMO (Algorithm~\ref{alg:gsemo}) with uniform crossover and standard bitwise mutation to optimize \newedit{any function in the class $\uRRRMOc$} is $O(1/p_c \cdot n^3 + \frac{1}{1-p_c} \cdot n^3)$.
\end{theorem}}

\begin{proof}
\newedit{As in the proof of Theorem~\ref{thm:gsemo} for \RRRMO we use the method of typical runs and divide the optimisation procedure into several phases.}

\newedit{\textbf{Phase 1:} Create a search point $x$ with $\uRRRMO(x) \neq (0,0)$.}
% and $\uRRRMO(x) \neq (1/2,0)$.}
%\todo{Fix the argument for this phase with the new function definition}

\newedit{By a Chernoff bound and a union bound, the probability of initializing an
individual $x$ with fitness $%\uRRRMO(x) =
(0,0)$ %or $\uRRRMO(x) = (1/2,0)$
is at most $2^{-\Omega(n)}$ as it is necessary to create a search point $x$ with
$\ones{\strl{x}^i} \notin [n/24,n/12]$ for some $i \in [4]$ and the expected
number of ones in $\strl{x}^1$ or $\strl{x}^2$ is $n/16$ after initialization.
If this occurs then as long as every individual created has fitness $(0,0)$,
the population always consists of the latest search point and crossover,
if executed, has no effect.
To bound the expected number of generations for all $\ones{\strl{x}^i}$
to simultaneously reach the interval $[n/24,n/12]$, we divide the run of the phase
into subphases. Each subphase lasts for $\tau := \max_{i\in[4]}\{\tau_i\}$
generations where $\tau_i$ for each $i\in[4]$ is the first point during the subphase
that $\ones{\strl{x}^i}$ reaches interval $[3n/56, n/14]$, which is a subinterval
of our target $[n/24,n/12]$. A subphase is called \emph{successful} if during its
length, none of the $\ones{\strl{x}^i}$ which have entered the subinterval leave
our target. If a subphase is unsuccessful, a new one is assumed to start and
we repeat our analysis, thus the expected number of generations of the phase
is bounded from above by the expected waiting time for a successful subphase.}

\newedit{Recall that each block $\strl{x}^i$ has length $n/8$, thus it follows from
Lemma~\ref{lem:repeat-bitwise-mutation} with $c=1/8$ and $\varepsilon=(1/14-1/16)(2/c)=1/7$
that $\expect{\tau_i}=\bigO(n)$ for all $i\in[4]$. Because $\tau \leq \sum_{i=1}^{4}\tau_i$,
then by linearity of expectation we get the expected length of a subphase is
$\expect{\tau} \leq \expect{\sum_{i=1}^{4}\tau_i} = \sum_{i=1}^{4}\expect{\tau_i}=\bigO(n)$.
Furthermore, for any $i \in [4]$ if we define
    $Y_t:=\max\{\ones{\strl{x}^i},\zeros{\strl{x}^i}\}$,
then as shown in the proof Lemma~\ref{lem:repeat-bitwise-mutation}
it holds that $\expect{Y_{t+1}-Y_t \mid Y_t = s} \leq -c\varepsilon+o(1)=-1/56+o(1)$
for all $s \in [n/14,n/12]$.
Additionally, in order to have $|Y_t-Y_{t+1}|\geq m$ at least $m$ bits in $\strl{x}^i$
have to be flipped in order to have either $Y_t-Y_{t+1}\geq m$ or $Y_{t+1}-Y_{t}\geq m$,
hence by a union bound
\begin{align*}
\prob{|Y_t - Y_{t+1}|\geq m}
    \leq 2 {n/8 \choose m}n^{-m}
    = 2\cdot\frac{(n/8)!}{(n/8-m)!m!} \cdot n^{-m}
    \leq 2(n/8)^{m} n^{-m}
    \leq \frac{1}{4^m}.
\end{align*}
Theorem~\ref{thm:negative-drift}
    applied to the process $Y'_t = n/8 - Y_t$
    with $a=(n/8-n/12)/n=1/24$,
         $b=(n/8-n/14)/n=3/56$,
         $r=1$ and $\eta=3$ implies that even within an exponential number of generations,
the probability that $Y_t$ goes above $n/12$ again, or in other words $\ones{\strl{x}^i}$
leaves the target interval $[n/24,n/12]$, is at most $2^{-\Omega(n)}$. Since this holds
for any block $\strl{x}^i$, the probability of an unsuccessful subphase is still at most
$4! \cdot 3\cdot 2^{-\Omega(n)}=2^{-\Omega(n)}$ even by taking into account any possible
order for the blocks to reach the subinterval and a union bound on the failure probabilities
of the three first blocks.}
\newedit{
By the method of typical runs, the expected length of the first phase is then
%$1 + 2^{-\Omega(n)}\left((\tau_1+\tau_2+\tau_3+\tau_4)/(1-2^{-\Omega(n)})\right)
$1 + 2^{-\Omega(n)}\left(\expect{\tau}/(1-2^{-\Omega(n)})\right)
=1 + 2^{-\Omega(n)}\left(\bigO(n)/(1-2^{-\Omega(n)})\right)
=1+o(1)$.}

\textbf{Phase 2:} Create a search point $x$ with $\strr{x} \in C$. %or $\LO(\strl{x})+\TZ(\strl{x}) = m$ and $\strr{x} \in T$.

Suppose that there is no such search point in $P_t$. Let
\[
k_1:=\max\{\LO(\strr{x})+\TZ(\strr{x}) \mid x \in P_t \text{ with } \LO(\strr{x}) \neq 0\}
\]
and
\[
k_2:=\max\{\LZ(\strr{x})+\TOs(\strr{x}) \mid x \in P_t \text{ with } \LO(\strr{x}) = 0\}.
\]
Let $m:=\max\{k_1,k_2\} \in \{1, \ldots , n/2-1\}$. Since a search point $y$ with $\LO(\strr{y})+\TZ(\strr{y}) = k_1$ ($\LZ(\strr{y})+\TOs(\strr{y}) = k_2$) can be only dominated by a search point $z$ with $\LO(\strr{z})+\TZ(\strr{z}) \geq k_1$ ($\LZ(\strr{z})+\TOs(\strr{z}) \geq k_2$), we see that $m$ cannot decrease.
%with $\LO(\strl{y})+\TZ(\strl{y}) \neq m$ and $(\LO(\strr{y})+\TZ(\strr{y}))$-value $k_1$ is incomparable to every point $x$ with $\LO(\strr{x})=0$ (since $f_1(x) > n/2$ and $f_1(y) \leq n/2$ while $f_2(x) < n/2$ and $f_2(y) \geq n/2$), $y$ can be only dominated by a search point $z$ with $\LO(z)+\TZ(z)>k_1$ %Suppose that $m=k_1$. %Since a search point $y$ with $\LO(\strl{y})+\TZ(\strl{y}) \neq m$ and $(\LO(\strr{y})+\TZ(\strr{y}))$-value $k_1$ is incomparable to every point $x$ with $\LO(\strr{x})=0$ (since $f_1(x) > n/2$ and $f_1(y) \leq n/2$ while $f_2(x) < n/2$ and $f_2(y) \geq n/2$), $y$ can be only dominated by a search point $z$ with $\LO(z)+\TZ(z)>k_1$.
To increase the value $m$ it suffices to choose an individual $y$ with $\strl{y} \neq 0$, $\LO(\strr{y})+\TZ(\strr{y})=k_1$ if $m=k_1$ or $\LZ(\strr{y})+\TOs(\strr{y})=k_2$ if $m=k_2$, omit crossover and flip one specific bit of $y$ in the mutation step. Since the parent is chosen uniformly at random, i.\;e.\ with probability at least $1/n$ by Lemma \ref{lem:size-non-dom-set2}, the probability for this event is at least $(1-p_c) \cdot 1/n^2 \cdot (1-1/n)^{n-1} \geq \frac{1-p_c}{en^2}$. Since there are at most $n/2$ different values for $m$ we need at most $n/2-1$ such improving steps. So the expected number of fitness evaluations to complete this phase is at most $\frac{1}{1-p_c}(n/2-1)en^2 = O(\frac{n^3}{1-p_c})$.
%%Was ist mit $\ell = k_2 ?

\textbf{Phase 3:} Create a search point $x$ with $\strr{x} \in C$ and $\strl{x}=0^{n/2}$.% or a pareto optimal solution.

We pessimistically assume that there is no such search point \newedit{in the current population}. To make progress towards the goal to this phase, it suffices to first select an individual $x$ which has a minimum value $i$ of ones in the first half of the string amongst those individuals with $\strr{x} \in C$, then to omit crossover and to flip one of the $i$ $1$-bits in the first half during mutation. %while keeping the other bits unchanged.
The probability of this event is at least
%$s_i$ for improvement is then bounded by
%of this event is at least
%\[
$\frac{1}{n} \cdot (1-p_c) \cdot \binom{i}{1} \cdot \frac{1}{n} \cdot (1-1/n)^{n-1} = \frac{1}{n} \cdot (1-p_c) \cdot \frac{i}{n} \cdot (1-1/n)^{n-1} \geq \frac{i(1-p_c)}{en^2}.$
Since $n/2$ such steps are enough we have that the goal of this phase is reached after at most
\[
\sum_{i=1}^{n/2} \frac{en^2}{i(1-p_c)} = O \left(\frac{n^2 \log(n)}{1-p_c} \right)
\]
fitness evaluations in expectation.

\textbf{Phase 4:} The population contains every search point $x$ with $\strr{x} \in C$ and $\strl{x} = 0^{n/2}$. %or a pareto optimal solution.
\andre{Rewrote this}

Suppose that this condition is not fulfilled. Then there must exist a search point $z \notin P_t$ with $\strl{z}=0^{n/2}$ and $\strr{z} \in C$ as well as a search point $y \in P_t$ with $\strl{y}=0^{n/2}$, $\strr{y} \in C$ and $H(y,z)=1$. Then the individual $z$ can be generated by choosing $y$ as a parent, omitting crossover and flipping a specific bit while keeping the other bits unchanged. The probability of this event is at least $(1-p_c) \cdot 1/n^2 \cdot (1-1/n)^{n-1} \geq (1-p_c) \cdot 1/n^2 \cdot 1/e$. Note that $n-1$ such steps suffice to finish this phase since there at most $n$ different search points $x$ with $\strl{x}=0^{n/2}$ and $\strr{x} \in C$. Thus the expected number of fitness evaluations is at most
\[
\frac{en^2(n-1)}{1-p_c} = O \left(\frac{n^3}{1-p_c} \right).%,
\]

\textbf{Phase 5:} Create one point in $\subsetW$.

Suppose that there is no such point \newedit{in the current population}. Then $x \in P_t$ if and only if $x=(\strl{x},\strr{x})$ with $\strl{x}=0^{n/2}$ and $\strr{x} \in C$. To generate a point in $\subsetW$ one can choose two individuals $x,y$ with maximal Hamming distance (which happens with probability $1/n$), apply uniform crossover to generate an individual $z=(\strl{z},\strr{z})$ with $\strl{z} = 0^{n/2}$ and $\strr{z} \in T$ and not flip any bit during the mutation step. For the outcome $\strr{z}=(\strr{z}^1,\strr{z}^2,\strr{z}^3,\strr{z}^4)$ of uniform crossover it is sufficient that $\strr{z}^i$ contains $n/16$ zeros and $n/16$ ones. The probability of this event is $\binom{n/8}{n/16} (\frac{1}{2})^{n/8} = \Theta(n^{-1/2})$ by Stirling's formula for a given $i \in [4]$. Since these events are independent for different $i$, the probability of generating a suitable $z$ is at least $\Theta(n^{-2})$. Thus, the probability of generating a search point in $W$
%and thus to generate the global optimum
is at least $1/n \cdot p_c \cdot \Theta(n^{-2}) \cdot (1-1/n)^n = \Theta(p_cn^{-3})$ (since $(1-1/n)^n \geq 1/4$). So the expected number of fitness evaluations to complete this phase is at most $\Theta(\frac{1}{p_c} \cdot n^3)$.

\textbf{Phase 6:} Find a Pareto-optimal set.
%Find a Pareto-optimal set.

Once a search point $x$ with $\LO(\strl{x})+\TZ(\strl{x})=n/2$ and $\strr{x} \in T$ is created, the process of covering the front
is similar to that of creating every search point $y$ with $\strr{y} \in C$ and $\strl{y}=0^{n/2}$ in Phase 4, with only minor differences (e.g. we have to consider only $n/2$ optimization steps instead of $n-1$ many where we flip a specific bit in the first half of the string).  %while keeping the rest of the bits unchanged).
The expected number of fitness evaluations to finish this phase is at most $O \left(\frac{n^3}{1-p_c} \right)$.%,

Summing up the expected times of all the phases gives an upper bound $O(\frac{1}{p_c} \cdot n^3 + \frac{1}{1-p_c} \cdot n^3)$ on the number of generations to optimize $\uRRRMO$.
\newedit{Noting $\frac{1}{p_c} + \frac{1}{1-p_c} = \frac{1-p_c}{p_c(1-p_c)} + \frac{p_c}{p_c(1-p_c)} = \frac{1}{p_c(1-p_c)}$ completes the proof.}
 %which is polynomial if $1/p_c \in \text{poly}(n)$ and $\frac{1}{1-p_c} \in \text{poly}(n)$.
\end{proof}

\subsubsection{Analysis of NSGA-II}\label{sec:nsga-ii-2}

\newedit{We now turn to the analysis of NSGA-II. Here we deal with a population size of $\mu=5n$ since we have at most $n$ non-dominated solutions.}
	 %The fact that we have at most $n$ non-dominated solutions (see Lemma~\ref{lem:size-non-dom-set2}) gives us with Lemma~\ref{lem:nsga-ii-protect-layer} an upper bound of $4n$ on the number of individuals with positive crowding distance in the first layer.}

\begin{theorem}\label{thm:uRRMO-nsgaii}
NSGA-II~(Algorithm~\ref{alg:nsga-ii}) with $p_c \in (0,1)$ and $\mu \geq 5n$ finds a Pareto-optimal set of
\newedit{any function of the class \uRRRMOc}
in
$O\left(\frac{\mu n}{p_c} + \frac{n^2}{1-p_c} \right)$ generations, or
$O\left(\frac{\mu^2 n}{p_c} + \frac{\mu n^2}{1-p_c} \right)$ fitness evaluations,
in expectation.
\end{theorem}

\begin{proof}
As in the proof of Theorem~\ref{thm:nsga-ii} we use the method of typical runs and divide the optimisation procedure into several phases. Since by Lemma \ref{lem:size-non-dom-set2} there are at most $n$ non-dominated search points with respect to $\uRRRMO$ we see with Lemma \ref{lem:nsga-ii-protect-layer} that there can be at most $4n \leq (4/5) \mu$ individuals in $F_t^1$ with positive crowding distance.
Therefore at least $\mu/5$ individuals are either in $F_t^1$ and have zero crowding distance or are in a layer below $F_t^1$.

\textbf{Phase 1:} Create a search point $x$ with $\uRRRMO(x) \neq 0$.

\newedit{By a Chernoff bound the probability that every initial individual $x$
	fulfills $\ones{\strl{x}^i} \notin [n/12,n/24]$ for $i \in [4]$ is at most $2^{-\mu \Omega(n)}$. If this happens, %regardless,
	the probability of creating a specific individual
	by mutation is at least $n^{-n}$, regardless of the input
	%search point
	solution and the
	%operations preceding the mutation.
	preceding crossover.
	By the law of total probability, the expected number of evaluations to
	obtain a search point with fitness $f(x)$ is at most
	$%\[
	\mu + 2^{-\mu\Omega(n)}n^n
	= \mu + 2^{- \Omega(n^2)}\cdot n^n
	= \mu + o(1)
	%\mu + 2^{-\mu \Omega(n)}n^n \leq \mu + 2^{- \Omega(n^2)}\cdot n^n = \mu + o(1).
	$, %\]
	and these are $1+o(1)$ generations.}
%Suppose that there are only search points with fitness $0$. A point with fitness $0$ satisfies $\LO(\strl{x})+\TZ(\strl{x})=m$. As in Phase 1 in the proof for GSEMO the probability of creating a search point with larger fitness is $\binom{n/2-2}{1} (1/n) (1-1/n)^{n-1} \geq \frac{1}{2e}$ regardless of the preceeding operations. Therefore the expected number of fitness evaluations to complete this phase is at most $O(1)$.

\textbf{Phase 2:} Create a search point $x$ with $\strr{x} \in C$.%or $\LO(\strl{x}) + \TZ(\strl{x}) = m$.

Suppose that there is no such search point $x$ \brandnewedit{in the population}. Let $k_1,k_2$ and $m$ \andre{I introduced also $m$ here!} be as in Phase 3 in the proof for GSEMO. Note that $m$ cannot decrease, since there are at most $n$ non-dominated solutions and a search point with $\LO(\strr{y})+\TZ(\strr{y})=k_1$ ($\LZ(\strr{y}) + \TOs(\strr{y}) \geq k_2$) can only be dominated by a search point $z$ with $\LO(\strr{z})+\TZ(\strr{z}) \geq k_1$ ($\LZ(\strr{z}) + \TOs(\strr{z}) \geq k_2$).
%\[
%k_1:=\max\{LO(x'')+TZ(x'') \mid x \in P_t \text{ with } LO(x'') \neq 0 \text{ and } x' \neq 0^{n/2}\}
%\]
%and
%\[
%k_2:=\max\{LZ(x'')+TOs(x'') \mid x \in P_t \text{ with } LO(x'') = 0 \text{ and } x' \neq 0^{n/2}\}.
%\]
%Let $m:=\max\{k_1,k_2\}$.
Suppose that $m=k_1$. As in the proof for GSEMO, the case ''$m=k_2$'' is handled similarly. Let $y$ be an individual with positive crowding distance, $\LO(\strr{y})+\TZ(\strr{y}) = k_1$ if $m=k_1$ or $\LZ(\strr{y})+\TOs(\strr{y})=k_2$ if $m=k_2$. If $y$ appears as the first competitor in a binary tournament (which happens with probability $1/\mu$), and the second competitor is from $F_t^1$ and has zero crowding distance or is a layer below $F_t^1$ (which happens with probability at least $1-\frac{4n}{5n} = 1/5$ as remarked at the beginning of the proof of this theorem), then $y$ wins the tournament. The same holds for swapping the roles of the first and second competitor. Therefore we have with probability at least $\frac{2}{5 \mu}$ that $y$ is the outcome of one tournament (since the multiset of individuals with positive crowding distance is disjoint to the multiset of individuals which are either below $F_t^1$ or from $F_t^1$ with zero crowding distance).
Since there are two tournaments in generating a pair of offspring we obtain that the probability for generating $y$ as an outcome of one of the two tournaments is at least $1-(1-\frac{2}{5 \mu})^2 \geq \frac{4/(5 \mu)}{4/(5 \mu) + 1} = \frac{4}{4+5\mu}$ \newedit{by Lemma~\ref{lem:lambda-trials}}.
%Then $y \in F_t^1$ and $y$ is picked as a competitor for the two binary tournaments to choose $\{p_1,p_2\}$ with probability $1-(1-1/\mu)^4 \geq 4/(\mu+4)$ and this guarantees that at least one of the parents has a value of $k_1$. EBEN NICHT!!!! WENN CD=0!!!
Thus to increase $k_1$ it suffices to generate $y$ as such an outcome, omit crossover and flip a specific bit to one while keeping the other bits unchanged. Therefore, with probability at least $\frac{4}{4+5\mu} \cdot (1-p_c) \cdot \frac{1}{n} \cdot (1-1/n)^{n-1}:=b$, one of the offspring $s$ from $\{s_1',s_2'\}$ has a larger $(\LO(\strr{s})+\TZ(\strr{s}))$-value than $m$ if $m=k_1$ (larger $(\LZ(\strr{s})+\TOs(\strr{s}))$-value than $m$ if $m=k_2$). This reproduction process is repeated $\mu/2$ times, so the chance of increasing $m$ in one generation is at least $1-(1-b)^{\mu/2} \geq \frac{b \mu/2}{b \mu/2+1}$ \newedit{(also by Lemma~\ref{lem:lambda-trials})}. Since we have to increase $m$ at most $n/2-1$ times, the expected number of generations is no more than
\[
\left(\frac{n}{2}-1 \right)\left(1+\frac{1}{b \mu/2}\right) = \frac{n}{2}-1+\frac{2n(4+5\mu)}{4\mu \cdot (1-p_c) \cdot (1-1/n)^{n-1}} \cdot \left(\frac{n}{2}-1\right) = O \left(\frac{n^2}{1-p_c} \right).
\]

\textbf{Phase 3:} Create a search point $x$ with $\strr{x} \in C$ and $\strl{x}=0^{n/2}$. %or with $\strr{x} \in T$ and $\LO(\strl{x})+\TZ(\strl{x})=n/2$.

We pessimistically assume that there is no individual with these properties \newedit{in the current population}. We say that individual $x \in P_t$ has property $\mathcal{Q}_i$ for $i \in \{0, \ldots, m\}$ if $\strr{x} \in C$ and $\vert{\strl{x}}\vert_1=i$. Let $i>0$ be minimal such that there is an individual $x \in P_t$ with property $\mathcal{Q}_i$. Then $x \in F_t^1$ and thus at least one of the outcomes of the two binary tournaments has property $\mathcal{Q}_i$ if $x$ is picked as a competitor and the second competitor is in a layer below $F_t^1$ or has zero crowding distance. Similar as in Phase 2, this happens with probability at least $\frac{4}{4+5\mu}$. To generate an individual with property $\mathcal{Q}_j$ for $j<i$ from such an outcome it suffices to omit crossover and flip a $1$-bit in the first half of the string to zero while keeping the remaining bits unchanged. Therefore the probability is at least $r_i:=\frac{4}{4+5\mu} \cdot (1-p_c) \cdot \frac{i}{n} \cdot (1-1/n)^{n-1}$ that one of the offspring $\{s_1',s_2'\}$ has property $\mathcal{Q}_j$ for $j<i$. Since the reproduction process is repeated $\mu/2$ times we get no more than
\[
\sum_{i=1}^{n/2-1} \left(1+\frac{1}{r_i \mu/2}\right) = \sum_{i=1}^{n/2-1} \left(1+\frac{2n(4+5\mu)}{4 \cdot i \cdot \mu \cdot (1-p_c) \cdot (1-1/n)^{n-1}}\right) = O \left(\frac{n \log(n)}{1-p_c} \right)
\]
for the expected number of generations.

\textbf{Phase 4:} % Create a search point $x$ with $\LO(\strl{x})+\TZ(\strl{x})=n/2$ and $\strr{x} \in T$ or
Create every search point $x$ with $\strr{x} \in C$ and $\strl{x}=0^{n/2}$.

Suppose that $C$ is not completely covered by search points $x$ with $\strr{x} \in C$ and $\strl{x}=0^{n/2}$. Let $y$ and $z$ be defined as in Phase 4 in the proof of Theorem 1. Additionally assume that $y$ has a positive crowding distance in $F_t^1$. As in the previous phases the probability that $y$ is the outcome of a binary tournament is at least $\frac{4/(5 \mu)}{4/(5 \mu) + 1} = \frac{4}{4+5\mu}$. Thus the probability of generating $z$ is at least $a:=\frac{4}{4 + 5\mu} \cdot (1-p_c) \cdot 1/n \cdot (1-1/n)^{n-1} \geq \frac{4}{4 + 5\mu} \cdot (1-p_c) \cdot 1/n \cdot 1/e$ during the creation of the pair of offspring. The success probability for $\mu/2$ pairs is then at least $1-(1-a)^{\mu/2} \geq \frac{a \mu/2}{a \mu/2 + 1}$ \newedit{by Lemma~\ref{lem:lambda-trials}}. Therefore the expected number of generations to complete this phase is no more than $(n-1) \cdot (1 + \frac{1}{a \mu/2})$ which is
\[
n-1 + \frac{2en(n-1)(4+5\mu)}{4\mu(1-p_c)} = O \left(\frac{n^2}{1-p_c} \right).
\]

\textbf{Phase 5:} Create a search point $x \in \subsetW$. %with $\strr{x} \in T$ and $\LO(\strl{x})+\TZ(\strl{x})=n/2$.

Note that $P_t$ contains all search points $x$ with $\strl{x}=0^{n/2}$ and $\strr{x} \in C$ and they are in $F_t^1$. Note that $F_t^1$ consists only of search points of this form. There are at least $n$ solutions $x$ with $\strl{x}=0^{n/2}$, $\strr{x} \in C$ and positive crowding distance in $P_t$, thus one of them wins the tournament for selecting $p_1$ with probability at least $\frac{2n}{5 \mu}$. Then it suffices to select a specific solution in $P_t$ with positive crowding distance as $p_2$, which happens with probability at least $\frac{2}{5 \mu}$. (This could be a solution $y$ with $\strl{y}=0^{n/2}$, $\strr{y} \in C$, maximal Hamming distance to $p_1$ and positive crowding distance). Then apply uniform crossover on $p_1$ and $p_2$ to generate an individual $z=(\strl{z},\strr{z})$ with $\strl{z}=0^{n/2}$ (i.\;e.\ $\LO(\strl{z})+\TZ(\strl{z}) = n/2$) and $\strr{z} \in T$ (which happens with probability at least $p_c \cdot \left(\binom{n/8}{n/16} (1/2)^{n/8} \right)^4$ as in the proof in Phase 5 for GSEMO) and do not flip any bit during the mutation step. Therefore we have a probability of at least
\[
w:=p_c \cdot \frac{4}{25} \cdot \frac{n}{\mu^2}  \cdot \left(\binom{n/8}{n/16} (1/2)^{n/8} \right)^4 \cdot (1-1/n)^n = \Omega \left(\frac{p_c}{\mu^2 n} \right)
\]
that an offspring $x$ with $\strr{x} \in T$ and $\LO(\strl{x})+\TZ(\strl{x})=n/2$ (i.\;e.\ $x \in \subsetW$) is generated during the creation of a pair of offspring. \newedit{By Lemma~\ref{lem:lambda-trials}} at least one success occurs in a generation with probability at least $1-(1-w)^{\mu/2} \geq \frac{w\mu/2}{w\mu/2+1}$. Thus $1+\frac{2}{\mu w} = O(\frac{\mu n}{p_c})$ generations are sufficient in expectation to finish this phase.

\textbf{Phase 6:} Find a Pareto-optimal set.

Once a search point $x$ with $\LO(\strl{x})+\TZ(\strl{x})=n/2$ and $\strr{x} \in T$ is created, the process of covering the front is similar to that of creating every search point $y$ with $\strr{y} \in C$ and $\strl{y}=0^{n/2}$ in Phase 4 with only minor differences (e.g. we have to consider only $n/2$ optimization steps instead of $n-1$ many where we flip a specific bit in the first half of the string while keeping the rest of the bits unchanged). The expected number of generations to finish this phase is at most $O \left(\frac{n^2}{1-p_c} \right)$.

We see that the expected number of generations for finding a global optimum in total is $O(\frac{\mu n}{p_c} + \frac{n^2}{1-p_c})$ by adding the run times of the single phases. %This runtime is polynomial if $1/p_c \in \text{poly}(n)$, $\frac{1}{1-p_c} \in \text{poly}(n)$ and $\mu \in \text{poly}(n)$.
We obtain the expected number of fitness evaluations by multiplying the expected number of generations with $2 \mu$.
\end{proof}

%\section{Experiments}\label{sec:experiments}
%\todo{To discuss if we want to setup and show some experiments where the running time just shot up without the required operators}
\section{Conclusions}\label{sec:conclude}
\newedit{We have identified the function classes \RRRMO and \uRRRMOc as examples
on which EMO algorithms GSEMO and NSGA-II using crossover and standard bit
mutation can find a whole Pareto set in expected polynomial time,
\ie $\bigO(n^4)$ for \RRRMO and $\bigO(n^3)$ for \uRRRMOc,
with any constant crossover probability $p_c \in (0, 1)$.
For NSGA-II, these results hold for a linear population size, specifically
we require $\mu \geq 2n+5$
%in the case of
for \RRRMO with one-point crossover
and $\mu \geq 5n$
%in case of
for \uRRRMOc with uniform crossover.
More generally, the function classes can be optimised in expected
%by the algorithms in expected
polynomial time if $1/p_c$, $1/(1-p_c)$ and $\mu$ are polynomials in~$n$.
%In any case, crossover is a vital operator as simply finding any
%Pareto-optimal point requires exponential expected time for GSEMO and NSGA-II.
%Theorem~\ref{thm:elitist-blackbox-unary-unbiased} on a broad class of elitist
%EMO algorithms showed that this cannot be remedied by using other unbiased
%mutation operators.
%Even hypermutation simulating uniform crossover (if $r=1/2$) and one-point
%crossover (if $r=1$) fails to optimize GSEMO in case of \RRRMO, and fails to
%optimize GSEMO and a broad class of elitist EMO algorithms in case of \uRRRMO
%(where NSGA-II is included).
We have shown that crossover is a vital operator on these functions, as
simply finding any Pareto-optimal point requires exponential
expected time for GSEMO and NSGA-II when disabling crossover ($p_c=0$). This generalises to all elitist black-box algorithms, even when allowing arbitrary unbiased mutation operators.}

%, and (in the case of $\uRRRMO$) even
%with the use of hypermutation replacing standard bit mutation.}
%While previous work has \newedit{mainly considered standard bit mutation
%(except for \cite{Huang2021b}), and has} mostly
%used crossover to speed up filling the Pareto set~\citep{Qian2011,Qian2013,Bian2022PPSN},
%our work \newedit{looks at both standard bit mutation and hypermutation} and
%\newedit{we show that crossover} can also be essential for discovering the Pareto
%set in the first place.
%We are hopeful that our results and the proofs may serve as stepping stones towards a
%better understanding of the role of crossover in EMO, \brandnewedit{and that this paves the way for identifying more natural examples of problems where crossover is beneficial, in the same way that this was achieved for single-objective optimisation.}

While previous work has mostly
used crossover to speed up filling the Pareto set~\citep{Qian2011,Qian2013,Bian2022PPSN},
our work shows that crossover can also be essential for discovering the Pareto
set in the first place. Another novel aspect is the consideration of
hypermutation (which was only considered for \brandnewedit{GSEMO} in~\cite{Huang2021b}) and
we showed that for \brandnewedit{\uRRRMOc} allowing the use of hypermutation does not avoid
exponential optimisation times. The same holds when using GSEMO with
hypermutation as the only mutation operator %, however
\brandnewedit{on \RRRMO, but}
it remains an open problem whether this also holds for NSGA-II.

We are hopeful that our results and the proofs may serve as stepping stones towards a
better understanding of the role of crossover in EMO, \brandnewedit{and that this paves the way for identifying more natural examples %of problems
    where crossover is beneficial, in the same way that this was achieved for single-objective optimisation.}

%\newedit{Our paper also spawns an open question regarding the behaviour of
%NSGA-II on \RRRMO with hypermutation as the only variation operator, and
%this can be interesting for future research.}\dirk{Can we expand on future work? Otherwise, I'd prefer to cut this sentence and to finish with the paragraph above.}

%\ifreview
%\else
%\bigskip

\section*{Acknowledgments}
The authors thank Bahare Salehi for her contributions to the preliminary version of this article~\citep{Dang2023}.
This work benefited from discussions at Dagstuhl seminar~22081.
%and from the ThRaSH seminar series (\url{https://thrash-seminars.github.io}).
%The third author was supported by the Erasmus+ Programme of the European Union
%\fi

%\clearpage
%\ifarxiv
\bibliographystyle{abbrvnat}
%\fi
\bibliography{references}

\end{document}